\PassOptionsToPackage{numbers}{natbib}
\documentclass{article}
\usepackage[main,final]{neurips_2026}

\usepackage[utf8]{inputenc}
\usepackage[T1]{fontenc}
\usepackage[hidelinks]{hyperref}
\usepackage{url}
\usepackage{booktabs}
\usepackage{amsfonts}
\usepackage{nicefrac}
\usepackage{microtype}
\usepackage[dvipsnames, table]{xcolor}
\usepackage{tikz,pgf}
\usepackage{stmaryrd}
\usetikzlibrary{positioning,fit,shapes,patterns,arrows,arrows.meta,shapes,automata,
  calc,decorations.pathreplacing,backgrounds}
\usepackage{pgfplots}
\pgfplotsset{compat=1.18}
\usepackage{amsmath, amssymb, amsthm, dsfont, mathtools, multirow, mathrsfs,
  relsize, bm}
\usepackage[ruled,vlined]{algorithm2e}
\usepackage{enumitem}
\usepackage{longtable}
\usepackage{subcaption}
\usepackage{pdflscape}
\usepackage{graphicx}
\usepackage{placeins}
\usepackage{float}
\usepackage{array}
\usepackage{wrapfig}
\graphicspath{{outputs/figs/}}
\newcommand{\bench}[1]{\textsf{#1}}
\newcommand{\arm}[1]{\textrm{#1}}

\newcommand{\benchfig}[1]{%
  \IfFileExists{outputs/figs/#1}%
    {\includegraphics[width=\linewidth]{#1}}%
    {\fbox{\parbox{0.95\linewidth}{\centering\small
      \texttt{\detokenize{#1}}\\[2pt]\emph{(figure pending bench re-run)}}}}%
}

\newtheorem{theorem}{Theorem}
\newtheorem{corollary}[theorem]{Corollary}
\newtheorem{lemma}[theorem]{Lemma}

\newtheorem{definition}{Definition}
\newtheorem{assumption}{Assumption}
\newtheorem{remark}{Remark}

\theoremstyle{definition}
\newtheorem{example}{Example}

\newcommand{\KL}{\mathrm{KL}}

\newcommand{\doIntervention}{\mathbf{do}}

\tikzset{
  obs/.style    = {circle, draw, thick, minimum size=1cm, font=\small},
  latent/.style = {circle, draw, thick, dashed, minimum size=1cm, font=\small},
  exoT/.style   = {rectangle, draw, thick, minimum size=0.9cm, font=\small},
  intA/.style   = {circle, draw=orange!80!black, fill=red!15, thick,
                   minimum size=1cm, font=\small},
  intM/.style   = {rectangle, draw=blue!60!black, fill=blue!15, thick,
                   minimum size=0.9cm, font=\small},
}
\colorlet{colM}{blue!60!black}
\colorlet{colA}{orange!80!black}
\colorlet{colU}{gray!60}
\colorlet{colCausal}{green!50!black}

\title{Certified Policy Optimisation for Nested Causal Bandits via PAC-Bayes Risk}

\author{%
  Tim Woydt \\
  ProdAxon\\
  Department of Computer Science\\
  Technical University of Darmstadt\\
  \And
  Paul-David Zuercher \\
  ProdAxon\\
  Institute for Manufacturing \\
  Department of Engineering \\
  University of Cambridge \\
  \AND
  \texttt{aegis+\{paul, tim\}@prodaxon.com}
}

\begin{document}

\maketitle

\begin{abstract}
Critical sequential decisions are rarely single-timescale: a strategic decision causally shapes the context in which every subsequent tactical choice is made; standard bandit and reinforcement-learning theory does not capture this causal coupling between timescales. We formalise the problem class as \emph{Nested Contextual Causal Bandits} (NCCBs), a hierarchical SCM where each level's action sets the next level's context distribution, and propose \emph{Nested Causal Thompson Sampling} (NCTS), which draws one mechanism-factorised belief per episode and acts recursively under it. Our main theoretical result is a causal PAC-Bayesian value bound that certifies any posterior-sampled deployment policy from historical data alone, off-policy and anytime, answering the deployment question: \emph{can we trust this agent here, and at what risk?} Experiments on a hierarchical SCM show that, against a matched RFF-GP joint regression on the same function class, the factorised SCM-mechanism posterior transfers \emph{significantly} better zero-shot under large exogenous distribution shifts, the recursive meta-to-inner commit \emph{significantly} dominates the joint-commit alternative in distribution, and the certificate \emph{significantly} contracts as offline data accumulates. Combining these results, we establish \emph{progressive certified handover}, a safe-deployment method: each timescale flips from a legacy controller to NCTS when gains can be certified, re-certified after each earlier flip.\footnote{Code: \url{https://github.com/PaulDavidZuercher/AEGIS}.}
\end{abstract}

\section{Introduction}
\label{sec:intro}

\begin{figure}[t]
\centering
\captionsetup{skip=4pt}
\begin{subfigure}[t]{0.28\textwidth}
\centering
\tikz\node[font=\smaller,inner xsep=5pt, inner ysep=2pt]
  {(a) NCCB (Def.~\ref{def:nccb})};\par\smallskip
\begin{tikzpicture}[>=Stealth, scale=0.75, transform shape, every node/.append style={font=\scriptsize}]
  \node[exoT]   (T)  at (0, 4.8)   {$C^{(2)}$};
  \node[latent] (U)  at (2, 5.5)     {$U$};
  \node[intM]   (M)  at (0, 3.2)   {$A^{(2)}$};
  \node[obs, draw=colCausal] (C)  at (2, 4)   {$C^{(1)}$};
  \node[intA]   (A)  at (4, 4)     {$A^{(1)}$};
  \node[obs, draw=colCausal] (Y)  at (2.8, 2) {$Y$};
  \node[obs, draw=colCausal] (R)  at (1.2, 2) {$R$};
  \draw[->, colA]      (T) to[bend left=10] (A);
  \draw[->, colA]      (M) to[bend right=10] (A);
  \draw[->, colA]      (C) -- (A);
  \draw[->, colM]      (T) -- (M);
  \draw[->, colCausal] (T) -- (C);
  \draw[->, colCausal] (M) -- (C);
  \draw[->, colCausal] (A) -- (Y);
  \draw[->, colCausal] (C) -- (Y);
  \draw[->, colCausal] (C) -- (R);
  \draw[->, colCausal] (A) -- (R);
  \draw[->, colCausal] (Y) -- (R);
  \draw[->, colCausal] (M) -- (R);
  \draw[->, colCausal] (T) -- (R);
  \draw[->, dashed, colU] (U) -- (C);
  \draw[->, dashed, colU] (U) to[bend left=15]  (Y);
  \draw[->, dashed, colU] (U) to[bend right=15] (R);
\end{tikzpicture}
\phantomcaption
\label{fig:banner-dag}
\end{subfigure}%
\hfill
\begin{subfigure}[t]{0.28\textwidth}
\centering
\tikz\node[font=\smaller,inner xsep=5pt, inner ysep=2pt]
  {(b) NCTS (Alg.~\ref{alg:ncts})};\par\smallskip
\begin{tikzpicture}[>=Stealth, scale=0.78, transform shape]
  \node[draw, rounded corners=2pt, fill=gray!12, thick,
        minimum width=5.3cm, minimum height=3.6cm]
    (theta) at (0, 2) {};
  \node[font=\smaller, gray!50!black] at (0, 4)
    {draw $\theta^{(k)} \sim p_{k-1}$ per ep. $k = 1,\dots,K$};
  \node[draw=colM, fill=blue!4, rounded corners=3pt, thick,
        minimum width=5cm, minimum height=1.0cm]
    (metabox) at (0, 2.9) {};
  \node[exoT, scale=0.78] (T) at (-0.8, 2.9) {$C^{(2)}$};
  \node[intM, scale=0.78] (M) at ( 0.8, 2.9) {$A^{(2)}$};
  \draw[->, colM, thick] (T) -- (M);
  \node[colM, font=\smaller] at (-1.3, 3.6) {$I_2$ meta steps};
  \node[colM, font=\smaller] at (1.8, 3.55) {($\ell=2$)};
  \node[draw=colA, fill=orange!4, rounded corners=3pt, thick,
        minimum width=5cm, minimum height=1.0cm]
    (innerbox) at (0, 0.85) {};
  \node[obs, draw=colCausal, scale=0.78] (C) at (-1.5, 0.85) {$C^{(1)}$};
  \node[intA, scale=0.78]                (A) at ( 0.0,  0.85) {$A^{(1)}$};
  \node[obs, draw=colCausal, scale=0.78] (R) at ( 1.5, 0.85) {$Y,R$};
  \draw[->, colCausal] (C) -- (A);
  \draw[->, colCausal] (A) -- (R);
  \draw[->, colCausal] (C) to[bend right=24] (R);
  \node[colA, font=\smaller] at (-1.3, 1.55) {$I_2 \times I_1$ steps};
  \node[colA, font=\smaller] at (1.8, 1.50) {($\ell=1$)};
  \draw[->, colCausal, thick]
    (metabox.south) -- (innerbox.north)
    node[midway, right, font=\tiny, colCausal!85!black] {\,$P(C^{(1)}\!\mid\! C^{(2)},A^{(2)})$};
  \node[font=\smaller, gray!50!black] at (0, 0)
    {observe $D_k$, update $p_{k-1}\!\to\!p_k$};
\end{tikzpicture}
\phantomcaption
\label{fig:banner-ncts}
\end{subfigure}%
\hfill
\begin{subfigure}[t]{0.36\textwidth}
\centering
\tikz\node[font=\smaller,inner xsep=5pt, inner ysep=2pt]
  {(c) AEGIS (Alg.~\ref{alg:ra-ncts})};\par\smallskip
\includegraphics[width=.95\linewidth]{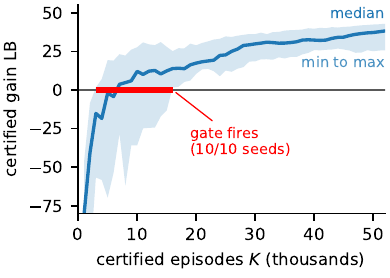}
\phantomcaption
\label{fig:banner-bench}
\end{subfigure}
\caption{\textbf{The NCCB framework at a glance} ($L\!=\!2$). (\subref{fig:banner-dag}) Hierarchical SCM problem class: meta-context $C^{(2)}$ drives meta-action $A^{(2)}$ ($\ell\!=\!2$); together they set the inner context $C^{(1)}$ for the inner action $A^{(1)}$, producing outcome $Y$ and reward $R$. (\subref{fig:banner-ncts}) NCTS (Algorithm~\ref{alg:ncts}) draws one belief per episode and acts recursively across levels, with $I_2$ meta steps and $I_2\!\times\!I_1$ inner steps; after the episode, the posterior is updated. (\subref{fig:banner-bench}) AEGIS deploys NCTS only after the per-level PRISM certificate (Theorem~\ref{thm:pac-bayes-causal}) flips, realising the \emph{progressive certified handover} the framework targets. On the two-level \bench{Crossing} bench (10 seeds, $\delta\!=\!0.05$, posterior still learning, gate suppressed), the leaf-level certified-gain lower bound (LB; median; band: min to max over seeds) first crosses $0$, where the gate would fire, between $K\!=\!3{,}000$ and $16{,}000$ (\textcolor{red}{red bar}: spread of first crossings over seeds) and reaches $+25.7$ to $+42.9$ by $K\!=\!52{,}000$; on $4$ of the $10$ seeds it dips below $0$ at least once before staying positive (App.~\ref{app:exp-crossings}). Visual encoding in (a, b): \textcolor{colM}{blue}/\textcolor{colA}{orange}, policy decisions; \textcolor{colCausal}{green}, learnable mechanisms; \textcolor{colU}{grey dashed}, latent confounder.}
\label{fig:banner}
\end{figure}

Critical sequential decision-making, treatment titration in
healthcare, input allocation in agronomy, batch configuration in
process engineering, is still entrusted to humans because the
learned-policy guarantees on offer answer the wrong question.
Frequentist regret bounds hold worst-case over all environments;
Bayesian regret bounds hold on average over a prior; neither speaks
to \emph{this} realised environment, with \emph{this} confounding,
under \emph{this} legacy controller. These domains are also naturally
hierarchical: a strategic decision generates the context for a
tactical one, which generates the context for an operational one.
Progressive transfer of control must therefore proceed level by
level, gated at each level by an empirical bound meeting a
practitioner-specified risk threshold. What is needed is a
$(1{-}\delta)$-bound that holds for the realised environment,
simultaneously over candidate policies, across both the offline
warm-start phase and the online deployment phase, and gates handover
timescale by timescale; no framework we know provides guarantees
that are simultaneously timescale-resolved, off-policy valid, and
context-specific.

We introduce the \emph{Nested Contextual Causal Bandit} (NCCB), an
$L$-level hierarchical SCM in which each level's decision causally
generates the context for the level below; the causal graph is given,
only the mechanism parameters are unknown. For this class we propose
\emph{Nested Causal Thompson Sampling} (NCTS), which draws one belief
per episode and acts on it at every nested level, and \emph{AEGIS}
(Anytime Episodic Gain-Certified Inference Schedule), a host-agnostic
deployment-time wrapper that runs progressive per-level handover off
any host with a Bayesian posterior and a reward predictor. The
analysis rests on a single causal PAC-Bayesian
value bound (Theorem~\ref{thm:pac-bayes-causal}), which certifies the
realised environment, off-policy and anytime, with a regulariser that
decomposes along the causal mechanisms.

\textbf{Contributions.} \emph{Problem class and algorithm.} We formalise the NCCB problem class (Section~\ref{sec:problem}) and propose NCTS (Section~\ref{sec:algorithm}), a recursive Thompson rule that draws one belief per episode and acts level by level under it. \emph{Theoretical guarantee.} PRISM (Theorem~\ref{thm:pac-bayes-causal}) is a causal PAC-Bayes value bound that, for any intervention scope $\mathcal{I} \subseteq [L]$, factorises both the KL regulariser (over learnable mechanisms, per-mechanism dimension) and the importance weight (over controlled levels). With logarithmic smoothing \citep{sakhi2024logarithmic,haddouche2025logarithmic} it is off-policy and anytime via Ville's inequality, with informative offline-data priors; Corollary~\ref{cor:excess-risk} gives excess risk $\widetilde{O}(1/\sqrt K)$ under its conditions. AEGIS (Algorithm~\ref{alg:ra-ncts}) deploys PRISM as a host-agnostic recipe that hands over level by level. \emph{Empirical evidence.} A hierarchical-SCM benchmark (Section~\ref{sec:experiments}) corroborates transfer, nesting and contraction at $p < 0.05$ or better with the function class held fixed (RFF-GP), and two public logs we did not simulate (HeartSteps V1, Open Bandit) run the certificate on real logged data, with zero coverage violations where coverage is checkable (Section~\ref{sec:exp-external}). The certified leaf-level gain becomes and stays positive on all $10$ seeds with the posterior still learning. With a frozen posterior, all $120$ per-level certificates at depths $L = 3$, $4$ and $5$ are positive as well (App.~\ref{app:exp-crossings}, App.~\ref{app:depth}).

\section{Related Work}
\label{sec:related}

\textbf{Causal MDPs and Causal Bandits.} \citet{lu2022efficient} introduced C-MDPs and C-UCBVI, proving regret $\widetilde{O}(HS\sqrt{ZKH})$ with known graph, where $Z$ compresses the action space via key variables; this compression requires $A$ to affect $R$ only via intermediate variables, so it is vacuous in our NCCB where $A$ directly parents $R$ (Definition~\ref{def:nccb}). \citet{mutti2023exploiting} used a partial causal graph prior with posterior sampling (C-PSRL) on factored MDPs, not directly comparable to our nested-bandit setting. \citet{lee2018structural,lee2020characterizing} introduced structural causal bandits with see/do separation and characterised optimal interventions; we extend these primitives hierarchically and add PAC-Bayes off-policy guarantees absent from their graphical analysis.

\textbf{Posterior-sampling RL and hierarchical bandits.} PSRL \citep{osband2013more,osband2014fpsrl,osband2017why} addresses posterior sampling for tabular and factored MDPs; NCTS is not directly comparable, operating over continuous-parameter SCM posteriors with two-scope (meta/inner) actions rather than tabular factorisations. Hierarchical-bandit work \citep{hong2022hierarchical, hong2022deep, guan2024improved} and dyadic RL \citep{cr_li2023dyadic} do not include causal mechanism structure.

\textbf{PAC-Bayesian learning.} \citet{seldin2012pac} established PAC-Bayes bounds for martingales; \citet{haddouche2022online} developed the online PAC-Bayes framework. \citet{haddouche2025logarithmic} extended the logarithmic smoothing of \citet{sakhi2024logarithmic} to adaptive off-policy PAC-Bayes, which we adapt to between-episode aggregation (Appendix~\ref{app:proof-pacbayes}, Remark~\ref{rem:ls-adapt}). \citet{aouali2024unified} provide a PAC-Bayesian pessimism analysis for offline contextual bandits with regularised IS; closest in technique to Theorem~\ref{thm:pac-bayes-causal}, but single-decision and without causal mechanism structure or decision timescales.

\textbf{Safe policy improvement and OPE.} \citet{thomas2015hcope} introduced high-confidence off-policy evaluation; \citet{laroche2019spibb} studied safe policy improvement (SPI) with baseline bootstrapping, which falls back on the baseline policy where uncertainty is high, the closest line in spirit to AEGIS's per-level handover gate. SPI is single-policy and value-estimate-based; AEGIS is per-level and mechanism-factorised through Theorem~\ref{thm:pac-bayes-causal} (anytime-valid via Ville's inequality), so the certificate compounds along the SCM hierarchy rather than collapsing into one policy-level guarantee. Adapting SPI / Aouali-style bounds to NCTS is non-trivial: they assume a single flat policy and an unstructured logged dataset, whereas NCCB requires a per-level certificate over an SCM-factorised policy with mechanism-wise KL decomposition. We leave a tailored SPI-for-NCTS comparison to future work.

\section{Preliminaries}
\label{sec:prelim}

\textbf{Causal models and interventions.}
A structural causal model (SCM) over endogenous variables $\mathbf{X}$ specifies, for each $X \in \mathbf{X}$, a deterministic equation $X = f_X(\mathrm{Pa}(X), \epsilon_X; \theta_X^*)$ with parent set $\mathrm{Pa}(X) \subseteq \mathbf{X}$ and exogenous noise $\epsilon_X$; the parent sets induce a DAG $\mathcal{G}$, treated as known. Pearl's hard intervention $\doIntervention(X = x)$ replaces $f_X$ by $X := x$; a \emph{soft intervention} $\doIntervention(X := h)$ generalises this by replacing $f_X$ with a new mechanism $h$ whose parents are non-descendants of $X$ in $\mathcal{G}$. A \emph{policy intervention} on an action $A$ is the soft intervention $\doIntervention(A := \pi(\cdot \mid X))$ for an observable context $X$ excluding descendants of $A$ \citep{pearl2000causality}.

\textbf{Contextual bandits.}
A contextual bandit is an SCM over context $C$, action $A$, reward $R$ with $R = f_R(C, A, \epsilon_R; \theta_R^*)$. A policy $\pi: \mathcal{X} \to \Delta(\mathcal{A})$ induces a policy intervention $\doIntervention(A := \pi(\cdot \mid C))$, with value $V(\pi) := \mathbb{E}_{C \sim P_C,\, A \sim \pi(\cdot \mid C)}[R]$. In the off-policy setting, data are collected under a behaviour policy $\mu \neq \pi$, and $V(\pi)$ is inferred via importance weighting $w_t = \pi(a_t \mid c_t)/\mu(a_t \mid c_t)$ or related estimators.

\textbf{Thompson sampling.}
Thompson sampling maintains a posterior $p_k(\theta)$ over the unknown parameters: at episode $k$, draw one \emph{belief} $\theta^{(k)} \sim p_{k-1}$, act greedily under it, observe $\mathcal{D}_k$, and Bayes-update to $p_k$. Drawing (rather than maximising) provides exploration with well-understood Bayes-regret guarantees \citep{russo2014learning, osband2013more}.

\textbf{PAC-Bayes for dependent data.}
Our off-policy analysis builds on the online PAC-Bayes framework of \citet{haddouche2022online}: for losses adapted to a filtration $(\mathcal{F}_t)_{t \geq 0}$ and an \emph{online predictive sequence} of priors $(P_t)$ (each $P_t$ being $\mathcal{F}_{t-1}$-measurable, $P_1$ data-free), a change-of-measure argument with an exponential moment bound yields concentration at a fixed horizon $T$, uniformly over posterior sequences $(Q_t)$, in the line of the martingale bounds of \citet{seldin2012pac}. Theorem~\ref{thm:pac-bayes-causal} instead adapts the logarithmic-smoothing bound \citep[Thm.~3.2]{haddouche2025logarithmic}, whose martingale and Ville's inequality give anytime validity; see Appendix~\ref{app:proof-pacbayes}.

\textbf{Notation.}
A summary of the main symbols used throughout the paper is provided in Appendix~\ref{app:notation}.

\section{Nested Contextual Causal Bandits}
\label{sec:problem}
 
We formalise the decision hierarchy informally described in Section~\ref{sec:intro}: a slow outer decision causally shapes the context of every faster inner decision. The SCM template (Definition~\ref{def:nccb}) together with the three assumptions stated below fix the problem class on which the certificate of Theorem~\ref{thm:pac-bayes-causal} operates.
 
\begin{definition}[Nested Contextual Causal Bandit (NCCB)]
\label{def:nccb}
An $L$-level NCCB is a hierarchical SCM with endogenous variable set $\mathbf{X} = \{C^{(\ell)},\, A^{(\ell)}\}_{\ell=1}^{L} \cup \{Y,\, R\}$. At level $\ell$, the agent observes the context $C^{(\ell)}$ and chooses the action $A^{(\ell)}$ via a stochastic policy $\pi^{(\ell)}$, which we treat as a \emph{soft intervention} on the SCM that replaces the natural (or legacy-logged) mechanism for $A^{(\ell)}$. $Y$ is the \emph{outcome} of the innermost decision and $R$ the reward. Levels are indexed inside-out: $\ell = 1$ is the fastest (innermost), $\ell = L$ the slowest (outermost). The level-$\ell$ decision causally generates the context $C^{(\ell-1)}$ at the next-faster level via the SCM. The outermost context $C^{(L)}$ is exogenous (no parents in $\mathcal{G}$); the action variables $A^{(\ell)}$ are policy-set; every remaining $X \in \mathbf{X}$ is generated by a learnable mechanism with parameters $\theta_X^*$, giving the canonical learnable-mechanism set $\Phi = \{C^{(\ell)} : \ell < L\} \cup \{Y, R\}$. At each episode $k$, level $\ell$ runs $I_\ell$ decisions: $I_L$ at the outermost level, $I_\ell$ per parent-level decision for $\ell < L$; the innermost level generates the per-step rewards. Total samples per episode: $N_L := \prod_{\ell=1}^L I_\ell$. The reward has one context plus one action per level as parents, plus the outcome $Y$, so $|\mathrm{Pa}_\mathcal{G}(R)| = 2L + 1$ in the scalar case.
\end{definition}
 
\begin{definition}[Intervention scope]
\label{def:intervention-scope}
An \emph{intervention scope} $\mathcal{I} \subseteq \{1, \ldots, L\}$ specifies the levels at which the agent acts under a new policy. For $\mathcal{I} \neq \emptyset$, let $\ell^\star_{\mathcal{I}} := \max(\mathcal{I})$ be the outermost new-policy level; only the contexts below $\ell^\star_{\mathcal{I}}$ are relevant for estimating the effect of actions, namely $\Phi_{\mathcal{I}} = \{C^{(\ell)} : \ell < \ell^\star_{\mathcal{I}}\} \cup \{Y, R\}$. The extremes $\ell^\star_{\mathcal{I}} = 1$ and $\ell^\star_{\mathcal{I}} = L$ give the minimal scope $\Phi_{\mathcal{I}} = \{Y, R\}$ and the full scope $\Phi_{\mathcal{I}} = \Phi$. Also, $\Phi_\emptyset := \emptyset$.
\end{definition}
 
\begin{assumption}[Causal Markov + exogenous noise]
\label{ass:markov}
The SCM satisfies the Causal Markov condition w.r.t.\ $\mathbf{X} \cup \mathbf{U}$, where $\mathbf{U} = \{U^{(\ell)}\}_{\ell=1}^L$ is a family of exogenous within-level latent variables. Each $U^{(\ell)}$ is i.i.d.\ across episodes and within-episode samples, mutually independent across $\ell$, and independent of the i.i.d.\ base noise terms $\varepsilon_X$ of the mechanisms in $\Phi$ (which combine with $U^{(\ell)}$ via additive loadings to form the residual noise of $f_X$). Soft interventions on $A^{(\ell)}$ replace its natural mechanism without altering any other mechanism in $\Phi$.
\end{assumption}
 
\begin{assumption}[Per-level i.i.d.\ contexts]
\label{ass:iid}
For each $\ell < L$, the level-$\ell$ contexts $C^{(\ell)}$ generated within one parent-level decision are i.i.d.\ given that decision and the relevant mechanism parameters, and independent of past actions taken inside the same parent-level decision. The outermost contexts are i.i.d.\ across episodes and across the $I_L$ outer decisions of an episode.
\end{assumption}
 
\begin{assumption}[Overlap + backdoor admissibility]
\label{ass:backdoor}
For each level $\ell$, any new policy $\pi^{(\ell)}$ is absolutely continuous w.r.t.\ the legacy policy $\mu^{(\ell)}$: $\pi^{(\ell)}(a \mid \cdot) > 0 \Rightarrow \mu^{(\ell)}(a \mid \cdot) > 0$. The policy context $\mathrm{Pa}_\mathcal{G}(A^{(\ell)})$ observable at level-$\ell$ decision time satisfies the backdoor criterion \citep{pearl2000causality} w.r.t.\ $(A^{(\ell)}, R)$ in $\mathcal{G}$, so the value of any new $\pi$ is identifiable from logged data under $\mu = (\mu^{(1)}, \ldots, \mu^{(L)})$.
\end{assumption}
 
\begin{remark}[Exogenous confounding]
\label{rem:confounding}
The framework allows for exogenous within-level confounding at the innermost level ($L \geq 2$): a latent factor $U^{(1)}$ shared across the mechanisms $\{C^{(1)}, Y, R\}$, e.g.\ a per-step disturbance affecting context, outcome and reward, is absorbed into the residual noise structure under Assumption~\ref{ass:markov}. $V(\pi;\theta^*)$ remains identifiable because $U^{(1)}$ is no parent of $A^{(1)}$, so Assumption~\ref{ass:backdoor} holds. Outer levels carry at most one learnable mechanism each, so within-level sharing is moot there. Cross-level confounding (a single latent acting on variables at multiple levels) lies outside this model class; if it also drives an action, handling it requires non-backdoor identification (e.g.\ instrumental or proximal variables), which we leave to future work.
\end{remark}

\textbf{Two roles of Assumption~\ref{ass:iid}.}
The certificate needs only that per-episode conditional means are constant and that episodes are independent given $\mathcal{H}_0$ (App.~\ref{app:notation}; within-episode dependence is permitted, Remark~\ref{rem:no-within-indep}); its within-level content (no state transitions inside a level) is needed by the planning recursion~\eqref{eq:R-recursion} and the conditional-mean target of Lemma~\ref{lem:episode-unbiased}. Within a level the certificate does require \emph{exogeneity}: the context process must not be driven by the within-level actions. Appendix~\ref{app:ext-limitations} measures both sides (exogenous AR(1) state at $\rho = 0.9$: 0 of 60 false certifications; endogenous state, engineered candidate, AR(1) coefficient $0.95$: 40 of 40); extending the certificate to within-level dynamics is left to future work.

\section{Nested Causal Thompson Sampling}
\label{sec:algorithm}
\begin{algorithm}[t]
\small
\caption{NCTS}
\label{alg:ncts}
\KwIn{$\mathcal{G}$ (known); prior $p_0$; scope $\mathcal{I}$;
      legacy $(\mu^{(\ell)})_{\ell \notin \mathcal{I}}$;
      $K$; $(I_\ell)_{\ell=1}^L$.}
\KwOut{Posteriors $(p_k)$; within-episode policies $\{\pi^{(\ell)}_k\}$.}
\SetKwProg{Fn}{Procedure}{:}{end}
\SetKwFunction{Traverse}{Traverse}
\Fn{\Traverse{$\ell$, $\tilde c^{(\ell+1)}$}}{
  \For{$i_\ell = 1, \ldots, I_\ell$}{
    Sample $C^{(\ell)} \!\sim\! P_{C^{(\ell)}}(\cdot\!\mid\!\tilde c^{(\ell+1)};\theta^*)$;\ extend $\tilde c^{(\ell)} \!\leftarrow\! (\tilde c^{(\ell+1)}, C^{(\ell)})$\;
    \lIf{$\ell \in \mathcal{I}$}{$a^{(\ell)} \!\leftarrow\! \pi^{(\ell)}_k(\tilde c^{(\ell)})$ via \eqref{eq:pi-ell} under $(\mathcal{G}, \theta^{(k)})$}\lElse{$a^{(\ell)} \!\sim\! \mu^{(\ell)}(\cdot\!\mid\!\tilde c^{(\ell)})$}
    $\tilde c^{(\ell)} \!\leftarrow\! (\tilde c^{(\ell)}, a^{(\ell)})$;\
    \lIf{$\ell > 1$}{\Traverse{$\ell - 1$, $\tilde c^{(\ell)}$}}\lElse{observe $Y, R$; append $(\tilde c^{(\ell)}, Y, R)$ to $\mathcal{D}_k$}
  }
}
\For{$k = 1, \ldots, K$}{
  Draw $\theta^{(k)} \!\sim\! p_{k-1}$ (held fixed for episode $k$); \Traverse{$L$, $()$}; update $p_k$ from $\mathcal{D}_k$ (per-mechanism: NIG / online Gibbs).
}
\end{algorithm}

NCTS is the recursive Thompson-sampling rule on the $L$-level NCCB, parametrised by an intervention scope $\mathcal{I} \subseteq \{1, \ldots, L\}$ identifying the levels at which the agent plans (Definition~\ref{def:intervention-scope}). Algorithm~\ref{alg:ncts} summarises one episode; §\ref{sec:within} gives the formal scope-aware policy and recursive value function, §\ref{sec:between} the per-mechanism Bayesian update. The $L = 2$ specialisation matching the Sec.~\ref{sec:experiments} experiments (innermost-level closed form, outer-level Thompson rule) is summarised at the end of this section. An expanded version, with the D-optimal alternative, is in Appendix~\ref{app:instantiation-example}; instances up to $L = 5$ are in Appendix~\ref{app:depth}.

\subsection{Within-Episode Bandit}
\label{sec:within}

\textbf{Recursive value function.}
Fix an intervention scope $\mathcal{I} \subseteq \{1, \ldots, L\}$. At the start of episode $k$, draw one belief $\theta^{(k)} \sim p_{k-1}(\theta)$ and hold it fixed throughout the episode. The agent traverses the $L$ levels from $\ell = L$ down to $\ell = 1$; at each level $\ell$, after observing the level-$\ell$ context $C^{(\ell)}$, the level-$\ell$ action $a^{(\ell)} \in \mathcal{A}^{(\ell)}$ is set by the scope-aware policy
\begin{equation}
  \pi^{(\ell)}_k\!\bigl(\tilde c^{(\ell)}\bigr)
  = \begin{cases}
    \arg\max_{a \in \mathcal{A}^{(\ell)}}\, \hat R^{(\ell)}_{\mathcal{I}}\!\bigl(a,\, \tilde c^{(\ell)};\, \theta^{(k)}\bigr), & \ell \in \mathcal{I}, \\[2pt]
    a \sim \mu^{(\ell)}\!\bigl(\,\cdot\, \mid \tilde c^{(\ell)}\bigr), & \ell \notin \mathcal{I},
  \end{cases}
  \label{eq:pi-ell}
\end{equation}
where the level-$\ell$ \emph{augmented context} $\tilde c^{(\ell)} := (C^{(L)}, a^{(L)}, \ldots, a^{(\ell+1)}, C^{(\ell)})$ collects every higher-level context and action together with the current level-$\ell$ context, $\tilde c^{(\ell-1)} := (\tilde c^{(\ell)}, a, C^{(\ell-1)})$ extends $\tilde c^{(\ell)}$ with the level-$\ell$ choice $a$ and the freshly drawn level-$(\ell-1)$ context, and $\hat R^{(\ell)}_{\mathcal{I}}(a, \tilde c^{(\ell)}; \theta)$ is the expected reward over all descendant levels of $\ell$ under belief $\theta$ when controlled levels follow $\pi^{(\ell')}_k$ and uncontrolled levels follow the legacy $\mu^{(\ell')}$:
\begin{equation}
  \hat R^{(\ell)}_{\mathcal{I}}\!\bigl(a,\, \tilde c^{(\ell)};\, \theta\bigr)
  = \begin{cases}
    \mathbb{E}\!\bigl[R \,|\, A^{(1)} {=} a,\, \tilde c^{(1)};\, \theta\bigr], & \ell = 1, \\[6pt]
    \displaystyle \mathbb{E}\!\bigl[\, \max_{a' \in \mathcal{A}^{(\ell-1)}} \hat R^{(\ell-1)}_{\mathcal{I}}\!\bigl(a',\, \tilde c^{(\ell-1)};\, \theta\bigr) \,\big|\, A^{(\ell)} {=} a,\, \tilde c^{(\ell)};\, \theta\bigr], & \ell{-}1 \in \mathcal{I}, \\[6pt]
    \displaystyle \mathbb{E}\!\bigl[\, \mathbb{E}_{a' \sim \mu^{(\ell-1)}}\!\bigl[\hat R^{(\ell-1)}_{\mathcal{I}}\!\bigl(a',\, \tilde c^{(\ell-1)};\, \theta\bigr)\bigr] \,\big|\, A^{(\ell)} {=} a,\, \tilde c^{(\ell)};\, \theta\bigr], & \ell{-}1 \notin \mathcal{I}.
  \end{cases}
  \label{eq:R-recursion}
\end{equation}
Equation~\eqref{eq:R-recursion} computes the level-$\ell$ expected reward recursively over the hierarchy: by Assumption~\ref{ass:iid}, no state transitions occur within a level, so the level-$\ell$ value is the expected level-$(\ell{-}1)$ optimum (or legacy mean) over the freshly drawn descendant context, evaluated in closed form for linear Gaussian ANMs and by Monte Carlo otherwise (Appendix~\ref{app:instantiation-example}).

\subsection{Between-Episode Bayesian Update}
\label{sec:between}

The learnable mechanisms are $\Phi = \{C^{(\ell)} : \ell < L\} \cup \{Y, R\}$ (Definition~\ref{def:nccb}). We adopt a \emph{mechanism-factorised working model} for prior and likelihood: $p_0(\theta) = \prod_{X \in \Phi} p_0(\theta_X)$ and $q(\mathcal{D}_k \,|\, \theta, \mathcal{G}) = \prod_{X \in \Phi} p\bigl(\mathcal{D}_k^{(X)} \,|\, \theta_X, \mathrm{Pa}(X)\bigr)$, yielding the per-mechanism update $p_k(\theta_X) \propto p\bigl(\mathcal{D}_k^{(X)} \,|\, \theta_X, \mathrm{Pa}(X)\bigr) \cdot p_{k-1}(\theta_X)$, with $\mathcal D_k = \bigl\{\bigl(C^{(\ell)}_{k,\mathbf{i}}, a^{(\ell)}_{k,\mathbf{i}}\bigr)_{\ell=1}^L,\, Y_{k,\mathbf{i}},\, R_{k,\mathbf{i}}\bigr\}$ indexed by the within-episode multi-index $\mathbf{i} = (i_1, \ldots, i_L)$. For linear ANMs with Gaussian noise, each factor is a closed-form NIG-conjugate update, no MCMC required. For nonlinear ANMs, we use the online Gibbs posterior of \citet{haddouche2022online}, which is PAC-Bayes optimal and absorbs any approximation error into the KL regulariser of Theorem~\ref{thm:pac-bayes-causal} without an additive penalty.

\begin{remark}[Working model vs.\ true likelihood]
\label{rem:working-model}
Under no shared latent confounding (i.e.\ at most one $X \in \Phi$ has $\delta_{UX} \neq 0$ in Example~\ref{ex:l2-instantiation}) and with correctly specified mechanisms, the working model $q(\mathcal{D}_k \mid \theta, \mathcal{G})$ is the true likelihood by the Causal Markov Condition \citep{pearl2000causality}, and $p_k$ the exact posterior. Under shared $U^{(1)}$ across $\{C^{(1)}, Y, R\}$ (Remark~\ref{rem:confounding}), the marginal likelihood does not factorise mechanism-wise, so $p_k$ is misspecified relative to it. PAC-Bayes absorbs this misspecification: Theorem~\ref{thm:pac-bayes-causal} certifies $V(\pi^{\mathcal{I}}_Q; \theta^*)$ via the IS estimator $\widehat V^\mathrm{LS}_\mathrm{pes}$, which depends on the data only through the level-wise importance weights and the realised reward, not on the working likelihood. Identifiability of $V$ from logged data is supplied by Assumption~\ref{ass:backdoor}, which holds regardless of the working-model factorisation; what the misspecification can do is widen the certificate (Effect~1, Appendix~\ref{app:offline-effect}), not invalidate it.
\end{remark}

\begin{remark}[Known vs.\ unknown]
\label{rem:known-unknown}
The agent is given the functional forms $f_X$ of the learnable mechanisms and an \emph{I-map} $\mathcal{G}$ of the data-generating process: a DAG over $\mathbf{X}$ that contains every true causal edge, possibly with additional spurious ones. This is strictly weaker than knowing the minimal causal graph; every d-separation read off $\mathcal{G}$ holds in the true distribution, so the backdoor criterion (Assumption~\ref{ass:backdoor}) and the mechanism factorisation (Definition~\ref{def:nccb}) remain sound while the practitioner is freed from resolving internal structure of multi-dimensional variables. The agent is also given the legacy logging policies $\mu = (\mu^{(\ell)})_{\ell=1}^L$ that produced the offline data and that act at every level $\ell \notin \mathcal{I}$ during deployment; outside $\mathcal{I}$, $\mu^{(\ell)}$ is treated as a known component of the environment when planning. The only unknown is $\theta^* = (\theta_X^*)_{X \in \Phi}$, the parameters of the learnable mechanisms.
In the motivating domains this superset structure is documented and does not need to be discovered (App.~\ref{app:ext-limitations}).
\end{remark}

\textbf{Worked instantiation at $L = 2$.}
Sec.~\ref{sec:experiments} experiments run at $L = 2$ unless stated otherwise, with $T = C^{(2)}$, $M = A^{(2)}$, $C = C^{(1)}$, $A = A^{(1)}$, $J = I_2$, $I = I_1$ (Figure~\ref{fig:banner-dag}), so $\tilde c^{(2)} = T$ and $\tilde c^{(1)} = (T, M, C) =: x$. The inner recursion \eqref{eq:R-recursion} reduces to a contextual-bandit argmax and the meta-rule integrates over the freshly drawn inner context; Appendix~\ref{app:instantiation-example} gives the formulas, the confounder absorption and a D-optimal meta-rule.

\section{Causal PAC-Bayesian Bound and Per-Level Certification}
\label{sec:results}

The causal factorisation of the posterior (Section~\ref{sec:between}) and the level-factored intervention scope $\mathcal{I} \subseteq \{1, \ldots, L\}$ (Definition~\ref{def:intervention-scope}) together enable a PAC-Bayesian value bound for any level-wise mixture of legacy and new policies. The bound features two parallel factorisations: the KL regulariser sums over the post-intervention learnable mechanisms $\Phi_{\mathcal{I}}$, and the importance weight is a product across the controlled levels $\ell \in \mathcal{I}$. It accommodates linear and nonlinear ANM mechanisms and does not require the reward to decompose additively. We use $Q$ to denote a generic factorised distribution over $\theta$ (the certified ``posterior'' in the PAC-Bayes sense); the bound holds simultaneously over all such $Q$, with the running Bayesian posterior $p_k$ from Section~\ref{sec:between} as the natural choice.

\begin{assumption}[Bounded rewards and per-level bounded importance weights]
\label{ass:bounded-weights}
\emph{(i) Rewards.} The reward is bounded: $R_t \in [0, R_{\max}]$ almost surely, with $0 < R_{\max} < \infty$ known to the practitioner. Raw rewards in a known range $[R_\mathrm{lo}, R_\mathrm{hi}]$ enter Eq.~\eqref{eq:ls-estimator} after the affine map $\tilde R_t = (R_t - R_\mathrm{lo})/(R_\mathrm{hi} - R_\mathrm{lo})$ to $[0, 1]$; an unbounded reward is first clipped to that range, so that (i) holds by construction for the clipped reward. \emph{(ii) Importance weights.} For each level $\ell \in \mathcal{I}$, there exists $B^{(\ell)}_w < \infty$ such that $\sup_{\theta,\, a^{(\ell)},\, \tilde c^{(\ell)}} \pi^{(\ell)}_\theta(a^{(\ell)} \mid \tilde c^{(\ell)}) / \mu^{(\ell)}(a^{(\ell)} \mid \tilde c^{(\ell)}) \leq B^{(\ell)}_w$. \emph{(iii) Loss range.} Combining (i) and (ii), the per-step loss $\mathcal{L}^{\mathcal{I}}_t(\pi^{\mathcal{I}}_\theta) = w^{\mathcal{I}}_t(\theta) \cdot R_t$ satisfies $\mathcal{L}^{\mathcal{I}}_t \in [0, B^{\mathcal{I}}]$ with $B^{\mathcal{I}} := R_{\max} \prod_{\ell \in \mathcal{I}} B^{(\ell)}_w$, a constant fixed before the data. In the on-policy case ($\pi^{(\ell)}_\theta = \mu^{(\ell)}$ for all $\ell \in \mathcal{I}$), $B^{(\ell)}_w = 1$ and $B^{\mathcal{I}} = R_{\max}$.
\end{assumption}

\begin{theorem}[PRISM: Per-mechanism Risk Inequality with Scope Mixing]
\label{thm:pac-bayes-causal}
Fix any intervention scope $\mathcal{I} \subseteq \{1, \ldots, L\}$. The hybrid policy $\pi^{\mathcal{I}}_\theta$ has $\pi^{(\ell)}_\theta = \mu^{(\ell)}$ at $\ell \notin \mathcal{I}$ (Eq.~\eqref{eq:pi-ell}); let $\pi^{\mathcal{I}}_Q$ denote the randomised version that draws $\theta \sim Q$ once per episode, so that $V(\pi^{\mathcal{I}}_Q; \theta^*) = \mathbb{E}_{\theta \sim Q}[V(\pi^{\mathcal{I}}_\theta; \theta^*)]$. Index the certification dataset $\mathcal{D}_\mathrm{cert}$ by episode $k = 1, \ldots, K$ and within-episode multi-index $\mathbf{i} \in [N_L]$. The level-factored importance weight is the product over all $L$ levels of (deployment policy)/(behaviour policy); when $\mathcal{D}_\mathrm{cert}$ was collected under the legacy $\mu = (\mu^{(\ell)})_{\ell=1}^L$, the legacy/legacy factors at $\ell \notin \mathcal{I}$ cancel, leaving
\begin{equation}
  w^{\mathcal{I}}_t(\theta)
  := \prod_{\ell=1}^L \frac{\pi^{(\ell)}_\theta\bigl(a^{(\ell)}_t \mid \tilde c^{(\ell)}_t\bigr)}{\mu^{(\ell)}\bigl(a^{(\ell)}_t \mid \tilde c^{(\ell)}_t\bigr)}
  \;=\; \prod_{\ell \in \mathcal{I}} \frac{\pi^{(\ell)}_\theta\bigl(a^{(\ell)}_t \mid \tilde c^{(\ell)}_t\bigr)}{\mu^{(\ell)}\bigl(a^{(\ell)}_t \mid \tilde c^{(\ell)}_t\bigr)}.
  \label{eq:hybrid-weight}
\end{equation}
Define the per-sample hybrid loss $\mathcal{L}^{\mathcal{I}}_{k, \mathbf i}(\pi^{\mathcal{I}}_\theta) := w^{\mathcal{I}}_{k, \mathbf i}(\theta) \cdot R_{k, \mathbf i}$ and the \emph{episode-mean loss} $\bar{\mathcal{L}}^{\mathcal{I}}_k(\pi^{\mathcal{I}}_\theta) := \frac{1}{N_L}\sum_{\mathbf i \in [N_L]}\mathcal{L}^{\mathcal{I}}_{k, \mathbf i}(\pi^{\mathcal{I}}_\theta) \in [0, B^{\mathcal{I}}]$. Under Assumptions~\ref{ass:markov}--\ref{ass:bounded-weights}, for any data-free prior $p_0 = \prod_{X \in \Phi} p_{0,X}$, any $\delta \in (0,1)$ and any fixed $\lambda \in (0, 1/B^{\mathcal{I}})$, in the off-policy phase (legacy $\mu$ fixed across episodes), with probability at least $1 - \delta$ simultaneously over all factorised $Q = \prod_{X \in \Phi_{\mathcal{I}}} Q_X$ and all $K \geq 1$:
\begin{equation}
  V(\pi^{\mathcal{I}}_Q;\,\theta^*)
  \geq \widehat{V}^{\mathrm{LS}}_{\mathrm{pes}}(\pi^{\mathcal{I}}_Q)
  - \frac{\sum_{X \in \Phi_{\mathcal{I}}}
    \KL(Q_X\|p_{0,X}) + \log(1/\delta)}{\lambda K},
  \label{eq:pac-bayes-causal}
\end{equation}
where the pessimistic Logarithmic Smoothing value estimator is
\begin{equation}
  \widehat{V}^{\mathrm{LS}}_{\mathrm{pes}}(\pi^{\mathcal{I}}_Q)
  := B^{\mathcal{I}}
  + \frac{1}{\lambda K}\sum_{k=1}^K
    \mathbb{E}_{\theta \sim Q}\!\Bigl[\log\!\Bigl(1 - \lambda\,\bigl(B^{\mathcal{I}} - \bar{\mathcal{L}}^{\mathcal{I}}_k(\pi^{\mathcal{I}}_\theta)\bigr)\Bigr)\Bigr].
  \label{eq:ls-estimator}
\end{equation}
The certified value is the expected reward of $\pi^{\mathcal{I}}_Q$ under the true environment $\theta^*$; the KL regulariser sums over $\Phi_{\mathcal{I}}$ only, mirroring the cancellation of the IS weight at $\ell \notin \mathcal{I}$. For ANMs whose mechanism posteriors satisfy $\KL(Q_X \| p_{0,X}) = O(d_X (\sum_{\beta \in \mathrm{Pa}(X)} d_\beta + 1) \log K)$ (a regularity condition met by Gaussian-conjugate updates), the total regulariser scales with the per-mechanism parameter dimension rather than the full joint space. Under further conditions, the excess-risk rate of $\widetilde{O}(1/\sqrt{K})$ in episodes is Corollary~\ref{cor:excess-risk}; the full proof is in Appendix~\ref{app:proof-pacbayes}.
\end{theorem}

\textit{Proof sketch.}
\emph{(1) Causal KL decomposition}: the mechanism-wise factorisation makes $Q$ and $p_0$ product distributions over $\Phi$; the chain rule gives $\KL(Q\|p_0) = \sum_X \KL(Q_X\|p_{0,X})$, and for $X \notin \Phi_{\mathcal{I}}$ the convention $Q_X = p_{0,X}$ collapses the sum to $\sum_{X \in \Phi_{\mathcal{I}}}$. \emph{(2) Hybrid IS unbiasedness}: actions are drawn level by level, so the chain rule along the level-ordering yields the level-factored Radon--Nikodym derivative $dP_{\mathrm{hybrid}}/dP_{\mathrm{legacy}} = w^{\mathcal{I}}$ per sample; together with the level-wise backdoor admissibility of Assumption~\ref{ass:backdoor} on $\mathcal{I}$, every per-sample loss has constant conditional mean $\mathbb{E}[\mathcal{L}^{\mathcal{I}}_{k, \mathbf i} \mid \mathcal{H}_{k-1}] = V(\pi^{\mathcal{I}}_\theta; \theta^*)$, and so does the episode mean $\mathbb{E}[\bar{\mathcal{L}}^{\mathcal{I}}_k \mid \mathcal{H}_{k-1}] = V(\pi^{\mathcal{I}}_\theta; \theta^*)$. \emph{(3) Episode-level LS supermartingale}: applying the logarithmic-smoothing construction of \citet[Thm.~3.2]{haddouche2025logarithmic} to the episode-mean gap $\Delta\bar{\mathcal{L}}^{\mathcal{I}}_k := B^{\mathcal{I}} - \bar{\mathcal{L}}^{\mathcal{I}}_k \in [0, B^{\mathcal{I}}]$ across the $K$ off-policy episodes, which are conditionally independent given $\mathcal{H}_0$ (Assumption~\ref{ass:iid}), and combining with Ville's inequality and change of measure (Appendix~\ref{app:proof-pacbayes}, Remark~\ref{rem:ls-adapt}), yields the bound.
\qed

\textbf{Interpreting the LS estimator, and what is certified.}
Since $-\log(1 - u) \geq u$ for $u < 1$, the log-tempered mean \eqref{eq:ls-estimator} lies below the IPS mean, and \eqref{eq:pac-bayes-causal} is a high-probability lower bound. As $\lambda \to 0$ the estimator recovers plain IPS; $\lambda > 0$ adds pessimism that grows as the loss falls below $B^{\mathcal{I}}$. The logarithm makes the underlying process a supermartingale (Lemma~\ref{lem:supermartingale}), hence anytime validity. Unlike Thompson-sampling regret analyses, which control the exploration cost of a sampling rule, Theorem~\ref{thm:pac-bayes-causal} certifies a given posterior from logged data. For the full scope, Gaussian priors and a value Lipschitz in $\theta$, Corollary~\ref{cor:excess-risk} bounds the excess risk of the certificate-maximising posterior by $O(p_\Phi \log(K{+}1) / \sqrt{K})$, with a constant affine in $B^{\mathcal{I}}$. Cumulated over $K$ episodes this has the $\sqrt{K}$ shape of Thompson-sampling analyses, with $p_\Phi$ in place of the joint dimension, but for a posterior fitted to logs, not the regret of NCTS.

\textbf{Scope and extensions.} Corollary~\ref{cor:adaptive} extends PRISM to $\mathcal{H}_{k-1}$-measurable behaviour $(\mu_k)$ with overlap.

\subsection{AEGIS: Anytime Episodic Gain-Certified Inference Schedule}
\label{sec:ra-ncts}

\textbf{Per-level certification and safe-deployment selection.}
\label{sec:per-level}
The bound \eqref{eq:pac-bayes-causal} promotes the deployment scope $\mathcal{I}$ from an all-or-nothing decision to a design choice the practitioner optimises against the certificate. Since $\widehat{V}^{\mathrm{LB}}_{\mathcal{I}}$, the right-hand side of \eqref{eq:pac-bayes-causal}, is a high-probability lower bound on the realised value of deploying $\pi^{\mathcal{I}}_Q$, the safest scope maximises the guarantee:
\begin{equation}
  \mathcal{I}^\star \;=\; \arg\max_{\mathcal{I} \in \mathfrak{I}}\, \widehat{V}^{\mathrm{LB}}_{\mathcal{I}}(\pi^{\mathcal{I}}_Q),
  \label{eq:safe-deploy}
\end{equation}
not necessarily the scope with the highest expected reward, since the multiplicative compounding $B^{\mathcal{I}} \leq R_{\max} \prod_{\ell \in \mathcal{I}} B^{(\ell)}_w$ in $\lambda < 1/B^{\mathcal{I}}$ can loosen the certificate faster than the mean improves.

Single-level scopes $\mathcal{I} = \{\ell\}$ are the natural building blocks of the progressive workflow advertised in Section~\ref{sec:intro}: deploy at the safest level, collect fresh data under the hybrid behaviour, re-apply Theorem~\ref{thm:pac-bayes-causal} for the next handover candidate.
Note that the per-level certificates need to be recomputed on fresh data after the deployment of any intervention-scoped policy.

\textbf{Computing the per-scope certificates.}
Every candidate scope $\mathcal{I} \in \mathfrak{I}$ is certified on the \emph{same} $\mathcal{D}_\mathrm{cert}$, recomputing only $w^{\mathcal{I}}$ \eqref{eq:hybrid-weight}, $B^{\mathcal{I}}$ and the KL sum over $\Phi_{\mathcal{I}}$. The temperature is
$\lambda^\star_{\mathcal{I}} = \bigl(2\,(\sum_{X \in \Phi_{\mathcal{I}}} \KL(Q_X \| p_{0,X}) + \log(1/\delta)) / (K\, \widehat{S}^2_{\mathcal{I}})\bigr)^{1/2}$
with $\widehat{S}^2_{\mathcal{I}}$ the second moment of the episode-mean gap, clipped inside $\lambda < 1/B^{\mathcal{I}}$ and computed on the prior split (Remark~\ref{rem:lam-grid}); all scopes together cost a $\log|\mathfrak{I}|$ union-bound term and linear enumeration (Appendix~\ref{app:proof-hybrid}).

\begin{wrapfigure}[29]{r}{0.52\textwidth}
\vspace{-1.5\baselineskip}
\small
\begin{algorithm}[H]
\caption{AEGIS schedule (host-agnostic, $L$-level). $\delta$ is split over levels, epochs and checks (App.~\ref{app:exp-crossings}); joint-scope diagnostic $G^{[L]}_k$ and design notes: App.~\ref{app:ra-ncts-design}.}
\label{alg:ra-ncts}
\KwIn{Host: $p_0$, $\mathsf{HostUpdate}(p_{k-1}, \mathcal{D}_k){\to}p_k$, predictor $\hat r_\theta$.\
Schedule: $K, L, (I_\ell)_{\ell=1}^L$.\
AEGIS: legacy $(\mu^{(\ell)}_\mathrm{legacy})$, thresholds $(\varepsilon_\ell)$, $\delta, \eta_{\mathrm{clip}}$.}
\KwOut{Sticky handover $(\Sigma_k)$; posteriors $(p_k)$.}
$\Sigma_1$, buffers $\!\leftarrow\! \emptyset$;\ \ $\lambda \!\leftarrow\!$ \eqref{eq:lam-star} on prior split, declared KL\;
\SetKwProg{Fn}{Procedure}{:}{end}
\SetKwFunction{Traverse}{Traverse}
\Fn{\Traverse{$\ell$, $\tilde c^{(\ell+1)}$}}{
  \For{$i_\ell = 1, \ldots, I_\ell$}{
    Sample $C^{(\ell)} \!\sim\! P_{C^{(\ell)}}(\cdot\!\mid\!\tilde c^{(\ell+1)};\theta^*)$;\ extend $\tilde c^{(\ell)} \!\leftarrow\! (\tilde c^{(\ell+1)}, C^{(\ell)})$\;
    \uIf{$\ell \in \Sigma_k$}{$a^{(\ell)} \sim \pi^{(\ell)}_{p_{k_\ell}}$\;}
    \Else{$a^{(\ell)} \sim \mu^{(\ell)}_\mathrm{legacy}(\cdot\!\mid\!\tilde c^{(\ell)})$; append to $\mathcal{D}^{\mu,(\ell)}$, also $\mathcal{D}^{\mu,[L]}$ if $\Sigma_k\!=\!\emptyset$\;}
    $\tilde c^{(\ell)} \!\leftarrow\! (\tilde c^{(\ell)}, a^{(\ell)})$;\
    \lIf{$\ell > 1$}{\Traverse{$\ell - 1$, $\tilde c^{(\ell)}$}}\lElse{observe $Y, R$; append $(\tilde c^{(\ell)}, Y, R)$ to $\mathcal{D}_k$}
  }
}
\For{$k = 1, \ldots, K$}{
  \Traverse{$L$, $()$};\ $p_k \!\leftarrow\! \mathsf{HostUpdate}(p_{k-1}, \mathcal{D}_k)$;\ $\Sigma_{k+1} \!\leftarrow\! \Sigma_k$\;
  \For{$\ell \in [L]\!\setminus\!\Sigma_k$ \textbf{\upshape with} $|\mathcal{D}^{\mu,(\ell)}|>0$}{
    \lIf{$G^{(\ell)}_k(\pi_{p_k};\,\mathcal{D}^{\mu,(\ell)}) \geq \varepsilon_\ell$}{$\Sigma_{k+1} \!\leftarrow\! \Sigma_{k+1} \cup \{\ell\}$ \emph{(sticky)};\ $k_\ell \!\leftarrow\! k$;\ clear $\mathcal{D}^{\mu,(\ell')}$, $\ell' \!\neq\! \ell$}
  }
}
\end{algorithm}
\end{wrapfigure}
\textbf{The AEGIS recipe.} AEGIS schedules \eqref{eq:safe-deploy}
on any host with posterior $p_k$, KL to data-free $p_0$, and
per-mechanism predictor $\hat r_\theta$, via three switches: an
optional leaf LCB policy
\begin{align}
  \pi_k^{\mathrm{RA}}(x) &=\; \arg\max_{a \in \mathcal{A}^{(1)}}
    \widehat\mu_k(a, x) - \kappa_k \widehat\sigma_k(a, x), \label{eq:ra-inner}\\
  \kappa_k &= \sqrt{\tfrac{2(\KL(p_{k-1}\|p_0) + \log(1/\delta))}{k}},
  \label{eq:kappa}
\end{align}
($x = \tilde c^{(1)}$; $\widehat\mu_k, \widehat\sigma_k$: mean, std of
$\hat r_\theta(a,x)$ over $n_\mathrm{RA}$ draws $\theta \sim p_{k-1}$); a closed-form
bound-minimising
\begin{equation}
  \lambda^\star_k \;=\; \min\!\left(\tfrac{1-\eta_{\mathrm{clip}}}{B^{\mathcal{I}}},\;
    \sqrt{\tfrac{2 C_k}{k\, S^2_k}}\,\right),
  \label{eq:lam-star}
\end{equation}
with $C_k = \KL(p_k\|p_0)+\log(1/\delta)$, episode-mean second-moment
$S_k^2$, and clipping margin $\eta_{\mathrm{clip}}$ (App.~\ref{app:ra-ncts-design}; evaluated on the prior split,
Remark~\ref{rem:lam-grid}); and a sticky per-level handover gate
$\Sigma_k$ that adds $\ell$ to $\Sigma_{k+1}$ once
$G^{(\ell)}_k$, the right-hand side of \eqref{eq:gain-bound} at level $\ell$,
reaches $\varepsilon_\ell$. A flipped level keeps the $\pi_{p_k}$ certified at its flip; no reported run plays \eqref{eq:ra-inner} after a flip.
The gate fires on the $(1{-}\delta)$ \emph{lower bound} of the gain, so it can stay closed while $V(\pi^{(\ell)}_{p_k})$ already exceeds $V(\mu^{(\ell)}_\mathrm{legacy})$. This conservativeness follows from the log directly, as legacy data without randomisation over the candidate's support caps the certificate at $\widehat V(\mu)$ for every $K$ (a deterministic legacy has $w \le 1$), so overlap in the log is a design-time precondition that a longer run cannot supply. Given it, the remaining requirement is computable before data, as the gain must exceed the approximate median firing point $G_{\min}(K) = c/\sqrt{K}$ (Appendix~\ref{app:offline-effect}, Effect~5).
Sec.~\ref{sec:experiments} evaluates the AEGIS gate on NCTS runs;
rationale, hosts and notes: App.~\ref{app:ra-ncts-design}.

\subsection{Approximately Correct Graphs: Error Taxonomy and Corruption Study}
\label{sec:graph-error}

\textbf{Four error cases and the logging condition.}
We distinguish four ways the supplied graph, or the log it is applied to, can be wrong (Appendix~\ref{app:graph-error-full}). \emph{(i)~Extra edges} cost compression, not validity (Remark~\ref{rem:known-unknown} requires only an I-map). \emph{(ii)~Working-model misspecification} can widen the certificate without invalidating it (Remark~\ref{rem:working-model}). \emph{(iii)~A missing observation} no policy acted on cannot confound the action, so Lemma~\ref{lem:episode-unbiased} still holds, even if the coarser graph is no I-map; it costs performance, not validity. \emph{(iv)~Missing logging}, a log that omits a variable the behaviour policy acted on or replaces it by a post-action measurement, breaks the identification step of Lemma~\ref{lem:episode-unbiased} and alone of the four can falsely certify. The requirement thus reduces from ``know the causal graph'' to a checkable \emph{logging condition}: every reweighted dataset must record, at their decision-time values, the variables its behaviour policy acted on.

\textbf{Corruption study.}
On \bench{$\Gamma$-SCM} (10 seeds, $\delta = 0.05$; Table~\ref{tab:corruption-study}, App.~\ref{app:graph-error-full}), every valid arm stays at 0/10 false certifications, including a supplied graph that is not an I-map. Only mode 3a (confounded logging) fails (up to 10/10) and opens the gain gate, on 4 seeds for a policy with negative true gain. Corrupted certificates are \emph{not self-diagnosing}, so the logging condition is a design-time requirement, not a runtime check.

\textbf{Logging requirement and post-handover buffers.}
The gate needs logged $(x_t, a^\mu_t, R_t)$ tuples and their propensities; for a deterministic legacy these are $1$ and the IS ratio in \eqref{eq:ls-estimator} is $\pi^{(\ell)}_\theta(a^\mu_t \mid x_t)$ for each $\theta \sim p_k$. Algorithm~\ref{alg:ra-ncts} clears the other levels' buffers $\mathcal{D}^{\mu,(\ell)}$ at a flip, so the next gate certifies the current hybrid; uncleared buffers certify a time average over pre- and post-flip hybrids (Corollary~\ref{cor:adaptive}), frozen ones the full-legacy hybrid (App.~\ref{app:ra-ncts-design}).

\section{Experiments}
\label{sec:experiments}

\captionsetup{font=small,skip=3pt}
\captionsetup[sub]{font=footnotesize,skip=2pt}

\subsection{Setup}
\label{sec:exp-setup}

The $L{=}2$ benchmarks instantiate one parametric SCM
(\bench{$\Gamma$-SCM}; Table~\ref{tab:bench-registry} lists every benchmark) with meta-action $M$ outer, inner-action $A$
inner (both bounded $[-2, 2]$), and latent noise $U$; unless
stated otherwise, every bench runs at the \bench{$\Gamma$-Stress} preset
($\eta{=}1$, $\sigma_Q{=}3$, $J{=}1$, $I{=}20$) with the
\texttt{RFF-GP-Gibbs} posterior backend. The four body bandits share a
common RFF-GP function class and differ only in reward model and
commit shape: \arm{Flat-TS} (joint regression
$(T,M,C,A){\to}R$, no SCM), \arm{Factored-TS} (SCM-factorised
$f_C, f_Y, f_R$, flat commit), \arm{Joint-Commit} (SCM-factorised
with $(M,A)$ committed jointly at meta-start), and \arm{NCTS}
(full NCTS, Algorithm~\ref{alg:ncts}); Table~\ref{tab:exp-arms} lists all arms.
Unless stated otherwise, numbers use $10$ seeds, $K{=}2000$ episodes,
mean$\,\pm\,1\sigma$ across seeds; significance is Welch's t-test
with reported $d$ and $p$. Deviations are stated inline; details: App.~\ref{app:exp-setup}.

\subsection{Constructive Ablations}
\label{sec:exp-constructive}
Three claims motivate NCTS: a factorised SCM-mechanism posterior is the dominant gain under shift (Sec.~\ref{par:bench-1a}); re-selecting $A$ per realised inner context dominates committing $(M, A)$ jointly (Sec.~\ref{par:bench-1b}; at $L{=}3$ by ${+}12.8$ per episode, App.~\ref{app:depth}); and each level can be certified in isolation, which drives the handover of Section~\ref{sec:per-level} (Sec.~\ref{par:bench-1c}).

\subsubsection{Factorisation gains \emph{significantly} under exogenous shift.}\label{par:bench-1a}
\arm{Flat-TS} and \arm{Factored-TS} differ only in the reward
model; trained on the \bench{$\Gamma$-Stress} source, frozen and replanned
zero-shot onto shifted targets, the factorised bandit attains
$\mathbf{113.6{\pm}4.9}$ versus $56.1{\pm}10.1$ under Shift-$T$ and
$\mathbf{194.7{\pm}4.9}$ versus $146.1{\pm}6.6$ under
Shift-$\mathcal{N}_Q$ (both $p\!<\!10^{-3}$; mild Shift-$U$:
${+}1.2$, App.~\ref{app:exp-attribution}). The joint regression
collapses once the shift exits the source support while the
ICM-invariant mechanisms hold, the empirical complement to the
necessity result of \citet{richens2024robust}
(Fig.~\ref{fig:exp:constructive}a in App.~\ref{app:exp-contrastive}).

\subsubsection{Recursive nesting \emph{significantly} dominates joint commit.}\label{par:bench-1b}
\arm{Joint-Commit} reuses the SCM posterior but commits $A$ with $M$
at meta-start (conditioning on $\widehat C\!=\!\mathbb{E}[C\!\mid\!T,
M]$, reused for all $I{=}20$ inner steps); \arm{NCTS}
re-selects $A$ at every realised $C$. Causal factorisation
contributes ${+}3.3$ reward and nesting a further ${+}3.9$
(Table~\ref{tab:effect-attribution}, Fig.~\ref{fig:exp:constructive}b);
joint commitment forfeits $37.3$ against \arm{Factored-TS}, since
$A$ can no longer adapt to the realised $C_t$.
The gate earns no reward itself; it decides \emph{when} to deploy. On \bench{Depth-Bin} at $L{=}3$ (App.~\ref{app:depth-sweep}) it fires on $10/10$ seeds for a good candidate at ${-}12.71$ per episode against \emph{always} (the pre-crossing delay) and on $0/10$ for the negated policy, saving ${+}23.56$ (Table~\ref{tab:effect-attribution}): asymmetric protection, not reward.

\begin{table}[t]
\centering\small
\caption{Effect attribution: per-episode reward, mean $\pm$ sd, $10$ seeds; Welch $p$, Cohen's $d$ (pooled sd). Top: decomposition on \bench{$\Gamma$-Stress}. Bottom: gate ablation on \bench{Depth-Bin} ($L{=}3$), AEGIS spliced at the crossing, $\Delta$ paired. More rows: App.~\ref{app:exp-attribution}.}
\label{tab:effect-attribution}
\setlength{\tabcolsep}{3pt}\begin{tabular}{@{}llccrrl@{}}
\toprule
\textbf{Effect} & \textbf{Contrast} & \textbf{A} & \textbf{B} & $\Delta$ & $d$ & $p$\\
\midrule
Factorisation & \arm{Factored-TS} vs \arm{Flat-TS} & $114.7{\pm}0.3$ & $111.5{\pm}0.5$ & ${+}3.3$ & $8.21$ & $5.1{\times}10^{-11}$\\
Nesting (on top) & \arm{NCTS} vs \arm{Factored-TS} & $118.6{\pm}0.6$ & $114.7{\pm}0.3$ & ${+}3.9$ & $7.81$ & $7.7{\times}10^{-10}$\\
Both (total) & \arm{NCTS} vs \arm{Flat-TS} & $118.6{\pm}0.6$ & $111.5{\pm}0.5$ & ${+}7.1$ & $12.34$ & $1.7{\times}10^{-15}$\\
Nesting vs joint & \arm{NCTS} vs \arm{Joint-Commit} & $118.6{\pm}0.6$ & $77.4{\pm}3.4$ & ${+}41.2$ & $16.95$ & $7.8{\times}10^{-12}$\\
\midrule
Gate, good cand. & AEGIS vs \emph{always} & $10.86{\pm}3.90$ & $23.56{\pm}0.53$ & ${-}12.71{\pm}3.73$ & \multicolumn{2}{l}{fires $10/10$}\\
Gate, bad cand. & AEGIS vs \emph{always} & $0.00{\pm}0.37$ & ${-}23.56{\pm}0.66$ & ${+}23.56{\pm}0.50$ & \multicolumn{2}{l}{fires $0/10$}\\
\bottomrule
\end{tabular}
\end{table}

\subsubsection{The bound contracts, and the certified gain crosses zero.}\label{par:bench-1c}
Two certified quantities must be kept apart: the \emph{absolute} bound $V_\mathrm{lower}$ of Theorem~\ref{thm:pac-bayes-causal} on $V(\pi_Q; \theta^*)$, and the \emph{gain} bound~\eqref{eq:gain-bound} on which the gate acts, which subtracts the incumbent's estimated value; an absolute bound far below zero is compatible with a positive certified gain.

\textbf{Absolute side: contraction.} The per-episode bound contracts $2.51{\times}$, from ${-}50{,}274{\pm}2{,}149$ to ${-}20{,}053{\pm}1{,}503$ (Table~\ref{tab:bench-1c}). We re-apply the theorem to each prefix of an \arm{NCTS} run on \bench{Contract} (App.~\ref{app:exp-contrastive}); clipping rewards to $[-15, 15]$ makes Assumption~\ref{ass:bounded-weights}(i) hold by construction. At this budget the bound is vacuous: it lies far below the floor of $-300$ per episode ($N_L{=}20$), because the regulariser $(\KL_\Phi + \log(1/\delta))/(\lambda K)$ dominates, with $\KL_\Phi \approx 2.3 \times 10^{3}$ nats at $K{=}2000$ for the $D{=}128$ RFF basis.

\textbf{Gain side: crossings.} To leading order the gain bound
crosses zero after a number of episodes linear in
$\KL_\Phi + \log(1/\delta)$ and quadratic in
$B_w (R_\mathrm{hi} - R_\mathrm{lo})/g$, $g$ the candidate's advantage
(Eq.~\eqref{eq:k-cross}). With a uniform incumbent on a $3$-point
leaf grid ($B_w{=}3$, attained exactly) and the posterior
\emph{learning}, the leaf certificate crosses zero on $10/10$ seeds at
$K{=}3{,}000$ to $16{,}000$ ($\delta{=}0.05$, KL $2.55$ to $25.73$
nats), reaching ${+}25.69$ to ${+}42.86$ per episode by $K{=}52{,}000$;
with the posterior frozen it crosses under non-additive noise, on
\bench{Bounded} and at every level for $L{=}3$, $4$, $5$
(App.~\ref{app:exp-crossings}, App.~\ref{app:depth}). We recommend
deploying only once the gain stays positive at every later checkpoint;
this rule still fires on all ten seeds ($K{=}3{,}000$ to $16{,}000$).

\subsection{External-Data Checks: Open Bandit and HeartSteps V1}
\label{sec:exp-external}
\label{sec:exp-heartsteps}

Two public logs test the certificate alongside simulated scenarios (App.~\ref{app:external}). On the \textbf{Open Bandit Dataset} \citep{cr_saito2021open}, the Bernoulli-TS policy certified from the platform's uniform-Random log ($B_w = 8.84$ exact, $\KL = 0$, $\delta = 0.05$) covers its directly measured value in all $81$ cells; the bound contracts from $-2.38$ at $K = 100$ to $-0.21$ at $K = 10{,}000$ (as does a zero-click control), and the reverse direction ($B_w = 277.8$) is about $32\times$ looser at every $K$. With binary clicks the gap to the on-policy value is $52$ to $570$ times that value, so these cells rule out only gross miscalibration. On the micro-randomised \textbf{HeartSteps V1} trial \citep{klasnja2019efficacy} ($B_w = 2.5$ exact), the outer certificate is positive ($V_\mathrm{lower} = {+}5.63$ and ${+}5.39$ per user-day) and the inner one lies ${+}2.43$ and ${+}1.33$ above its trivial floor of $-2.77$. Neither log certifies improvement over its logging policy, and the inner level lies in the endogenous regime Assumption~\ref{ass:iid} excludes, so the coverage claim rests on Open Bandit.

\section{Conclusion}
\label{sec:conclusion}
\label{sec:future}

We introduced the NCCB framework, an $L$-level hierarchical SCM for progressive certified handover across decision timescales, and proposed NCTS as a first concrete algorithm (recursive Thompson sampling, one belief per episode, known graph). Theorem~\ref{thm:pac-bayes-causal} is a causal PAC-Bayes bound with a $\KL$ regulariser that decomposes along the causal mechanisms; it is off-policy valid for linear and nonlinear ANMs and certifies any intervention scope $\mathcal{I} \subseteq [L]$, exposing the safest deployment scope as the certificate-maximising one (Section~\ref{sec:per-level}). Under its conditions, Corollary~\ref{cor:excess-risk} gives excess risk $\widetilde{O}(1/\sqrt{K})$ in episodes.

\label{sec:limitations}%
Three limitations stand out (extended discussion in Appendix~\ref{app:ext-limitations}). \emph{(i)~Within-level exogeneity} (Assumption~\ref{ass:iid}) is exact for batch domains with parallel or exchangeable inner units (wafer lots, field plots, quality samples), and state carried across timescales is covered through $C^{(\ell)}$. Its boundary is treatment titration: each dose step reacts to the previous response, so the within-level context is endogenous and the level-factored weight \eqref{eq:hybrid-weight} no longer identifies the value. \emph{(ii)~Mechanism class}: additive-noise models; the non-additive case is a stress test only. \emph{(iii)~Evaluation scope}: exact validity verification remains synthetic, since it requires ground-truth policy values; a real log allows checking coverage (Section~\ref{sec:exp-external}), not exact validity. NCTS is a baseline; Apps.~\ref{app:ext-limitations} and~\ref{app:further-open} give open directions.

\begin{ack}
We thank the reviewers and the AC for their constructive feedback, which helped to improve this work. This work was funded by the European Union (Grant Agreement no. 101120763 - TANGO). Views and opinions expressed are however those of the author(s) only and do not necessarily reflect those of the European Union or the European Health and Digital Executive Agency (HaDEA). Neither the European Union nor the granting authority can be held responsible for them. This work was also supported by the Engineering and Physical Sciences Research Council [EP/S023917/1].
We thank our supervisors and colleagues Kristian Kersting, Thomas Bohné, Moritz Willig, Raban Emunds, Yi-Shan Wu, and Yevgeny Seldin for the fruitful and inspiring discussions and feedback.
\end{ack}

\bibliographystyle{abbrvnat}
\bibliography{references}

\newpage
\appendix

\section{Notation}
\label{app:notation}

\begin{table}[h]
\small
\caption{Notation used throughout the paper. Symbols are defined at first use, in the body
or in the appendix cited here; this table is for reference.}
\label{tab:notation}
\begin{tabular}{@{}l p{0.62\linewidth}@{}}
\toprule
\multicolumn{2}{l}{\textsl{Causal model.}} \\
$\mathcal{G}$, $\mathrm{Pa}(X)$ & Causal structure (notational
shorthand for the parent assignment); parent set of mechanism $X$ \\
$\doIntervention(X{=}x)$ & Pearl intervention setting $X$ to $x$ \\
$\theta^* = (\theta_C^*, \theta_Y^*, \theta_R^*)$ & True SCM mechanism parameters \\
$\Phi$, $d_X := \dim(X)$, $D := \sum_{X \in \Phi} d_X$
  & Set of learnable mechanisms (Def.~\ref{def:nccb});
    per-mechanism output dimension; total parameter dimension \\
\midrule
\multicolumn{2}{l}{\textsl{Decision protocol.}} \\
$K$, $L$, $(I_\ell)_{\ell=1}^L$, $N_L := \prod_\ell I_\ell$
  & Episodes; nesting depth ($\ell{=}1$ innermost, $\ell{=}L$ outermost);
    per-level loop counts; total inner samples per episode \\
$J = I_2$, $I = I_1$, $H = JI = N_2$ & $L{=}2$ shorthand:
  meta-steps; inner steps; per-episode horizon \\
$\mathcal{A}$, $\mathcal{M}$, $\mathcal{X}$ & Inner-action,
  meta-action, inner-context spaces \\
$T_{k,j}$, $m_{k,j}$, $C_{k,j,i}$, $a_{k,j,i}$, $R_{k,j,i}$
  & Task context, meta-action, inner context, inner action,
    inner reward at episode $k$, meta-step $j$, inner step $i$ \\
$x_{k,j,i} = (m_{k,j}, C_{k,j,i}, T_{k,j}) \sim P_X$
  & Inner-step context tuple; its marginal distribution \\
$\mathcal{D}_k$, $\mathcal{D}_\mathrm{off}$ & Episode-$k$ online data;
  pre-online offline data \\
$\mathcal{D}_\mathrm{prior}$, $\mathcal{D}_\mathrm{cert}$
  & Partition of $\mathcal{D}_\mathrm{off}$ into prior-construction
    and certification portions (Appendix~\ref{app:offline-effect}, Effect~4) \\
\midrule
\multicolumn{2}{l}{\textsl{Policies, value, off-policy estimators.}} \\
$\mu_k$, $\pi_k$, $\pi_Q$ & Behaviour, target, PAC-Bayes-induced policy \\
$w_t = \pi(a_t \mid x_t) / \mu_k(a_t \mid x_t)$
  & Importance weight (Assumption~\ref{ass:bounded-weights}) \\
$R_{\max}$, $B^{(\ell)}_w$, $B^{\mathcal{I}}$
  & Reward bound; per-level weight bound; loss bound $B^{\mathcal{I}} \leq R_{\max} \prod_{\ell \in \mathcal{I}} B^{(\ell)}_w$ \\
$V(\pi; \theta^*)$ & Policy value under true parameters \\
$\widehat{V}_k^{\mathrm{IPS}}$, $\widehat{V}^{\mathrm{LS}}_{\mathrm{pes}}$
  & IPS / LS estimators (Eq.~\eqref{eq:ls-estimator}) \\
$\mathcal{L}_t(\pi)$, $c_t = B^{\mathcal{I}} - \mathcal{L}_t$
  & Per-step loss; per-step cost \\
\midrule
\multicolumn{2}{l}{\textsl{PAC-Bayes machinery.}} \\
$\KL(P \| Q)$ & KL divergence \\
$p_0$, $Q$, $p_k$ & Reference prior, data-free with respect to $\mathcal{D}_\mathrm{cert}$ (Bayesian update on
  $\mathcal{D}_\mathrm{prior}$); generic PAC-Bayes posterior;
  episode-$k$ Bayesian posterior \\
$\lambda > 0$ & PAC-Bayes temperature \\
$(\mathcal{F}_t)_{t \geq 0}$, $(\mathcal{H}_k)_{k \geq 0}$
  & Inner-step filtration; episode-level filtration
    ($\mathcal{H}_k = \mathcal{F}_{kJI} \vee \sigma(\theta^*, \mathcal{G},
    \mathcal{D}_\mathrm{prior})$) \\
$F$ & Fisher information (D-optimal design) \\
\bottomrule
\end{tabular}
\end{table}

\section{Instantiation Example: $L = 2$}
\label{app:instantiation-example}

The body's experiments run at $L = 2$ unless stated otherwise, with the shorthand $T = C^{(2)}$, $M = A^{(2)}$, $C = C^{(1)}$, $A = A^{(1)}$, $J = I_2$, $I = I_1$, $N_2 = JI$ summarised in the notation table (Appendix~\ref{app:notation}) and pictured in Figure~\ref{fig:banner-dag}. The example below records the explicit structural equations, the latent-confounder absorption, and the mechanism class; it is not used in any proof and may be skipped on a first reading.

\begin{example}[Linear / nonlinear ANM at $L = 2$]
\label{ex:l2-instantiation}
The mechanism-generated $C$, $Y$, $R$ obey
\begin{align}
  C_{k,j,i} &= f_C(m_{k,j}, T_{k,j};\,\theta_C^*) + \varepsilon_{C,k,j,i}',
     \label{eq:scm-c}\\
  Y_{k,j,i} &= f_Y(a_{k,j,i}, C_{k,j,i};\,\theta_Y^*) + \varepsilon_{Y,k,j,i}',
     \label{eq:scm-y}\\
  R_{k,j,i} &= f_R(Y_{k,j,i}, C_{k,j,i}, m_{k,j}, T_{k,j}, a_{k,j,i};\,\theta_R^*)
                   + \varepsilon_{R,k,j,i}',
     \label{eq:scm-r}
\end{align}
where $\varepsilon_{X,k,j,i}' = \delta_{UX}\,U_{k,j,i} + \varepsilon_{X,k,j,i}$ realises Assumption~\ref{ass:markov} via additive loadings $\delta_{UX}$ on a Bernoulli inner-level confounder $U_{k,j,i} \sim \mathrm{Ber}(\lambda_U)$ shared across $\{C, Y, R\}$. The mechanisms fall within the additive noise model (ANM) family \citep{peters2017elements}: \emph{Case~A1} is linear ANM with Gaussian noise, admitting NIG-conjugate posteriors; \emph{Case~A2} is nonlinear ANM with additive noise, treated by the online Gibbs posterior of \citet{haddouche2022online}. Identifiability of $V(\pi; \theta^*)$ from logged data uses the conditional expectations $\mathbb{E}[C \mid M, T]$, $\mathbb{E}[Y \mid A, C]$, $\mathbb{E}[R \mid Y, C, M, T, A]$, all available under Assumption~\ref{ass:backdoor}.
\end{example}

\begin{example}[NCTS recursion at $L = 2$]
\label{ex:l2-recursion}
We unroll the recursion of Eqs.~\eqref{eq:pi-ell} and~\eqref{eq:R-recursion} at $L = 2$ for the algorithm and experiments. At $L = 2$, the augmented contexts collapse to $\tilde c^{(2)} = T$ and $\tilde c^{(1)} = (T, M, C) =: x$.

\textbf{Innermost-level closed form (independent of $L$).} At $\ell = 1$, Eq.~\eqref{eq:R-recursion} reduces to $\hat R^{(1)}_{\mathcal{I}}(a, \tilde c^{(1)}; \theta) = \mathbb{E}[R \mid A^{(1)} = a, \tilde c^{(1)}; \theta]$, and the level-$1$ argmax of Eq.~\eqref{eq:pi-ell} is the best-action map
\begin{equation}
  \pi_k(\tilde c^{(1)}) = \arg\max_{a \in \mathcal{A}^{(1)}}\,
    \mathbb{E}\bigl[R \mid A^{(1)} = a,\, \tilde c^{(1)};\, \theta^{(k)}\bigr].
  \label{eq:ts-policy}
\end{equation}
This applies at every $L$: the innermost level is always a contextual bandit whose context is the augmented context $\tilde c^{(1)}$. The expectation unfolds via SCM forward-propagation through $f_Y$ then $f_R$:
\begin{equation}
  \mathbb{E}\!\left[R \mid \doIntervention(A^{(1)}{=}a),\, \tilde c^{(1)};\, \theta^{(k)}\right]
  = \mathbb{E}_{\varepsilon_Y', \varepsilon_R'}\!\left[
    f_R\!\bigl(
      \underbrace{f_Y(a, C^{(1)};\, \theta_Y^{(k)}) + \varepsilon_Y'}_{=:\,\hat{Y}},\,
      \tilde c^{(1)},\, a;\, \theta_R^{(k)}
    \bigr)
    + \varepsilon_R'
  \right].
  \label{eq:scm-forward}
\end{equation}
At $L = 2$, $\tilde c^{(1)} = (T, M, C)$ and the parents of $R$ are $(Y, C, M, T, A)$ as in Eq.~\eqref{eq:scm-r}. For linear ANMs with Gaussian noise, \eqref{eq:scm-forward} is closed-form and the $\arg\max$ in \eqref{eq:ts-policy} reduces to evaluating $|\mathcal{A}^{(1)}|$ linear-quadratic expressions; for nonlinear ANMs the expectation is evaluated by Monte Carlo with $n_\mathrm{mc}$ draws. Equation~\eqref{eq:scm-forward} integrates the residuals marginally, which is exact when they are independent of $\tilde c^{(1)}$. If a latent $U^{(1)}$ is shared with $C^{(1)}$ (Remark~\ref{rem:confounding}), the expectation must instead be conditional on $\tilde c^{(1)}$: the policy then also depends on the mechanism of $C^{(1)}$, outside $\Phi_{\mathcal{I}}$ of Definition~\ref{def:intervention-scope}. Theorem~\ref{thm:pac-bayes-causal} and Lemma~\ref{lem:causal-kl} then hold with $\Phi_{\mathcal{I}} \cup \{C^{(1)}\}$ in place of $\Phi_{\mathcal{I}}$, and the KL sum charges that mechanism. The marginal form keeps $\Phi_{\mathcal{I}}$ but need not be optimal; reading $\pi^{\mathcal{I}}_{\theta^*}$ as the best policy in scope (Appendix~\ref{app:step5}) needs the conditional form. Corollary~\ref{cor:excess-risk} is relative to $\pi^{\mathcal{I}}_{\theta^*}$ in either case.

\textbf{Outer-level argmax at $L = 2$.} At $\ell = 2$, the agent observes $T_{k,j}$ ($\tilde c^{(2)} = T_{k,j}$), and Eq.~\eqref{eq:R-recursion} integrates over the freshly drawn $C \sim P_C(\cdot \mid T_{k,j}, m, \theta_C^{(k)})$. With both levels in $\mathcal{I}$, the $\ell{-}1 \in \mathcal{I}$ case of \eqref{eq:R-recursion} applies, giving the meta-action argmax
\begin{equation}
  m_{k,j} = \arg\max_{m \in \mathcal{M}}\;
    \mathbb{E}_{C \sim P_C(\cdot \mid T_{k,j}, m, \theta_C^{(k)})}\!
    \bigl[ \max_a \mathbb{E}[R \mid A^{(1)}{=}a, C, m, T_{k,j}; \theta^{(k)}] \bigr].
  \label{eq:ck-ts}
\end{equation}

\textbf{D-optimal alternative.} A Bayesian experimental-design alternative \citep{chaloner1995bayesian} replaces the Thompson $\arg\max$ at any chosen level $\ell$ by an information-gain $\arg\max$,
\begin{equation*}
  a^{(\ell), \star}_{k} = \arg\max_{a \in \mathcal{A}^{(\ell)}}\,
    \mathbb{E}_{p_{k-1}}\!\left[\log\det F\bigl(\theta;\, \doIntervention(A^{(\ell)}{=}a),\, \tilde c^{(\ell)}\bigr)\right],
\end{equation*}
where $F$ is the posterior Fisher information. This is tractable in closed form under the factorised posterior (Section~\ref{sec:between}) for linear Gaussian ANMs at any $L$. Specialised to the meta level at $L = 2$ ($\ell = 2$, $a = m$, $\tilde c^{(\ell)} = T_{k,j}$), this gives the $L = 2$ D-optimal meta-rule. Proving that the D-optimal variant accelerates posterior contraction beyond the Thompson sampling baseline of \eqref{eq:ck-ts} is an open problem.
\end{example}

\section{Full Proofs of Theorem~\ref{thm:pac-bayes-causal} and Corollary~\ref{cor:excess-risk}}
\label{app:proof-pacbayes}

We prove Theorem~\ref{thm:pac-bayes-causal} in four steps and Corollary~\ref{cor:excess-risk} in Appendix~\ref{app:step5}. Throughout this appendix, we work in the \emph{off-policy startup phase} (episodes $1,\ldots,K$ under a fixed legacy policy $\mu = (\mu_M, \mu_A)$). In this phase the meta-action is supplied by the legacy $\mu_M$ rather than chosen by the agent, and together with the i.i.d.\ meta-contexts $T_{k,j} \sim P_T$ (Assumption~\ref{ass:iid}) this makes the per-episode IPS estimates target a common value $V(\pi;\,\theta^*)$.

\textbf{Filtration conventions.}
Throughout the main text, $(\mathcal{F}_t)_{t \geq 0}$ denotes the inner-step-level filtration accumulating all observations up to the $t$-th inner step (flattening $t = (k-1)JI + (j-1)I + i$). In this appendix we work at the \emph{episode level}: $(\mathcal{H}_k)_{k \geq 0}$ (Eq.~\eqref{eq:ep-filtration}) is the $\sigma$-algebra generated by $\theta^*$, $\mathcal{G}$, $\mathcal{D}_\mathrm{prior}$, and all episode data $\mathcal{D}_1, \ldots, \mathcal{D}_k$. The two are related by $\mathcal{H}_k = \mathcal{F}_{kJI} \vee \sigma(\theta^*, \mathcal{G}, \mathcal{D}_\mathrm{prior})$. Under Assumption~\ref{ass:bounded-weights}(iii), $\mathcal{L}^{\mathcal{I}}_t(\pi) \in [0, B^{\mathcal{I}}]$ and $\widehat{V}_k^{\mathrm{IPS}}(\pi) \in [0, B^{\mathcal{I}}]$.

\subsection{Step 1: Episode-Level Filtration and Unbiasedness}
\label{app:step1}

Define the between-episode filtration:
\begin{equation}
  \mathcal{H}_0 := \sigma(\theta^*,\, \mathcal{G},\, \mathcal{D}_\mathrm{prior}),
  \quad
  \mathcal{H}_k := \sigma\bigl(\mathcal{H}_0,\,
    \mathcal{D}_1,\ldots,\mathcal{D}_k\bigr),
  \quad k \geq 1,
  \label{eq:ep-filtration}
\end{equation}
where $\mathcal{D}_\mathrm{prior} \subseteq \mathcal{D}_\mathrm{off}$ is the portion of offline data used to construct $p_0$ (Eq.~\eqref{eq:p0-construction}; $\mathcal{D}_\mathrm{prior} = \emptyset$ recovers the data-free hyper-prior special case discussed in Appendix~\ref{app:offline-effect}, Effect~4); and the LS-summed episodes $\mathcal{D}_1, \ldots, \mathcal{D}_K$ comprise both the certification portion $\mathcal{D}_\mathrm{cert} = \mathcal{D}_\mathrm{off} \setminus \mathcal{D}_\mathrm{prior}$ (re-indexed as offline records) and any subsequent online episodes. By construction, the data-free prior $p_0$ is $\mathcal{H}_0$-measurable (since it is a deterministic function of $\mathcal{D}_\mathrm{prior}$ via Eq.~\eqref{eq:p0-construction}); in the off-policy phase the $\mathcal{D}_k$ ($k \geq 1$) are independent of one another given $\mathcal{H}_0$ (Assumptions~\ref{ass:markov} and~\ref{ass:iid}).

\begin{lemma}[Per-sample and per-episode unbiasedness]
\label{lem:episode-unbiased}
Fix any intervention scope $\mathcal{I} \subseteq \{1, \ldots, L\}$. Under Assumptions~\ref{ass:markov}--\ref{ass:backdoor}, for every $\theta$, every episode $k$, and every within-episode multi-index $\mathbf i \in [N_L]$:
\begin{equation}
  \mathbb{E}\bigl[\mathcal{L}^{\mathcal{I}}_{k, \mathbf i}(\pi^{\mathcal{I}}_\theta) \mid \mathcal{H}_{k-1}\bigr]
  \;=\; \mathbb{E}\bigl[\bar{\mathcal{L}}^{\mathcal{I}}_k(\pi^{\mathcal{I}}_\theta) \mid \mathcal{H}_{k-1}\bigr]
  \;=\; V(\pi^{\mathcal{I}}_\theta;\,\theta^*),
  \label{eq:ep-unbiased}
\end{equation}
where $\bar{\mathcal{L}}^{\mathcal{I}}_k := \frac{1}{N_L}\sum_{\mathbf i \in [N_L]} \mathcal{L}^{\mathcal{I}}_{k, \mathbf i}$ is the episode-mean loss. Moreover, the episode-mean estimates $\{\bar{\mathcal{L}}^{\mathcal{I}}_k(\pi^{\mathcal{I}}_\theta)\}_{k=1}^K$ are conditionally independent given $\mathcal{H}_0$.
\end{lemma}

\begin{remark}[Why the lemma does \emph{not} claim within-episode independence]
\label{rem:no-within-indep}
The samples $\{\mathcal{L}^{\mathcal{I}}_{k, \mathbf i}\}_{\mathbf i \in [N_L]}$ within episode $k$ are \emph{not} mutually conditionally independent given $\mathcal{H}_{k-1}$ alone, because they share within-episode parent-level draws (e.g.\ at $L=2$, all $I_1$ inner-step samples within meta-step $j$ share $(T_{k,j}, m_{k,j})$). Their conditional means coincide regardless: $\mathbb{E}[\mathcal{L}^{\mathcal{I}}_{k, \mathbf i} \mid \mathcal{H}_{k-1}]$ averages over the random within-episode parent draws under $P_X$ and yields the same value $V(\pi^{\mathcal{I}}_\theta;\theta^*)$ for every $\mathbf i$. The episode-level supermartingale of Step~3 only uses this constancy of conditional means, not within-episode independence; cross-episode independence given $\mathcal{H}_0$ gives the constant conditional mean given $\mathcal{H}_{k-1}$, which is all the supermartingale uses.
\end{remark}

\begin{proof}
\textbf{Sequential factorisation of the trajectory.} Within one episode, the agent draws actions level by level from $\ell = L$ down to $\ell = 1$ (Section~\ref{sec:within}); each action $a^{(\ell)}$ is sampled given the augmented context $\tilde c^{(\ell)}$, which records all higher-level decisions and exogenous/SCM-generated contexts at and above level $\ell$. The chain rule along this fixed level-ordering then factorises the joint trajectory under the legacy behaviour $\mu = (\mu^{(\ell)})_{\ell=1}^L$ as
\[
\begin{aligned}
  P_{\mathrm{legacy}}(c^{(1:L)}, a^{(1:L)}, R) ={}& \prod_{\ell=1}^L p\bigl(c^{(\ell)} \mid \tilde{c}^{(\ell+1)}, a^{(\ell+1)};\, \theta^*\bigr)\, \mu^{(\ell)}\bigl(a^{(\ell)} \mid \tilde{c}^{(\ell)}\bigr) \\
  &\cdot p\bigl(R \mid a^{(1:L)}, \text{contexts};\, \theta^*\bigr).
\end{aligned}
\]
The hybrid distribution under $\pi^{\mathcal{I}}_\theta$ replaces the policy factor $\mu^{(\ell)}$ by $\pi^{(\ell)}_\theta$ at every $\ell \in \mathcal{I}$ while leaving every other factor, including the context kernels, unchanged (Assumption~\ref{ass:markov}):
\[
\begin{aligned}
  P_{\mathrm{hybrid}}(c^{(1:L)}, a^{(1:L)}, R) ={}& \prod_{\ell=1}^L p\bigl(c^{(\ell)} \mid \tilde{c}^{(\ell+1)}, a^{(\ell+1)};\, \theta^*\bigr) \prod_{\ell \in \mathcal{I}} \pi^{(\ell)}_\theta \prod_{\ell \notin \mathcal{I}} \mu^{(\ell)} \\
  &\cdot p(R \mid a^{(1:L)}, \text{contexts};\, \theta^*).
\end{aligned}
\]
The Radon--Nikodym derivative is therefore $dP_{\mathrm{hybrid}}/dP_{\mathrm{legacy}} = \prod_{\ell \in \mathcal{I}} \pi^{(\ell)}_\theta / \mu^{(\ell)} = w^{\mathcal{I}}_{k, \mathbf i}$, with the legacy/legacy factors at $\ell \notin \mathcal{I}$ cancelling and the reward factor cancelling because both joints share the same $p(R \mid a^{(1:L)}, \text{contexts};\, \theta^*)$. The level-wise overlap clause of Assumption~\ref{ass:backdoor} on $\mathcal{I}$ ensures these ratios are well defined.

\textbf{Causal identification of the reward expectation.} The reward factor $p(R \mid a^{(1:L)}, \text{contexts};\, \theta^*)$ is the same conditional under both joints, but the conditional expectation $\mathbb{E}[R \mid a^{(1:L)}, \text{contexts};\, \theta^*]$ taken under the legacy joint is the \emph{interventional} reward $\mathbb{E}_\mathrm{hybrid}[R]$ only when each $a^{(\ell)}$ is identifiable as a do-intervention from observed quantities. The level-wise backdoor admissibility (Assumption~\ref{ass:backdoor}) supplies this identification at every $\ell \in \mathcal{I}$. Combining,
\[
  \mathbb{E}\bigl[\mathcal{L}^{\mathcal{I}}_{k, \mathbf i} \mid \mathcal{H}_{k-1}\bigr]
  \;=\; \mathbb{E}\bigl[w^{\mathcal{I}}_{k, \mathbf i} \cdot R_{k, \mathbf i} \mid \mathcal{H}_{k-1}\bigr]
  \;=\; \mathbb{E}_{\mathrm{hybrid}}[R_{k, \mathbf i}]
  \;=\; V(\pi^{\mathcal{I}}_\theta;\, \theta^*),
\]
where the conditional expectation is taken under the marginal of the within-episode SCM-generated parent contexts given $\mathcal{H}_{k-1}$. Linearity of expectation gives the same identity for $\bar{\mathcal{L}}^{\mathcal{I}}_k$, proving \eqref{eq:ep-unbiased}. Cross-episode conditional independence given $\mathcal{H}_0$ holds because each $\mathcal{D}_k$ is generated from the same fixed $(P_X, \mu, \theta^*)$ in the off-policy phase and is independent of $\mathcal{D}_{k'}$ for $k' \neq k$ given $\mathcal{H}_0$ (Assumptions~\ref{ass:markov} and~\ref{ass:iid}). The episode-mean estimates $\bar{\mathcal{L}}^{\mathcal{I}}_k$ are deterministic functions of $\mathcal{D}_k$, hence inherit this independence.
\end{proof}

\subsection{Step 2: Causal KL Decomposition}
\label{app:step2}

\begin{lemma}[Causal KL decomposition]
\label{lem:causal-kl}
Let $Q = \prod_{X \in \Phi} Q_X$ and $p_0 = \prod_{X \in \Phi} p_{0,X}$ be product distributions over $\theta_{\Phi} = \bigtimes_{X \in \Phi} \theta_X$, induced by the per-mechanism factorisation. Then:
\begin{equation}
  \KL(Q \| p_0)
  = \sum_{X \in \Phi} \KL(Q_X \| p_{0,X}).
  \label{eq:causal-kl}
\end{equation}
For any intervention scope $\mathcal{I} \subseteq \{1, \ldots, L\}$, under the convention $Q_X = p_{0,X}$ for $X \notin \Phi_{\mathcal{I}}$, the chain rule collapses to a sum over the post-intervention learnable mechanisms:
\begin{equation}
  \KL(Q \| p_0) = \sum_{X \in \Phi_{\mathcal{I}}} \KL(Q_X \| p_{0,X}).
  \label{eq:causal-kl-restricted}
\end{equation}
\end{lemma}

\begin{proof}
KL divergence is additive over product measures (both sides are $+\infty$ unless $Q_X \ll p_{0,X}$ for every $X$): $\KL(Q\|p_0) = \int \log\prod_X (dQ_X/dp_{0,X}) \prod_\beta dQ_\beta = \sum_X \int \log(dQ_X/dp_{0,X})\, dQ_X = \sum_X \KL(Q_X\|p_{0,X})$. The posterior factorises as $p_k(\theta) = \prod_{X \in \Phi} p_k(\theta_X)$ because (i)~the prior factorises by assumption (Section~\ref{sec:between}), a choice that mechanism autonomy justifies \citep{pearl2000causality}, and (ii)~the likelihood of the mechanism-factorised working model factorises by construction (Section~\ref{sec:between}). For the restricted form \eqref{eq:causal-kl-restricted}, the convention $Q_X = p_{0,X}$ for $X \notin \Phi_{\mathcal{I}}$ gives $\KL(Q_X \| p_{0,X}) = 0$ for those terms, leaving only the sum over $\Phi_{\mathcal{I}}$. The policy $\pi^{\mathcal{I}}_\theta$ depends on $\theta$ only through $\theta_{\Phi_{\mathcal{I}}}$, since Eq.~\eqref{eq:R-recursion} integrates only over contexts below $\ell^\star_{\mathcal{I}}$; a posterior over $\Phi$ is therefore certified through its $\Phi_{\mathcal{I}}$-marginal. The exception is $\ell^\star_{\mathcal{I}} = 1$ with a latent $U^{(1)}$ shared with $C^{(1)}$: the conditional policy then also reads the mechanism of $C^{(1)}$ (Appendix~\ref{app:instantiation-example}).
\end{proof}

\subsection{Step 3: Logarithmic Smoothing Supermartingale}
\label{app:step3}

We adapt the proof of \citet[Thm.~3.2]{haddouche2025logarithmic} from within-deployment rounds to between-episode samples. The key adaptation is that we apply the LS construction to the \emph{episode-mean} loss $\bar{\mathcal{L}}^{\mathcal{I}}_k$ (Lemma~\ref{lem:episode-unbiased}), not the per-sample stream $\mathcal{L}^{\mathcal{I}}_{k, \mathbf i}$; this side-steps the within-episode dependence highlighted in Remark~\ref{rem:no-within-indep}.

\noindent\textbf{Setup.}\; For $\lambda > 0$ and every $\theta$, define the per-episode gap $\Delta\bar{\mathcal{L}}^{\mathcal{I}}_k(\pi_\theta) := B^{\mathcal{I}} - \bar{\mathcal{L}}^{\mathcal{I}}_k(\pi_\theta) \in [0, B^{\mathcal{I}}]$ (since $\bar{\mathcal{L}}^{\mathcal{I}}_k \in [0, B^{\mathcal{I}}]$ by Assumption~\ref{ass:bounded-weights}(iii)). True gap $\overline{\Delta\bar{\mathcal{L}}}(\pi_\theta;\theta^*) := B^{\mathcal{I}} - V(\pi_\theta;\theta^*)$. For each episode $k$, define
\begin{equation}
  Y_k(\theta) := -\log\!\bigl(1 - \lambda\, \overline{\Delta\bar{\mathcal{L}}}(\pi_\theta;\theta^*)\bigr)
    + \log\!\bigl(1 - \lambda\,\Delta\bar{\mathcal{L}}^{\mathcal{I}}_k(\pi_\theta)\bigr),
  \label{eq:Yk}
\end{equation}
and $f_K(\theta) := \sum_{k=1}^K Y_k(\theta)$. Under Assumption~\ref{ass:bounded-weights} and the fixed $\lambda \in (0, 1/B^{\mathcal{I}})$ of Theorem~\ref{thm:pac-bayes-causal}, both logarithms are well-defined.

\begin{lemma}[Episode-level LS supermartingale]
\label{lem:supermartingale}
Define $M_k := \mathbb{E}_{\theta \sim p_0}[\exp(f_k(\theta))]$ for $k \geq 1$ and $M_0 := 1$. Under Assumption~\ref{ass:bounded-weights} and Lemma~\ref{lem:episode-unbiased}, $(M_k)_{k \geq 0}$ is a non-negative supermartingale w.r.t.\ $(\mathcal{H}_k)_{k \geq 0}$.
\end{lemma}

\begin{proof}
We show $\mathbb{E}[M_k \mid \mathcal{H}_{k-1}] \leq M_{k-1}$. Since $p_0$ is $\mathcal{H}_0$-measurable (Eq.~\eqref{eq:ep-filtration}; $p_0$ depends only on $\mathcal{D}_\mathrm{prior}$, and $\mathcal{H}_0 \subseteq \mathcal{H}_{k-1}$), the law $p_0$ is fixed under conditioning on $\mathcal{H}_{k-1}$, so Fubini gives $\mathbb{E}[M_k \mid \mathcal{H}_{k-1}] = \mathbb{E}_{\theta\sim p_0}[e^{f_{k-1}(\theta)} \mathbb{E}[e^{Y_k(\theta)} \mid \mathcal{H}_{k-1}]]$ (noting that $f_{k-1}(\theta)$ is $\mathcal{H}_{k-1}$-measurable in $\omega$ and jointly Borel in $(\theta, \omega)$, the argmax policies of Eq.~\eqref{eq:pi-ell} being jointly measurable in $(\theta, \tilde c^{(\ell)})$ under any fixed measurable tie-breaking rule (we adopt the smallest maximising index on the finite grid $\mathcal{A}^{(\ell)}$), and that the integrand is non-negative and bounded by $(1 - \lambda B^{\mathcal{I}})^{-k}$, so Fubini applies and $M_k$ is integrable). The exponent $Y_k(\theta)$ is a single log-ratio, so
\begin{equation}
  \mathbb{E}\bigl[e^{Y_k(\theta)} \mid \mathcal{H}_{k-1}\bigr]
  \;=\; \frac{\mathbb{E}\!\bigl[1 - \lambda\, \Delta\bar{\mathcal{L}}^{\mathcal{I}}_k(\pi_\theta) \,\big|\, \mathcal{H}_{k-1}\bigr]}{1 - \lambda\, \overline{\Delta\bar{\mathcal{L}}}(\pi_\theta;\theta^*)}.
  \label{eq:ratio-product}
\end{equation}
By Lemma~\ref{lem:episode-unbiased}, $\mathbb{E}[\bar{\mathcal{L}}^{\mathcal{I}}_k(\pi_\theta) \mid \mathcal{H}_{k-1}] = V(\pi_\theta;\theta^*)$, hence $\mathbb{E}[\Delta\bar{\mathcal{L}}^{\mathcal{I}}_k(\pi_\theta) \mid \mathcal{H}_{k-1}] = B^{\mathcal{I}} - V(\pi_\theta;\theta^*) = \overline{\Delta\bar{\mathcal{L}}}(\pi_\theta;\theta^*)$. The numerator equals $1 - \lambda\, \overline{\Delta\bar{\mathcal{L}}}(\pi_\theta;\theta^*)$, the ratio is $1$, and $\mathbb{E}[M_k \mid \mathcal{H}_{k-1}] = M_{k-1}$, which is in particular $\leq M_{k-1}$.\footnote{The conditional expectation is in fact equal to $M_{k-1}$, so $(M_k)$ is a martingale; we state the supermartingale property because it is what Ville's inequality (Step~4) requires. The argument uses only the constancy of the conditional mean $\mathbb{E}[\bar{\mathcal{L}}^{\mathcal{I}}_k \mid \mathcal{H}_{k-1}]$ across $k$, never within-episode independence (Remark~\ref{rem:no-within-indep}). If overlap fails, $R \geq 0$ gives $\mathbb{E}[\bar{\mathcal{L}}^{\mathcal{I}}_k \mid \mathcal{H}_{k-1}] \leq V(\pi_\theta;\theta^*)$, the ratio in \eqref{eq:ratio-product} is at most $1$, and $(M_k)$ stays a supermartingale. This covers the deterministic legacy of Appendix~\ref{app:ra-ncts-design}, where the weights are evaluated only on the support of $\mu$ and the bound is valid but capped.}
\end{proof}

\subsection{Step 4: Ville's Inequality and Change of Measure}
\label{app:step4}

By Ville's inequality, which is Markov's inequality applied to the stopped supermartingale \citep[Secs.~6.4 and 10.9]{williams1991probability}, since $(M_k)$ is a non-negative supermartingale with $M_0 = 1$:
\begin{equation}
  \Pr\!\bigl[\exists\, k \geq 1: M_k \geq 1/\delta\bigr] \leq \delta.
  \label{eq:ville}
\end{equation}
On $\{M_K < 1/\delta\}$, apply Donsker--Varadhan: for any posterior $Q$,
\begin{equation}
  \mathbb{E}_{\theta \sim Q}[f_K(\theta)]
  \leq \KL(Q \| p_0) + \log M_K \leq \KL(Q \| p_0) + \log(1/\delta).
  \label{eq:change-measure}
\end{equation}
Expanding $\mathbb{E}_Q[f_K(\theta)]$ via the definition of $Y_k(\theta)$:
\begin{align}
  \mathbb{E}_Q[f_K(\theta)]
  &= -K\, \mathbb{E}_Q\!\bigl[\log(1 - \lambda\, \overline{\Delta\bar{\mathcal{L}}}(\pi_\theta;\theta^*))\bigr]
    + \sum_{k=1}^K \mathbb{E}_Q\!\bigl[\log(1 - \lambda\, \Delta\bar{\mathcal{L}}^{\mathcal{I}}_k(\pi_\theta))\bigr].
  \label{eq:expand-fK}
\end{align}
Jensen's inequality applied to the convex $u \mapsto -\log(1-\lambda u)$ under the pushforward of $\theta \mapsto B^{\mathcal{I}} - V(\pi_\theta;\theta^*)$, with the mixture convention $V(\pi^{\mathcal{I}}_Q;\theta^*) = \mathbb{E}_Q[V(\pi^{\mathcal{I}}_\theta;\theta^*)]$, gives $-\mathbb{E}_Q[\log(1 - \lambda\, \overline{\Delta\bar{\mathcal{L}}}(\pi_\theta;\theta^*))] \geq -\log(1 - \lambda\, (B^{\mathcal{I}} - V(\pi_Q^{\mathcal{I}};\theta^*)))$. Combining with \eqref{eq:change-measure}, dividing by $\lambda K$, and using $-\log(1-\lambda x)/\lambda \geq x$ for $\lambda x < 1$:
\begin{equation}
  V(\pi^{\mathcal{I}}_Q;\,\theta^*)
  \geq \widehat{V}^{\mathrm{LS}}_{\mathrm{pes}}(\pi^{\mathcal{I}}_Q)
  - \frac{\sum_{X \in \Phi_{\mathcal{I}}}
    \KL(Q_X \| p_{0,X}) + \log(1/\delta)}{\lambda K},
  \label{eq:gen-bound}
\end{equation}
with $\widehat{V}^{\mathrm{LS}}_{\mathrm{pes}}(\pi^{\mathcal{I}}_Q) = B^{\mathcal{I}} + \frac{1}{\lambda K}\sum_{k=1}^K \mathbb{E}_Q[\log(1-\lambda\, \Delta\bar{\mathcal{L}}^{\mathcal{I}}_k(\pi_\theta))]$ as in Eq.~\eqref{eq:ls-estimator}, and we used the causal KL decomposition (Lemma~\ref{lem:causal-kl}) to replace $\KL(Q\|p_0)$ by $\sum_{X \in \Phi_{\mathcal{I}}} \KL(Q_X\|p_{0,X})$. The bound holds simultaneously for all factorised $Q$ and all $K \geq 1$ (time-uniform, inherited from Ville).

\subsection{Excess-Risk Bound: Proof of Corollary~\ref{cor:excess-risk}}
\label{app:step5}

Fix the scope $\mathcal{I}$ and write $\theta = (\theta_X)_{X \in \Phi_{\mathcal{I}}}$. By backward induction on Eq.~\eqref{eq:R-recursion} (in the conditional form of Appendix~\ref{app:instantiation-example} when $U^{(1)}$ is shared with $C^{(1)}$), $\pi^{\mathcal{I}}_{\theta^*}$ maximises $V(\cdot\,;\theta^*)$ among all policies that follow $\mu^{(\ell)}$ at every $\ell \notin \mathcal{I}$; for $\mathcal{I} = \{1, \ldots, L\}$ it is the optimal policy $\pi^* := \arg\max_\pi V(\pi;\theta^*)$ on the true environment. For $K \geq 1$ let $\lambda_K := 1/(2 B^{\mathcal{I}} \sqrt{K})$, with $B^{\mathcal{I}}$ the assumption-level constant of Assumption~\ref{ass:bounded-weights}(iii), so that $\lambda_K$ is data-free and $\lambda_K B^{\mathcal{I}} \leq 1/2$. Let $\hat{Q}_K$ (written $\hat{Q}$ when $K$ is clear) maximise $\widehat{V}^{\mathrm{LS}}_{\mathrm{pes}}(\pi^{\mathcal{I}}_Q) - \sum_{X \in \Phi_{\mathcal{I}}} \KL(Q_X\|p_{0,X})/(\lambda_K K)$ over factorised $Q$, with the estimator evaluated at $\lambda = \lambda_K$ on the first $K$ episodes; the proof uses only that this objective at $\hat{Q}_K$ is at least its value at the reference posterior $Q^*_K$ below.

\begin{corollary}[Excess risk of the certificate-maximising posterior]
\label{cor:excess-risk}
Assume the conditions of Theorem~\ref{thm:pac-bayes-causal}. Let each $p_{0,X}$ be a non-degenerate Gaussian $\mathcal{N}(\theta_{0,X}, \Sigma_{0,X})$ on $\theta_X$ with $\dim\theta_X = d_X (\sum_{\beta \in \mathrm{Pa}(X)} d_\beta + 1)$. Suppose there is $C < \infty$ with $V(\pi^{\mathcal{I}}_{\theta^*};\theta^*) - V(\pi^{\mathcal{I}}_{\theta};\theta^*) \leq C \|\theta - \theta^*\|$ for all $\theta$. Then for every $\delta \in (0,1)$ there is a constant $c$, depending on $\delta$, $B^{\mathcal{I}}$, $C$, $p_0$ and $\theta^*$ but not on $K$, such that with probability at least $1 - \delta$, simultaneously for all $K \geq 1$,
\[
  V(\pi^{\mathcal{I}}_{\theta^*};\theta^*) - V(\pi^{\mathcal{I}}_{\hat{Q}_K};\theta^*)
  \;\leq\; c\, \frac{p_\Phi \log(K+1)}{\sqrt{K}},
  \qquad
  p_\Phi = \sum_{X \in \Phi} d_X \Bigl(\textstyle\sum_{\beta \in \mathrm{Pa}(X)} d_\beta + 1\Bigr).
\]
\end{corollary}

The certificate-maximising posterior approaches the best policy in scope at this rate. For $\mathcal{I} = \{1, \ldots, L\}$ that policy is the optimal policy $\pi^*$ on the \emph{realised} environment; for a smaller scope it keeps the legacy at $\ell \notin \mathcal{I}$. The certificate behind the bound holds uniformly over posteriors, and $p_\Phi$ counts parameters per mechanism, not jointly. The Lipschitz condition holds whenever $|V(\pi;\theta) - V(\pi;\theta^*)| \leq \tfrac{C}{2}\|\theta - \theta^*\|$ uniformly over that policy class, because $\pi^{\mathcal{I}}_\theta$ is optimal under model $\theta$. It requires no unique argmax.

\begin{proof}
Theorem~\ref{thm:pac-bayes-causal} at $\hat{Q}_K$ and the optimality of $\hat{Q}_K$ reduce the excess risk to three terms at a fixed reference posterior $Q^*_K$ near $\theta^*$. The Lipschitz condition, Hoeffding's inequality and the Gaussian KL formula bound these terms, and a union bound over $K$ gives simultaneity.

\emph{(i) Decomposition.} Let $\delta_K := 3\delta/(\pi^2 K^2)$, so that $\sum_{K \geq 1} 2\delta_K = \delta$, and let $L_K := \log(1/\delta_K) = \log(\pi^2/(3\delta)) + 2\log K$. Fix $K$ and write $\lambda = \lambda_K$. Let $Q^*_K := \mathcal{N}(\theta^*, \sigma_K^2 I)$ with $\sigma_K := K^{-1/2}$, a factorised posterior in dimension $p_{\mathcal{I}} := \sum_{X \in \Phi_{\mathcal{I}}} \dim\theta_X \leq p_\Phi$. Apply \eqref{eq:gen-bound} to $\hat{Q}_K$ at temperature $\lambda$ and confidence $\delta_K$, and use the optimality of $\hat{Q}_K$:
\begin{align}
  V(\pi^{\mathcal{I}}_{\theta^*};\theta^*) - V(\pi^{\mathcal{I}}_{\hat{Q}_K};\theta^*)
  &\leq \underbrace{V(\pi^{\mathcal{I}}_{\theta^*};\theta^*) -
    V(\pi^{\mathcal{I}}_{Q^*_K};\theta^*)}_{\text{(a) approximation error}}
  + \underbrace{V(\pi^{\mathcal{I}}_{Q^*_K};\theta^*) -
    \widehat{V}^{\mathrm{LS}}_{\mathrm{pes}}(\pi^{\mathcal{I}}_{Q^*_K})}_{\text{(b) estimation error at } Q^*_K}\notag\\
  &\quad + \underbrace{\frac{\sum_{X \in \Phi_{\mathcal{I}}} \KL(Q^*_{K,X} \| p_{0,X}) +
    L_K}{\lambda K}}_{\text{(c) regulariser}}.
  \label{eq:excess-oracle}
\end{align}

\emph{(ii) Term (a).} The Lipschitz condition and Jensen's inequality give $\text{(a)} \leq C\, \mathbb{E}_{\theta \sim Q^*_K}\|\theta - \theta^*\| \leq C \sigma_K \sqrt{p_{\mathcal{I}}} = C\sqrt{p_{\mathcal{I}}/K}$.

\emph{(iii) Term (b).} Let $g_k := \mathbb{E}_{\theta \sim Q^*_K}[-\log(1 - \lambda\, \Delta\bar{\mathcal{L}}^{\mathcal{I}}_k(\pi_\theta))]/\lambda$, so that $\widehat{V}^{\mathrm{LS}}_{\mathrm{pes}}(\pi^{\mathcal{I}}_{Q^*_K}) = B^{\mathcal{I}} - \frac{1}{K}\sum_{k=1}^K g_k$. Since $\lambda\, \Delta\bar{\mathcal{L}}^{\mathcal{I}}_k \in [0, 1/2]$, the inequality $-\log(1-u) \leq u + u^2$ on $[0, 1/2]$ and Lemma~\ref{lem:episode-unbiased} give $\mathbb{E}[g_k \mid \mathcal{H}_{k-1}] \leq B^{\mathcal{I}} - V(\pi^{\mathcal{I}}_{Q^*_K};\theta^*) + \lambda (B^{\mathcal{I}})^2$. The last term is the pessimism of the LS estimator. Moreover $g_k \in [0, 2 \log 2 \cdot B^{\mathcal{I}}]$, because $u \mapsto -\log(1-u)/u$ is nondecreasing, and the $g_k$ are i.i.d.\ across $k$ given $\mathcal{H}_0$. Hoeffding's inequality gives, with probability at least $1 - \delta_K$, $\text{(b)} \leq \lambda (B^{\mathcal{I}})^2 + B^{\mathcal{I}}\sqrt{L_K/K} = B^{\mathcal{I}}(\tfrac{1}{2} + \sqrt{L_K})/\sqrt{K}$.

\emph{(iv) Term (c).} We have $1/(\lambda K) = 2 B^{\mathcal{I}}/\sqrt{K}$. The Gaussian KL formula with $\sigma_K^2 = 1/K \leq 1$ gives $\sum_{X \in \Phi_{\mathcal{I}}} \KL(Q^*_{K,X}\|p_{0,X}) \leq c_0 + \tfrac{1}{2} p_{\mathcal{I}} \log K$, where $c_0 := \tfrac{1}{2}\sum_{X \in \Phi_{\mathcal{I}}} [\operatorname{tr}\Sigma_{0,X}^{-1} + (\theta^*_X - \theta_{0,X})^{\!\top} \Sigma_{0,X}^{-1} (\theta^*_X - \theta_{0,X}) + \log\det\Sigma_{0,X}]$.

\emph{(v) Summing up.} The three bounds give
\[
  V(\pi^{\mathcal{I}}_{\theta^*};\theta^*) - V(\pi^{\mathcal{I}}_{\hat{Q}_K};\theta^*)
  \;\leq\; \frac{1}{\sqrt{K}}\Bigl(C\sqrt{p_{\mathcal{I}}} + B^{\mathcal{I}}\bigl[\tfrac{1}{2} + \sqrt{L_K} + 2 c_0 + 2 L_K + p_{\mathcal{I}} \log K\bigr]\Bigr),
\]
which is at most $c\, p_\Phi \log(K+1)/\sqrt{K}$, since $p_{\mathcal{I}} \leq p_\Phi$ and $L_K = O(\log(K+1))$. For this $K$ the bound fails with probability at most $2\delta_K$ (the certificate and Hoeffding). A union bound over $K \geq 1$ gives total failure probability at most $\delta$.
\end{proof}

\subsection{Discussion}
\label{app:pacbayes-discussion}

\textbf{Role of the causal factorisation.} The KL regulariser in
\eqref{eq:gen-bound} sums over $|\Phi|$ terms, each scaling with $d_X (\sum_\beta d_\beta + 1)$. Without the causal factorisation, it would scale with the full joint count $\sum_{X \in \mathbf{X}} d_X (\sum_\beta d_\beta + 1)$. For $\mathcal{I} = \{1\}$ ($\Phi_{\{1\}} = \{Y, R\}$), the saving relative to $\Phi = \{C, Y, R\}$ is $d_C(d_M + d_T + 1)$ parameters (coefficients of the parent set $\{M, T\}$ plus an intercept, per output dimension); substantial when $d_C$, $d_M$, or $d_T$ is large.

\begin{remark}[NCCB episodes as nested deployment rounds]
\label{rem:ls-adapt}
Theorem~3.2 of \citet{haddouche2025logarithmic} is stated for per-sample costs $c \in [-1, 0]$ in conditionally independent batches, with the smoothed term $-\frac{p}{\lambda}\log(1 - \lambda c / q)$ and any $\lambda > 0$. We keep its proof template (log-ratio process, Fubini, Ville, change of measure) and change the summand: one term per episode, built from the episode-mean loss $\bar{\mathcal{L}}^{\mathcal{I}}_k$ of Lemma~\ref{lem:episode-unbiased}, with the logarithm applied to the gap $B^{\mathcal{I}} - \bar{\mathcal{L}}^{\mathcal{I}}_k \geq 0$. Eq.~\eqref{eq:ls-estimator} is therefore a variant of their estimator, not an instance of it, and the sign of the gap is why it needs $\lambda < 1/B^{\mathcal{I}}$. This form is a design choice, not a requirement of the proof: applying the logarithm to the loss itself, $\log(1 + \lambda \bar{\mathcal{L}}^{\mathcal{I}}_k)$, gives a valid bound for every $\lambda > 0$ by the same steps, without Assumption~\ref{ass:bounded-weights}(ii), and measures the second moment from zero rather than from $B^{\mathcal{I}}$. We leave this variant to future work. The $N_L$ samples within an episode are \emph{not} i.i.d.\ given $\mathcal{H}_{k-1}$ (they share within-episode parent draws; Remark~\ref{rem:no-within-indep}), and the supermartingale construction does \emph{not} require them to be: it operates on the episode mean, whose conditional expectation is the same scalar $V(\pi^{\mathcal{I}}_\theta;\theta^*)$ for every $k$ (Lemma~\ref{lem:episode-unbiased}). The supermartingale property carries over because (i)~the episode-mean is conditionally unbiased; (ii)~$\mu$ and $P_X$ are fixed across episodes in the off-policy phase, so $\bar{\mathcal{L}}^{\mathcal{I}}_k$ are conditionally i.i.d.\ given $\mathcal{H}_0$; and (iii)~$p_0$ is data-free w.r.t.\ $\mathcal{D}_1, \ldots, \mathcal{D}_K$. We use a single global temperature $\lambda \in (0, 1/B^{\mathcal{I}})$ at the episode level; the accelerated ``Adjusted LS'' variant of their Sec.~4 bounds a difference of risks and is outside our scope. The within-episode samples are still useful: they reduce the variance of the empirical episode mean $\bar{\mathcal{L}}^{\mathcal{I}}_k$ around its conditional expectation, which tightens the empirical LS term in Eq.~\eqref{eq:ls-estimator} (Jensen: aggregating samples before the log is tighter than averaging after). The adaptive online setting with $\mathcal{H}_{k-1}$-measurable $\mu_k$ is treated in Appendix~\ref{app:adaptive-extension} below.
\end{remark}

\subsection{Adaptive online extension and mixed-policy data}
\label{app:adaptive-extension}

Theorem~\ref{thm:pac-bayes-causal} treats $\mu$ as fixed across episodes, so the level-wise factor cancellation $\mu^{(\ell)}/\mu^{(\ell)}=1$ at $\ell \notin \mathcal{I}$ in \eqref{eq:hybrid-weight} is automatic. Two operational regimes break this assumption: \emph{(i)} the online phase, where the agent acts under its own evolving policy from episode to episode; and \emph{(ii)} progressive handover (Section~\ref{sec:per-level}), where successive certification rounds reuse data collected under partially-handover-deployed behaviour. A single adaptive corollary covers both under overlap.

\begin{corollary}[Adaptive online extension]
\label{cor:adaptive}
Let $(\mu_k)_{k\geq 1}$ be a sequence of behaviour policies (each $\mu_k = (\mu_k^{(\ell)})_{\ell=1}^L$ specifying level-wise action choices), with $\mu_k$ measurable w.r.t.\ $\mathcal{H}_{k-1}$ extended by the randomisation drawn before episode $k$ (e.g., NCTS draws $\theta^{(k)} \sim p_{k-1}$ and acts via $\pi^{(\ell)}_{\theta^{(k)}}$ at every $\ell \in \mathcal{I}$, so $\theta^{(k)}$ joins the filtration before episode $k$). Suppose Assumption~\ref{ass:backdoor} holds with $\mu_k^{(\ell)}$ in place of $\mu^{(\ell)}$ for every $k$ and every $\ell \in \mathcal{I}$, and $\sup_k \sup_{\theta,\, a, \tilde c^{(\ell)}} \pi^{(\ell)}_\theta(a \mid \tilde c^{(\ell)}) / \mu_k^{(\ell)}(a \mid \tilde c^{(\ell)}) \leq B^{(\ell)}_w$ for every $\ell \in \mathcal{I}$ ($\mathcal{H}_{k-1}$-uniform per-level overlap). Then Theorem~\ref{thm:pac-bayes-causal} holds with the per-sample weight $w^{\mathcal{I},\,(\mu_k)}_{k,\mathbf{i}}(\theta) = \prod_{\ell \in \mathcal{I}} \pi^{(\ell)}_\theta / \mu_k^{(\ell)}$ in place of $w^{\mathcal{I}}_{k,\mathbf{i}}$, and the corresponding episode-mean $\bar{\mathcal{L}}^{\mathcal{I},\,(\mu_k)}_k := \frac{1}{N_L}\sum_{\mathbf i} w^{\mathcal{I},\,(\mu_k)}_{k,\mathbf{i}}(\theta) R_{k, \mathbf i}$ in place of $\bar{\mathcal{L}}^{\mathcal{I}}_k$, certifying the per-episode-averaged value $\frac{1}{K}\sum_k V(\pi^{\mathcal{I}}_Q; \theta^*, \mu_k^{\mathcal{I}^c})$, where $\mu_k^{\mathcal{I}^c}$ is the behaviour at levels $\ell \notin \mathcal{I}$, and covering simultaneously: \emph{(i)} the offline startup ($\mu_k = \mu$ fixed legacy) and \emph{(ii)} any mixture across episodes of behaviour policies that satisfy the overlap condition. A deterministic online behaviour does not satisfy it for $\pi^{(\ell)}_\theta \neq \mu_k^{(\ell)}$: the ratio at the played action is $\pi^{(\ell)}_\theta(a \mid \tilde c^{(\ell)}) \in \{0, 1\}$, the identity of Lemma~\ref{lem:episode-unbiased} becomes $\leq$ because $R \geq 0$, and the bound stays valid but credits only the episodes in which $\pi_\theta$ agrees with the played action. The bound certifies any $Q$, in particular $Q = p_K$ (final posterior) or $Q = p_k$ for any $k \leq K$ (anytime over $K$).
\end{corollary}

\begin{proof}
Throughout, $\mathcal{H}_{k-1}$ includes the randomisation drawn before episode $k$. The key observation is that $V(\pi^{\mathcal{I}}_\theta; \theta^*, \mu_k^{\mathcal{I}^c})$ is $\mathcal{H}_{k-1}$-measurable. Lemma~\ref{lem:episode-unbiased} extends to adaptive $\mu_k$: the identity $\mathbb{E}[\mathcal{L}^{\mathcal{I},\,(\mu_k)}_{k, \mathbf i}(\pi^{\mathcal{I}}_\theta) \mid \mathcal{H}_{k-1}] = V(\pi^{\mathcal{I}}_\theta; \theta^*, \mu_k^{\mathcal{I}^c})$ holds because (a)~$\mu_k$ is $\mathcal{H}_{k-1}$-measurable, so conditioning on $\mathcal{H}_{k-1}$ fixes $\mu_k$; (b)~the level-wise $\mu_k^{(\ell)}/\mu_k^{(\ell)}$ cancellation in the IPS step still produces $\sum_{a^{(\ell)}} \pi^{(\ell)}_\theta(a^{(\ell)}\mid \tilde c^{(\ell)})\, \mathbb{E}[R\mid a^{(\ell)}, \tilde c^{(\ell)}; \theta^*]$ at every $\ell \in \mathcal{I}$; (c)~level-wise backdoor identification (Assumption~\ref{ass:backdoor} on $\mathcal{I}$) holds independently of $\mu_k$. Linearity gives the same identity for $\bar{\mathcal{L}}^{\mathcal{I},\,(\mu_k)}_k$. The supermartingale construction (Lemma~\ref{lem:supermartingale}) and Ville's inequality go through unchanged with $f_K(\theta) := \sum_k Y_k(\theta)$ constructed from $\Delta\bar{\mathcal{L}}^{\mathcal{I},\,(\mu_k)}_k = B^{\mathcal{I}} - \bar{\mathcal{L}}^{\mathcal{I},\,(\mu_k)}_k$, with true gap $B^{\mathcal{I}} - V(\pi^{\mathcal{I}}_\theta; \theta^*, \mu_k^{\mathcal{I}^c})$ now depending on $k$. Step~4 then applies Jensen's inequality jointly over $k$ and $\theta \sim Q$, which gives the averaged value. The $\mathcal{H}_{k-1}$-uniform per-level overlap together with Assumption~\ref{ass:bounded-weights}(i) ensures $\bar{\mathcal{L}}^{\mathcal{I},\,(\mu_k)}_k \in [0, B^{\mathcal{I}}]$ and $1 - \lambda \Delta\bar{\mathcal{L}}^{\mathcal{I},\,(\mu_k)}_k > 0$ uniformly for $\lambda < 1/B^{\mathcal{I}}$. This is the discretised, mechanism-factorised, scope-aware analogue of the adaptive PAC-Bayes machinery of \citet{haddouche2022online} with the logarithmic smoother of \citet[Thm.~3.2]{haddouche2025logarithmic}.
\end{proof}

\textbf{Mixed-policy data: data-generating scope $\neq$ deployment scope.} Corollary~\ref{cor:adaptive} certifies the time-averaged value $\frac{1}{K}\sum_k V(\pi^{\mathcal{I}}_Q; \theta^*, \mu_k^{\mathcal{I}^c})$, with $\mu_k^{\mathcal{I}^c}$ the per-episode behaviour at uncontrolled levels. In progressive handover this matters: the cancellation $\mu^{(\ell)}/\mu^{(\ell)}=1$ at $\ell \notin \mathcal{I}$ in \eqref{eq:hybrid-weight} requires $\mathcal{D}_\mathrm{cert}$ to have been collected under $\mu_k$ at every uncontrolled level, but as earlier handover candidates are accepted, the $\mu_k^{\mathcal{I}^c}$ evolves and the corresponding factors no longer cancel; the full $L$-product weight is then required. The candidate $\pi^{(\ell)}_\theta$ at $\ell > 1$ depends on the deployment scope $\mathcal{I}$ via the recursive value definition~\eqref{eq:R-recursion}: the level-$\ell$ greedy is computed under the assumption that descendants follow $\pi^{(\ell')}_\theta$ for $\ell' \in \mathcal{I}, \ell' < \ell$, and $\mu^{(\ell')}$ otherwise. Hence the data-generating scope (call it $\mathcal{I}_\mathrm{data}$) and the deployment scope $\mathcal{I}$ are not interchangeable, even when both are subsets of $[L]$. Reusing data across handover rounds therefore requires recomputing both $\pi^{(\ell)}_\theta$ (under the new $\mathcal{I}$) and the full level-wise IS weight~\eqref{eq:hybrid-weight}.

\textbf{Scope gap benchmark.} On \bench{$\Gamma$-SCM} the gap is large. The meta action is a direct structural parent of the inner context: $f_Q = 2\tanh(2M) + N_Q + 0.3U$ at the \bench{$\Gamma$-Stress} preset ($\gamma = 1$, $\eta = 1$, $\sigma_Q = 3$). So $\mathbb{E}[Q \mid do(M)]$ moves from $-1.86$ to $+2.15$ across the seven-point meta grid, a shift of $1.34$ standard deviations of $Q$, while $\mathrm{sd}[Q]$ stays invariant (20{,}000 draws per grid point). A level-1 certificate on full-legacy logs therefore bounds $V(\pi^{\{1\}}_Q; \theta^*, \mu^{(2)})$, the value of handing over the inner level alone, not the value of the composition $(\pi^{(2)}_Q, \pi^{(1)}_Q)$. On the same environment, at $20{,}000$ Monte Carlo episodes per arm, the inner candidate is worth $-2.337$ per episode under the legacy meta action and $+1.145$ once the meta level is handed over as well, a gap of $+3.48$. The meta action alone changes the legacy value by $-0.103$, so the gap is almost pure interaction, which no single-level weight can express. For a grid-oracle inner policy the gap is $+23.2$ per episode ($+96.0$ against $+119.2$). The sign is not guaranteed. Here the composition exceeds the certified object, which costs power; with a different meta candidate the same interaction can make the certified object exceed the deployment, which costs validity. Per-level certificates therefore certify per-level handovers, and each further round is re-certified on data collected under the levels already handed over.

\subsection{Practical computation of safe-deployment selection}
\label{app:proof-hybrid}

The proof of Theorem~\ref{thm:pac-bayes-causal} (Steps~1--4 of Appendix~\ref{app:proof-pacbayes}) handles arbitrary intervention scope $\mathcal{I} \subseteq [L]$ uniformly: Lemma~\ref{lem:episode-unbiased} establishes hybrid-IS unbiasedness via the level-factored Radon--Nikodym derivative under level-wise overlap (Assumption~\ref{ass:backdoor} on $\mathcal{I}$), Lemma~\ref{lem:causal-kl} delivers the restricted KL chain rule via the convention $Q_X = p_{0,X}$ for $X \notin \Phi_{\mathcal{I}}$, and the LS supermartingale construction (Lemma~\ref{lem:supermartingale}) applies verbatim to the episode-mean gap $\Delta\bar{\mathcal{L}}^{\mathcal{I}}_k = B^{\mathcal{I}} - \bar{\mathcal{L}}^{\mathcal{I}}_k$, requiring $\lambda < 1/B^{\mathcal{I}}$.

\textbf{Computation of \eqref{eq:safe-deploy}.} For each candidate $\mathcal{I} \in \mathfrak{I}$, the bound is evaluated on the same $\mathcal{D}_\mathrm{cert}$ as Theorem~\ref{thm:pac-bayes-causal}, with three quantities recomputed: the per-sample weights $\{w^{\mathcal{I}}_{k, \mathbf i}\}$ (one product per record, aggregated to episode-means $\bar{\mathcal{L}}^{\mathcal{I}}_k$), the loss bound $B^{\mathcal{I}}$ of Assumption~\ref{ass:bounded-weights}(iii) (sets the upper limit on $\lambda$), and the relevant mechanism set $\Phi_{\mathcal{I}}$ (sets the KL summation). The optimal temperature in the leading-order expansion of \eqref{eq:pac-bayes-causal} is closed-form,
\begin{equation*}
  \lambda^\star_{\mathcal{I}} \;=\; \sqrt{\frac{2\bigl(\sum_{X \in \Phi_{\mathcal{I}}} \KL(Q_X\|p_{0,X}) + \log(1/\delta)\bigr)}{K \cdot \widehat{S}^2_{\mathcal{I}}}}, \quad \widehat{S}^2_{\mathcal{I}} := \frac{1}{K} \sum_{k=1}^K \mathbb{E}_Q\!\bigl[(B^{\mathcal{I}} - \bar{\mathcal{L}}^{\mathcal{I}}_k)^2\bigr],
\end{equation*}
clipped to $\min(\lambda^\star_{\mathcal{I}}, c/B^{\mathcal{I}})$ for some $c < 1$ to stay strictly inside the domain of \eqref{eq:pac-bayes-causal}. On the prior split the sum runs over its episodes, $K$ is the planned certification budget and the KL a declared budget. Enumerating $\mathfrak{I}$ needs no extra data, since $\mathcal{D}_\mathrm{cert}$ is shared; the cost of evaluating the bound grows linearly in $|\mathfrak{I}|$.

\begin{remark}[Grid-uniform validity of the empirical temperature]
\label{rem:lam-grid}
The supermartingale of Lemma~\ref{lem:supermartingale} needs one fixed
$\lambda$. The temperatures $\lambda^\star_{\mathcal{I}}$ above and $\lambda^\star_k$
of Eq.~\eqref{eq:lam-star} meet this when computed on the prior split, with
a declared KL budget in place of $\KL(p_k\|p_0)$ and the planned certification budget in place of $k$, since they are then $\mathcal{H}_0$-measurable. Computed on the
certification data, they do not: a union bound over a data-free grid
closes this gap at $O(\log\log K)$ cost and keeps the anytime property. Throughout this
remark, $B^{\mathcal{I}}$ is the assumption-level constant of
Assumption~\ref{ass:bounded-weights}(iii) (e.g.\
$R_{\max}\prod_{\ell \in \mathcal{I}} B^{(\ell)}_w$), and
$\eta_{\mathrm{clip}} \in (0,1)$ is the clipping margin.

Fix the data-free countable grid
$\Lambda := \{\lambda_j = (1-\eta_{\mathrm{clip}})\, 2^{-j} / B^{\mathcal{I}} : j \ge 0\}$
and the confidence weights $\delta_j := \delta \cdot 6/(\pi^2 (j+1)^2)$,
so that $\sum_{j \ge 0} \delta_j = \delta$. Every $\lambda_j$ lies strictly
inside $(0, 1/B^{\mathcal{I}})$, so Theorem~\ref{thm:pac-bayes-causal}
applies at each $\lambda_j$ with confidence $\delta_j$. A union bound
gives Eq.~\eqref{eq:pac-bayes-causal} simultaneously for all $j \ge 0$,
all posteriors and all $K \ge 1$.

Deploy the rounded temperature
$\widehat\lambda := \max\{\lambda_j \in \Lambda : \lambda_j \le
\lambda^\star\}$. Geometric spacing gives
$\widehat\lambda > \lambda^\star/2$, so the regulariser term at
$\widehat\lambda$ exceeds its $\lambda^\star$-value by at most a factor
$2$. The pessimism term does not increase, because
$u \mapsto -\log(1-u)/u$ is nondecreasing. The term $\log(1/\delta)$
grows to $\log(1/\delta_j) = \log(1/\delta) + 2\log(j+1) + \log(\pi^2/6)$.

The index $j$ is logarithmically bounded. Since $S_k \le B^{\mathcal{I}}$
and $C_k \ge \log(1/\delta)$, Eq.~\eqref{eq:lam-star} gives
$\lambda^\star_k \ge \sqrt{2\log(1/\delta)} / (B^{\mathcal{I}} \sqrt{k})$,
unless the clip is active, in which case $j = 0$. So
$2^{-j} \ge \lambda^\star / (2 (1-\eta_{\mathrm{clip}}) / B^{\mathcal{I}})$ yields
$j \le \log_2\!\bigl((1-\eta_{\mathrm{clip}})\sqrt{k} / \sqrt{2\log(1/\delta)}\bigr) + 1
= O(\log k)$ for fixed $\delta \in (0,1)$, and hence
$2\log(j+1) = O(\log\log k)$.

The same construction applies to $\lambda^\star_{\mathcal{I}}$ at fixed
$K$ once the clip constant of the preceding paragraph is
$c := 1-\eta_{\mathrm{clip}}$, so that
$\lambda^\star_{\mathcal{I}} \le \lambda_0$. Combined with the enumeration
of $\mathfrak{I}$, the confidence $\delta$ is replaced by
$\delta/|\mathfrak{I}|$.
\end{remark}

\section{How offline data and the legacy policy shape the bound}
\label{app:offline-effect}

The bound interacts with two quantities a practitioner controls before the online phase: the data $\mathcal{D}_\mathrm{off}$ from which the reference prior $p_0$ is built, and the legacy controller $\mu$ under which $\mathcal{D}_\mathrm{off}$ was collected (and which acts during the off-policy startup phase). The reference prior $p_0$ in Theorem~\ref{thm:pac-bayes-causal} must be data-free with respect to the data used in the LS estimator. It need not be uninformed. Concretely, the workflow constructs $p_0$ from a portion $\mathcal{D}_\mathrm{prior} \subseteq \mathcal{D}_\mathrm{off}$ of the offline data via a single Bayesian update,
\begin{equation}
  p_0(\theta) \;\propto\; p_\mathrm{hyper}(\theta) \;\cdot\!\!
  \prod_{X \in \Phi} p\bigl(\mathcal{D}_\mathrm{prior}^{(X)} \mid \theta_X, \mathrm{Pa}(X)\bigr),
  \label{eq:p0-construction}
\end{equation}
where $p_\mathrm{hyper}$ is any data-free factorised hyper-prior (e.g.\ the broad NIG of Appendix~\ref{app:env-params}). The remainder $\mathcal{D}_\mathrm{cert} := \mathcal{D}_\mathrm{off} \setminus \mathcal{D}_\mathrm{prior}$ enters the LS estimator alongside any online episodes $\mathcal{D}_1, \ldots, \mathcal{D}_K$ collected later under behaviour policies $\mu_k$. Since $p_0$ is a deterministic function of $\mathcal{D}_\mathrm{prior}$ alone, it qualifies as the data-free reference prior of Theorem~\ref{thm:pac-bayes-causal}; the two roles (Bayesian prior for the NCTS update loop and PAC-Bayes reference for the certificate) are served by the same object. The data split between $\mathcal{D}_\mathrm{prior}$ and $\mathcal{D}_\mathrm{cert}$ is itself a design choice, taken up explicitly in Effect~4 below. We now examine how the size, alignment, and coverage of $\mathcal{D}_\mathrm{off}$ enter the certificate.

\textbf{Effect 1: Precision via prior concentration.}
The regulariser term in Theorem~\ref{thm:pac-bayes-causal} scales as $\KL(Q^* \| p_0) / (\lambda K)$. For Gaussian per-mechanism marginals $p_{0,X} = \mathcal{N}(\theta_{0,X}, \Sigma_{0,X})$ and concentrated $Q^*_X = \mathcal{N}(\theta^*_X, \sigma^2 I)$ with $\sigma \to 0$, the per-mechanism KL is
\begin{equation}
  \KL(Q^*_X \| p_{0,X}) \;=\;
  \tfrac{1}{2}(\theta^*_X - \theta_{0,X})^{\!\top}\!
    \Sigma_{0,X}^{-1}(\theta^*_X - \theta_{0,X})
  \;+\; O\!\left(\dim\theta_X \log(1/\sigma)\right).
  \label{eq:kl-gaussian}
\end{equation}
The first term is a Mahalanobis distance from the prior centre $\theta_{0,X}$ to the truth $\theta^*_X$, weighted by the inverse prior covariance $\Sigma_{0,X}^{-1}$. Two regimes: \emph{(a)~tight and well-aligned prior}, $\Sigma_{0,X}^{-1}$ large and $\theta_{0,X} \approx \theta^*_X$, gives a small KL and a tight certificate; \emph{(b)~tight but misaligned prior}, $\Sigma_{0,X}^{-1}$ large and $\theta_{0,X}$ far from $\theta^*_X$, inflates the Mahalanobis distance and \emph{loosens} the certificate. The asymmetry of the KL is the protection mechanism: a prior that confidently concentrates mass in the wrong place pays for that confidence in the bound.

\textbf{Effect 2: Accuracy via legacy suboptimality.}
The bound certifies $V(\pi_{\hat Q}; \theta^*)$ in absolute terms, not relative to $V(\mu; \theta^*)$. The certified \emph{gain} over the legacy controller is therefore
\begin{equation}
  V(\pi_{\hat Q}; \theta^*) \;-\; V(\mu; \theta^*)
  \;\geq\;
  \widehat{V}^{\mathrm{LS}}_{\mathrm{pes}}(\pi_{\hat Q})
  \;-\; \widehat{V}_K^{\mathrm{IPS}}(\mu)
  \;-\; \frac{\KL(\hat{Q}\|p_0) + \log(2/\delta)}{\lambda K}
  \;-\; R_{\max}\sqrt{\frac{\log(2/\delta)}{2K_\mu}},
  \label{eq:gain-bound}
\end{equation}
with probability at least $1-\delta$, where $\widehat{V}_K^{\mathrm{IPS}}(\mu)$ is the unbiased on-policy estimate of $V(\mu)$ under $\mu$ itself (computed as the empirical mean of the $K_\mu$ episode-mean rewards logged under $\mu$), and the last term is its Hoeffding deviation; each of the two events is taken at confidence $\delta/2$. Only the first event is time-uniform: the Hoeffding term holds at a fixed $K_\mu$, so a gate that re-checks the bound must spread that confidence over its checks, levels and epochs, as App.~\ref{app:exp-crossings} does. All terms are in reward units shifted by $R_\mathrm{lo}$, that is Assumption~\ref{ass:bounded-weights}(i) before its division by the range: rewards lie in $[0, R_{\max}]$, so $R_{\max}$ is the width of the reward range and bounds every episode-mean reward. The per-episode totals in Section~\ref{sec:experiments} multiply every term on the right-hand side by $N_L$. When the incumbent enters as a measured constant rather than an estimate (the value form used on the external benches), the Hoeffding term and the factor~2 inside the logarithm drop. The right-hand side is large precisely when $\mu$ is suboptimal and $\pi_{\hat Q}$ has been certified high. Theorem~\ref{thm:pac-bayes-causal} thus delivers an \emph{accuracy boost}: the worse the legacy, the larger the certifiable improvement, with no dependence on the gap entering the regulariser. The expected reward $V(\mu; \theta^*)$ does not appear in Theorem~\ref{thm:pac-bayes-causal}, only on the gain side of \eqref{eq:gain-bound}.

\textbf{Effect 3: Coverage, not quality, of the legacy.}
Suboptimality of $\mu$ (Effect~2) is not the same as misbehaviour. The regulariser and the admissible range of $\lambda$ depend on $\mu$ through (i) the overlap clause of Assumption~\ref{ass:backdoor} and (ii) the loss bound $B^{\mathcal{I}}$ of Assumption~\ref{ass:bounded-weights}, never through $V(\mu)$; the estimate $\widehat V^{\mathrm{LS}}_{\mathrm{pes}}$ and its Jensen gap read the logged rewards, so $S$ in Effect~5 does depend on $\mu$. A uniformly random high-coverage legacy yields a tighter certificate than a near-optimal but concentrated one: the latter inflates $B^{\mathcal{I}}$, which shrinks the admissible range $\lambda < 1/B^{\mathcal{I}}$ and widens the bound, and it may violate overlap entirely. Certifying improvement therefore requires $\mu$ to be \emph{informative}, not good; a feature of PAC-Bayesian off-policy guarantees, not a bug.

\textbf{Effect 4: Pre-TS certification via data splitting.}
Theorem~\ref{thm:pac-bayes-causal} requires $K \geq 1$ episode-mean records in the LS sum and a reference prior $p_0$ that is data-free with respect to those records. A practitioner with only offline data $\mathcal{D}_\mathrm{off}$ (no online interaction yet) reaches certification by data-splitting: partition $\mathcal{D}_\mathrm{off}$ (taken to comprise $N$ offline episodes, each producing a within-episode batch of $N_L$ samples) into a prior portion $\mathcal{D}_\mathrm{prior}$ ($\alpha N$ episodes) and a certification portion $\mathcal{D}_\mathrm{cert}$ ($(1-\alpha) N$ episodes, $\alpha \in [0, 1)$), construct $p_0$ from $\mathcal{D}_\mathrm{prior}$ via Eq.~\eqref{eq:p0-construction}, and evaluate $\widehat{V}^{\mathrm{LS}}_{\mathrm{pes}}(\pi_Q)$ on $\mathcal{D}_\mathrm{cert}$ ($K = (1-\alpha) N$ episodes contributing one episode-mean each). Setting $Q = p_0$ collapses the regulariser to $\log(1/\delta)/(\lambda K)$ since $\KL(p_0 \| p_0) = 0$, which gives a pre-TS certificate with no KL cost; a $Q \neq p_0$ certifies more when its LS estimate gains more than its KL costs. The choice of $\alpha$ trades the quality of the prior against the certification sample $K = (1-\alpha) N$; we report no optimal $\alpha$. The boundary case $\alpha = 0$ is data-free by construction: $p_0$ reduces to the broad NIG hyper-prior of Appendix~\ref{app:env-params} and the LS estimator uses all $N$ episodes. This is the special case in which no offline information enters the prior, so the regulariser pays $\KL(Q \| p_\mathrm{hyper}) = O(p_\Phi \log N)$ for any informative $Q$ learned from the data; loose relative to splitting, but applicable when offline records are scarce or the practitioner declines to pre-commit to a single split.

\textbf{Effect 5: When the gate can fire (the certifiability frontier).}
Setting the expected right-hand side of \eqref{eq:gain-bound} to zero gives, to leading order and for a fixed candidate and KL budget, the true gain at which the certificate fires with probability about one half at budget $K$, before any data. Write $x_k := B^{\mathcal{I}} - \bar{\mathcal{L}}^{\mathcal{I}}_k \in [0, B^{\mathcal{I}}]$. Since $\log(1 - \lambda x) = -\lambda x - [-\log(1-\lambda x) - \lambda x]$, the LS estimator \eqref{eq:ls-estimator} is the plug-in IS estimate minus a Jensen gap,
\begin{equation*}
  \widehat V^{\mathrm{LS}}_{\mathrm{pes}} = \widehat V^{\mathrm{IS}} - D_K(\lambda),
  \qquad
  D_K(\lambda) := \frac{1}{\lambda K}\sum_{k=1}^K\bigl[-\log(1-\lambda x_k) - \lambda x_k\bigr] = \frac{\lambda}{2}S^2 + O(\lambda^2),
\end{equation*}
with $S^2 := \frac1K\sum_k x_k^2$ the second moment of the loss measured from the \emph{top} of the declared range (the $\widehat S^2_{\mathcal{I}}$ of Section~\ref{sec:ra-ncts}). The certified gain is therefore the \emph{coverage deficit} $\widehat\Delta := \widehat V^{\mathrm{IS}}(\pi_{\hat Q}) - \widehat V^{\mathrm{IPS}}_K(\mu)$ minus $D_K(\lambda)$, the regulariser $C/(\lambda K)$ with $C := \KL(\hat Q\|p_0) + \log(2/\delta)$, and the Hoeffding term. At the interior minimiser $\lambda^\star = \sqrt{2C/(KS^2)}$ of the leading-order expression the two $\lambda$-dependent terms are equal and sum to $S\sqrt{2C/K}$, so a candidate whose true gain is $G$ has a non-negative expected certified gain at $K$ only if
\begin{equation}
  G \;\geq\; G_{\min}(K) \;=\; S\sqrt{\frac{2C}{K}}\,\bigl(1 + O(K^{-1/2})\bigr) \;+\; R_{\max}\sqrt{\frac{\log(2/\delta)}{2K_\mu}} \;=\; \frac{c}{\sqrt K}
  \label{eq:frontier}
\end{equation}
(for $K_\mu \propto K$; multiply by $N_L$ for per-episode totals). The realised $\widehat\Delta$ fluctuates around $G$ at the same order $1/\sqrt K$, so under a normal approximation $G_{\min}(K)$ is the median firing point, not a hard threshold; selecting $\hat Q$ on the same data can shift it. The constant $c$ is fixed by the reward span and the overlap constant (through $S \le B^{\mathcal{I}}$), by $\delta$, and by a KL budget that is declarable when the posterior is pinned. Three conditions qualify \eqref{eq:frontier}, and each corresponds to a failure we observed. \emph{(C0)} $\lambda^\star$ must be interior, $K \geq 2C(1-\eta_{\mathrm{clip}})^{-2}(B^{\mathcal{I}}/S)^2$; below this the cap of \eqref{eq:lam-star} binds and the requirement is larger. \emph{(C1)} $\lambda$ must fall with $K$ by a licensed route: data-free constants, $\lambda^\star$ fitted on an independent split, or an a-priori geometric grid closed by a union bound (rounding multiplies the two $\lambda$-dependent terms by at most $5/4$ and adds $O(\log\log K)$ to $C$). At a fixed $\lambda_0$ the Jensen gap $D(\lambda_0)$ does not shrink with $K$, so $G_{\min}(K)$ converges to a positive floor, which is the plateau of fixed-$\lambda$ runs. \emph{(C2)} $\mathbb{E}[\widehat\Delta]$ equals the true gain only when no importance weight is truncated, $B_w \geq \sup \pi/\mu$ exactly. Otherwise the certified functional is the truncated one, whose deficit $\Delta_\infty$ is $K$-independent and, for rewards mapped to $R \geq 0$ before weighting, below $G$; if $\Delta_\infty \leq 0$, no budget certifies any gain. A deterministic legacy is the extreme case: $\mu(a^\mu \mid x) = 1$ at every logged action, so $w \le 1$ and the deficit is non-positive at every $K$, which caps the certificate at $\widehat V(\mu)$. Positivity of $\mu$ over the candidate's support is therefore necessary, and \eqref{eq:frontier} gives the budget it then takes. On the Open Bandit log and the two HeartSteps levels of Section~\ref{sec:experiments}, with the loss moments taken from the certification sample (a consistency check, not a prediction), the derived $c$ of the value form agrees with the measured one to within $1\%$ ($3.5\%$ on the HeartSteps outer level) once the positive third-order term of $D_K$ is included.

\subsection{Graph-error taxonomy and corruption study: full statement}
\label{app:graph-error-full}

\paragraph{Four error cases and the logging condition.}
We distinguish four ways the supplied graph, or the log it is applied to, can be wrong. Only the fourth threatens validity.
\emph{(i)~Extra edges} cost compression, not validity. Remark~\ref{rem:known-unknown} requires only an I-map, so every d-separation read off a superset DAG holds in the true distribution, and backdoor admissibility (Assumption~\ref{ass:backdoor}) survives. The price is a larger KL sum.
\emph{(ii)~Working-model misspecification} can widen the certificate without invalidating it, because the LS estimator touches the data only through the importance weights and realised rewards (Remark~\ref{rem:working-model}). Appendix~\ref{app:offline-effect} discusses the observable symptoms.
\emph{(iii)~Missing observation}, a context component the practitioner cannot see, costs performance, not validity. In an NCCB, a variable no policy acted on cannot confound the action: both policies read only the observed history, so the laws of a trajectory under $\pi$ and $\mu$ differ only in the policy factors, also if the variable acts across levels. Lemma~\ref{lem:episode-unbiased} then holds as long as the variable is independent across episodes given $\mathcal{H}_0$. If the legacy did act on it, the certificate stands provided the log records it; only the planner's candidate quality suffers.
\emph{(iv)~Missing logging} is the dangerous case: the log omits a variable the behaviour policy acted on, or replaces it by a post-action measurement. The recorded propensities then misrepresent how the actions were chosen, the identification step of Lemma~\ref{lem:episode-unbiased} fails, and the bias can point upward. Within this taxonomy only case (iv) can produce false certification, and no off-policy method escapes it without additional signals (instruments, proxies).
Cases (i) and (iii) therefore reduce the requirement from ``know the causal graph'' to a checkable \emph{logging condition}: every dataset a certification round reweights must record, at their decision-time values, the variables its behaviour policy acted on. Incompleteness on the practitioner's side is safe; incompleteness on the log's side is not.

\paragraph{Corruption study: design.}
We test the taxonomy on \bench{$\Gamma$-SCM} at $\sigma_Q = 1$ (10 seeds, $K = 8000$, $\delta = 0.05$, NIG posterior, rewards mapped to $[0,1]$ by the declared range, then IS weights truncated at $w_{\max} = 4$). The study compares a reference certificate (correct graph, true propensities) with six corruption modes. Mode 1 adds a spurious edge (case i). Mode 2 deletes the true edge $Q \to Y$, so the supplied graph is not an I-map. Modes 3a and 3b record propensities marginalised over the trigger the legacy acted on, using its design prior (3a) or the empirical action frequencies (3b). Mode 4 evaluates them at a post-action measurement. Mode 5 records propensities at a permuted context column. Modes 3a, 3b, 4 and 5 are all case iv. The metric is the \emph{false-certification rate}: the fraction of seeds whose certified estimand falls below $\widehat V^{\mathrm{LB}}$ at the final $K$. The weight covers only the inner action, so the bound certifies a hybrid policy: at the meta level it plays whatever the logging agent chose in that episode, and at the inner level it plays $\pi_Q$, drawing one $\theta \sim Q$ per episode. Because the meta choices change from episode to episode, the certified value is this hybrid's value averaged over the certification episodes (Corollary~\ref{cor:adaptive}). We estimate it by Monte Carlo ($4000$ episodes per checkpoint, standard error at most $1.3$). The handover gate is the gain certificate~\eqref{eq:gain-bound} against the on-policy legacy mean, with its Hoeffding term and $\delta$ split over the three checkpoints.

\begin{table}[t]\small\centering
\caption{Corruption study ($\sigma_Q = 1$, 10 seeds, $K = 8000$, $\delta = 0.05$), at $\lambda^\star$ rounded to the grid of Remark~\ref{rem:lam-grid}; $\widehat V^{\mathrm{LB}}$ and the IS point estimate $\widehat V^{\mathrm{IS}}$ in per-episode reward units, against a certified value of $32.3$ ($37.9$ and $77.0$ for modes 1 and 2). False cert.: seeds whose certified estimand lies below $\widehat V^{\mathrm{LB}}$ at final $K$, in parentheses the number of vacuous seeds ($\widehat V^{\mathrm{LB}} \le N_L R_\mathrm{lo} = -300$); gate: gain certificate~\eqref{eq:gain-bound} $> 0$ at any checkpoint $k \in \{2000, 4000, 8000\}$, in parentheses firings with negative true gain.}
\label{tab:corruption-study}
\setlength{\tabcolsep}{4pt}\begin{tabular}{@{}lrrrrr@{}}
\toprule
& \multicolumn{4}{c}{$w_{\max} = 4$} & $w_{\max} = 2.5$\\
\cmidrule(lr){2-5}\cmidrule(l){6-6}
mode & $\widehat V^{\mathrm{LB}}$ (mean$\,\pm\,$sd) & $\widehat V^{\mathrm{IS}}$ & false cert. & gate & false cert.\\
\midrule
reference (correct graph, true propensities) & $-308.5 \pm 28.0$ & $-90.1$ & 0/10 (5) & 0/10 & 0/10\\
1: spurious extra edge (case i) & $-278.8 \pm 55.5$ & $-81.8$ & 0/10 (2) & 0/10 & 0/10\\
2: missing true edge $Q \to Y$ (not an I-map) & $-201.6 \pm 23.1$ & $-32.3$ & 0/10 & 0/10 & 0/10\\
3a: confounded, design-marginal $\mu$ (iv) & $+180.0 \pm 176.0$ & $358.9$ & \textbf{10/10} & \textbf{6/10} (4) & \textbf{2/10}\\
3b: confounded, empirical-marginal $\mu$ (iv) & $-107.8 \pm 37.0$ & $100.7$ & 0/10 & 0/10 & 0/10\\
4: descendant leakage (iv) & $-312.2 \pm 27.2$ & $-93.6$ & 0/10 (8) & 0/10 & 0/10\\
5: permuted-context propensities (iv) & $-75.3 \pm 25.0$ & $124.6$ & 0/10 & 0/10 & 0/10\\
\bottomrule
\end{tabular}
\end{table}

Table~\ref{tab:corruption-study} confirms the taxonomy. Every valid arm, including the graph that is not an I-map, stays at 0/10 at both caps, and at $w_{\max} = 2.5$ no valid certificate is vacuous. Only mode 3a fails (2/10 at $w_{\max} = 2.5$, 10/10 at $w_{\max} = 4$), and at $w_{\max} = 4$ it hands over to a harmful policy: four of its six gate firings are on seeds where the true gain is negative. Modes 3b and 5 do not falsely certify, but at $w_{\max} = 4$ their IS point estimates ($100.7$, $124.6$) exceed the certified value ($32.3$); only the pessimism of the bound keeps them below it. The missing edge costs planning, not certification: the deployed policy collapses to a constant action and loses $11.7$ reward units of true value against the reference's good-policy seeds, yet the bound stays valid. Remark~\ref{rem:known-unknown} does not cover it, but Lemma~\ref{lem:episode-unbiased} and Theorem~\ref{thm:pac-bayes-causal} hold for any posterior, so the table is consistent with the theory. Corrupted certificates do not diagnose themselves: mode 2 and modes 3a, 3b and 5 produce certificates \emph{narrower} than the reference, and mode 2 a lower split-KL. A range check would catch only 2 of mode 3a's 10 false bounds at $w_{\max} = 4$, those above $N_L R_\mathrm{hi} = 300$. Every quantity a practitioner watches looks better while the policy degrades. The logging condition is therefore a design-time requirement, not a runtime check.

\paragraph{Corruption study: protocol and per-seed reading.}
The study of Section~\ref{sec:graph-error} uses \bench{$\Gamma$-SCM} at $\sigma_Q = 1$ ($\gamma = 1$, $\eta = 1$, $N_L = 20$, reward clipped to $[-15, 15]$), an $\varepsilon$-mixture threshold legacy ($\varepsilon = 0.2$ on the 7-point inner grid) acting on the observed trigger $Q$, an even prior/certification split, and true values from $2\times 10^4$ Monte Carlo episodes per seed on the uncorrupted environment. Mode 1 trains under a spurious edge $R \leftarrow \{Y, T\}$. Mode 2 trains under a graph with $Q \to Y$ deleted, the effect modifier in $f_Y = A \cdot Q$, so the supplied graph omits a true parent. Modes 3a, 3b, 4 and 5 reuse the reference runs and corrupt only the recorded propensities (the trigger stays in the log for model fitting), so they isolate the logging failure. Table~\ref{tab:corruption-study} reports the certificate at the closed-form $\lambda^\star$ rounded down to the grid of Remark~\ref{rem:lam-grid}, which conforms to the theorem. The exact overlap constant of this legacy is $n_a/\varepsilon = 35$, above both caps. Rewards are mapped to $[0,1]$ by the declared range $[-15, 15]$ before the weight is applied (Assumption~\ref{ass:bounded-weights}(i)), so truncation can only lower each loss and the certificate remains a lower bound on the untruncated value. At the fixed $\lambda = 1/(2B')$ every sound arm stays vacuous (reference at $w_{\max} = 4$: $-847.8$ at $k = 8000$), because the pessimism of the LS estimator at a fixed temperature barely shrinks with $K$; the grid temperature falls with $K$ and lifts the reference to $-308.5$.

Per seed, one of the ten reference seeds converges to a wrong-signed policy (true $V \approx -71$); the other nine are good-policy seeds (true $V \approx 88$, mean $88.7$). The certified object is a different policy: it keeps the logged meta actions and randomises the inner action over $\theta \sim Q$; its true gain over the legacy is positive only on seeds 5 and 6. Under mode 3a the gate fires on seeds 5 and 6 at $k = 2000$ and on seeds 0, 1, 7 and 8 at $k = 8000$, where the true gain is negative. Mode 2 changes what the agent deploys: a constant action worth $76.9 \pm 0.1$ replaces a $Q$-conditional sign policy worth $88.7$ on good seeds. Its bound stays valid on every seed. Its certificate is narrower than the reference's (mean width $\widehat V^{\mathrm{IS}} - \widehat V^{\mathrm{LB}}$ of $169.4$ against $218.4$) and its split-KL lower ($3.13$ against $9.90$ nats, paired $t = -3.59$ over 10 seeds), because a mechanism with fewer parents moves less from its prior. The propensity-corrupted modes are also narrower ($178.9$, $208.5$ and $200.0$ for 3a, 3b and 5).

\textbf{Self-diagnosis without ground truth.}
A practitioner cannot separate Effect~1's two regimes (well-aligned vs misaligned prior) by inspecting~$p_0$ alone, because the truth $\theta^*_X$ in \eqref{eq:kl-gaussian} is unknown. The bound itself provides the diagnostic, but its two candidate symptoms differ in reliability. The KL symptom, a large $\KL(Q_K \| p_0)$ as the online posterior pulls away from $p_0$ towards parameters consistent with the LS estimator, depends on the posterior family. Under the online Gibbs posterior, misspecification inflates the KL: in the wrong-graph contrast at $\sigma_Q = 1$ (Appendix~\ref{app:exp-contrastive}, item iv), the mean mechanism KL rises from $3363.2$ under the correct graph to $19884.6$ under the wrong one. Under the conjugate NIG working model, the KL can instead \emph{shrink}: from $468.3$ to $58.5$ in the same contrast, while the bound still looks valid. A small KL is therefore no evidence of alignment. The reliable symptom is the certificate itself: misalignment that matters shows as a low $\widehat{V}^{\mathrm{LS}}_{\mathrm{pes}}(\pi_{Q_K})$, a certificate that widens, or a certificate that fails to clear the incumbent, whichever backend produced the posterior. Widening is the analytical mechanism of Effect~1; in the NIG contrast the shrunken-KL certificate stayed valid, so no alarm was owed. The bound is therefore self-conservative with respect to prior and working-model misalignment: \emph{it does not falsely tighten in the wrong-direction case, it loosens or refuses to certify}. This self-diagnosis does not extend to corrupted inputs. When the logged propensities are wrong (case iv of Section~\ref{sec:graph-error}), the certificate can come out narrower, not wider, while being invalid (mode 3a in Table~\ref{tab:corruption-study}). An omitted true edge likewise narrows the certificate and lowers its KL while the deployed policy degrades (mode 2, which stayed valid). Auditing the logged propensities is a precondition for trusting the certificate, not something the certificate can check for itself.

\section{AEGIS Design Choices and Algorithm Coverage}
\label{app:ra-ncts-design}

Section~\ref{sec:ra-ncts} presents AEGIS's three deployment-time modifications. This appendix collects: the host instantiations covered by the recipe, the per-switch design justifications, an implementation note tying the switches to a single configurable agent class, and the gate's correctness clauses on legacy-action logging and post-handover data buffers.

\subsection*{Host instantiations}

Theorem~\ref{thm:pac-bayes-causal} applies to any algorithm that produces a Bayesian posterior $p_k$ with a defined $\KL(p_k \| p_0)$ to a data-free prior, together with a per-mechanism reward predictor $\hat r_\theta(a, x)$. A posterior that does not factorise over mechanisms pays the joint KL in place of the sum. We name four host families:
\begin{itemize}[leftmargin=1.4em,itemsep=1pt,topsep=2pt]
\item \emph{AEGIS-NCTS}: AEGIS over Algorithm~\ref{alg:ncts}; the host whose runs Sec.~\ref{sec:experiments} certifies.
\item \emph{AEGIS-UCB}: per-mechanism upper-confidence-bound exploration during data collection, the certified $\pi_{p_k}$ at deployment.
\item \emph{AEGIS-IDS}: information-directed sampling \citep{russo2014ids} as the exploration rule, identical handover gate.
\item \emph{AEGIS-PB}: direct PAC-Bayes minimisation $Q_k = \arg\min_Q\bigl(-\widehat V^{\mathrm{LS}}_{\mathrm{pes}}(\pi_Q) + \KL(Q\|p_0)/(\lambda K)\bigr)$ replacing passive Bayesian updating.
\end{itemize}
The three switches touch only $p_k$, $\hat r_\theta$, and the legacy log $(x_t, a^\mu_t, R_t)$, so they apply verbatim across the four hosts; none of them depend on \emph{how} $p_k$ was produced.

\subsection*{Why the LCB form of \eqref{eq:ra-inner}}

The deviation factor $\kappa_k = \sqrt{2(\KL + \log(1/\delta))/k}$ has the rate of the PAC-Bayesian deviation bounds of \citet{seldin2012pac}, evaluated at episode $k$ on the $k$ accumulated episode means. It is large when the posterior is broad (early episodes) and shrinks at $\widetilde{O}(1/\sqrt{k})$. At $\kappa_k = 0$, \eqref{eq:ra-inner} reduces to greedy posterior-mean acting; at $\kappa_k$ from \eqref{eq:kappa}, it is a lower confidence bound on the action value, in the spirit of the pessimistic learning principles of \citet{aouali2024unified}.

Two alternatives we considered and rejected:
\emph{(a) Pessimistic Thompson sampling} (sample $\theta$ from a tilted lower-quantile of $p_{k-1}$). Equivalent in expectation but reintroduces the per-step sampling variance the LCB removes; loses the deterministic-action advantage at deployment time.
\emph{(b) McAllester-style $\sqrt{\KL/k}$.} The Logarithmic-Smoothing supermartingale of \citet[Thm.~3.2]{haddouche2025logarithmic} that underlies Theorem~\ref{thm:pac-bayes-causal} gives this form after tuning $\lambda$; using a McAllester rate would either require a separate proof or pay a redundant constant factor.

\subsection*{Why the closed-form of \eqref{eq:lam-star}}

The derivation expands $-\log(1-\lambda c) = \lambda c + \lambda^2 c^2/2 + O(\lambda^3 c^3)$ in \eqref{eq:ls-estimator}, then minimises the leading-order $\lambda$-dependent terms in $\widehat V^{\mathrm{IPS}} - \frac{\lambda}{2} S^2 - C/(\lambda k)$. The first-order condition gives $\lambda^\star = \sqrt{2C/(k S^2)}$; the higher-order correction is $O(\lambda^2 B^3)$ and immaterial for $\lambda B \ll 1$. The clip $(1-\eta_{\mathrm{clip}})/B$ enforces the domain constraint with safety margin $\eta_{\mathrm{clip}}$ (subscripted to distinguish it from the SCM nonlinearity axis $\eta$ of Appendix~\ref{app:exp-setup}).

Comparison: The bounds of \citet{haddouche2022online} hold for a fixed $\lambda$, so tuning it requires a grid over $\mathbb{R}_+$. \eqref{eq:lam-star} replaces the grid with a single algebraic step. Validity requires a temperature fixed independently of the data it certifies. So $S^2_k$ and $C_k$ must come from data outside the LS sum (the prior split), or $\lambda$ must come from an a-priori grid closed by a union bound. At $S^2_k = 0$ only the cap in \eqref{eq:lam-star} applies. Theorem~\ref{thm:pac-bayes-causal} does not cover a $\lambda^\star$ read off the certification sample itself (Appendix~\ref{app:offline-effect}, Effect~5, condition C1).

\subsection*{Why sticky per-level $\Sigma_k$, not adaptive}

The sticky semantics ($\ell$ in $\Sigma_k \implies \ell$ in $\Sigma_{k'}$ for all $k' > k$) is a deliberate restriction. An adaptive variant that allows $\ell$ to leave $\Sigma_k$ would make the realised behaviour of episode $k+1$ a function of all of $\mathcal{F}_k$, but Theorem~\ref{thm:pac-bayes-causal} certifies a fixed deployment scope, not a re-optimised one. In a regulated handover setting (the practical use case), reverting from agent control back to legacy control is also operationally disruptive and rarely warranted.

\subsection*{Implementation note}

The three switches above are flags on a single \texttt{ICTSAgent} class in the codebase: \texttt{risk\_adjust\_inner}, \texttt{optimal\_lambda}, \texttt{progressive\_handover}. The YAML preset \texttt{aegis\_cts\_handover\_tsign\_nig.yaml} flips all three on. With all three off, \texttt{ICTSAgent} reduces to vanilla NCTS (Algorithm~\ref{alg:ncts}); the preset flag is data, not type. This is also why the body keeps two algorithm boxes rather than collapsing into one with conditional branches: the proof in Appendix~\ref{app:proof-pacbayes} reasons about Algorithm~\ref{alg:ncts}'s policy directly and the three switches are independent ablation knobs over it. Theorem~\ref{thm:pac-bayes-causal} certifies the Thompson mixture $\pi^{\mathcal{I}}_Q$. The LCB action of \eqref{eq:ra-inner} is a deterministic function of the posterior and lies outside this class, so Algorithm~\ref{alg:ra-ncts} plays $\pi_{p_{k_\ell}}$ after the flip at episode $k_\ell$; playing \eqref{eq:ra-inner} there forgoes the certificate. A flipped level keeps this policy, so the played and the certified policy coincide; the level stops learning, and a later posterior would need its own certificate. In Algorithm~\ref{alg:ra-ncts}, $\pi_{p_{k_\ell}}$ denotes one draw $\theta \sim p_{k_\ell}$ per episode, shared by the steps of that level. The loop over $\ell$ runs in increasing order and the first flip clears the other buffers, so at most one level flips per episode. A buffer holds the logged steps of the episodes in which $\ell$ acted under the legacy; Eq.~\eqref{eq:gain-bound} reads it as one mean per episode, and $K$ and $K_\mu$ count these episodes. The temperature of the first line is computed per level, with $B^{\{\ell\}} = R_{\max} B^{(\ell)}_w$ declared before the data. The joint-scope diagnostic $G^{[L]}_k$ is the right-hand side of \eqref{eq:gain-bound} at scope $[L]$ on $\mathcal{D}^{\mu,[L]}$; it is logged, not gated on.

\subsection*{Edge-case implementation notes}

\textbf{Why a deterministic legacy needs no density.}
A deterministic legacy $\mu^{(\ell)}$ has $\mu^{(\ell)}(a^\mu_t \mid x_t) \equiv 1$ at every observed action by construction, so the weighted loss in \eqref{eq:ls-estimator} collapses to $\pi_\theta^{(\ell)}(a^\mu_t \mid x_t) \cdot R_t$ per draw $\theta \sim p_k$, with the expectation outside the logarithm, which the agent can compute from its own posterior $p_{k}$. For the joint scope $[L]$ the weight becomes the per-draw product $\prod_{\ell \in [L]}\pi_\theta^{(\ell)}(a^\mu_t\mid x_t)$, still without any $\mu$ density. The empirical leg $\widehat V^{\mathrm{IPS}}_K(\mu)$ is simply the average reward observed under $\mu$ during the corresponding legacy-phase steps; data is on-policy under $\mu$, so no IS reweighting is needed. A deterministic legacy therefore requires logging $(x_t, a^\mu_t, R_t)$ tuples only; a randomising legacy also needs its propensities.

\textbf{Single-level scope while another level has flipped.}
Algorithm~\ref{alg:ra-ncts} clears the per-level buffers $\mathcal{D}^{\mu,(\ell)}$ of the other levels when a gate flips, because the data distribution shifts: e.g.\ once $1 \in \Sigma_k$, the meta-level data in $\mathcal{D}^{\mu,(2)}$ is generated under $(\mu^{(2)}_{\mathrm{legacy}}, \pi^{(1)}_{p_{k_1}})$ rather than full-legacy. An uncleared buffer would certify, by Corollary~\ref{cor:adaptive}, the average over episodes of the value of $(\pi^{(2)}_Q, \mu^{(1)}_k)$, which mixes the full-legacy hybrid before the flip with the post-handover hybrid after it, and neither alone. A flipped level keeps its policy, so the behaviour is fixed between flips, and a cleared buffer certifies the gain of $\pi^{(\ell)}_Q$ over $\mu^{(\ell)}_\mathrm{legacy}$ under the current hybrid, with Eq.~\eqref{eq:gain-bound} as stated. Practitioners targeting $(\pi^{(2)}_Q, \mu^{(1)}_{\mathrm{legacy}})$ should instead freeze $\mathcal{D}^{\mu,(2)}$ at the moment $\Sigma_k$ first becomes non-empty.

\section{Environment Parameters}
\label{app:env-params}

As a worked example, we instantiate the NCCB (Definition~\ref{def:nccb}) for the linear-ANM setting (Case~A1). The experiments of Section~\ref{sec:experiments} do not use this instance; they use \bench{$\Gamma$-SCM} (App.~\ref{app:exp-setup}). All variables are scalar ($d_X = 1$).

\textbf{Domains.} Task type $T \in \{1,\ldots,3\}$ (observed, exogenous);
machine setting $M \in \{1,\ldots,5\}$ (between-episode); action $A \in \mathcal{A} = \{1,\ldots,10\}$ (within-episode); context $C \in \mathbb{R}$, yield $Y \in \mathbb{R}$, reward $R \in [0,1]$ (clipped). Discretisation: $n = 10$ values per scalar component.

\textbf{Legacy logging policy $\mu = (\mu_M, \mu_A)$.} In the offline
startup phase, $M$ and $A$ are produced by a legacy logging policy $\mu$. For the simulation, $\mu_M$ draws $M$ per episode as $M = 0.6\,T + 1.0 + \varepsilon_M$, $\varepsilon_M \sim \mathcal{N}(0, 0.5^2)$, and $A_t \sim \mu_A(\cdot \mid M, C_t, T)$ is drawn uniformly over $\mathcal{A}$.

\textbf{True mechanism parameters.}
\begin{align*}
  C &= 0.8\,M + \varepsilon_C', &\varepsilon_C' &\sim \mathcal{N}(0, \sigma_{C'}^2), \\
  Y &= 0.5\,A + 0.4\,C + \varepsilon_Y', &\varepsilon_Y' &\sim \mathcal{N}(0, \sigma_{Y'}^2), \\
  R &= \mathrm{clip}_{[0,1]}\!\bigl(0.3\,Y + 0.08\,C + 0.1\,M + 0.05\,T + 0.02\,A + \varepsilon_R'\bigr), &\varepsilon_R' &\sim \mathcal{N}(0, \sigma_{R'}^2).
\end{align*}

\textbf{Confounder.} $U_t \sim \mathrm{Ber}(0.3)$ i.i.d.\ per round.
Compound noise: $\delta_{UC} = 0.2$, $\delta_{UY} = 0.15$, $\delta_{UR} = 0.1$. Marginalised variances: $\sigma_{C'}^2 = \delta_{UC}^2 \lambda_U(1{-}\lambda_U) + \sigma_C^2 \approx 0.10$, similarly for $\sigma_{Y'}^2, \sigma_{R'}^2$ (with $\sigma_C=0.3,\sigma_Y=0.2,\sigma_R=0.15$).

\textbf{Prior hyperparameters (NIG).} For each mechanism $X \in
\{C,Y,R\}$, $p_0(\theta_X, \sigma_X^2) = \mathrm{NIG}(\mu_0^{(X)}, \Lambda_0^{(X)}, a_0^{(X)}, b_0^{(X)})$ with $\mu_0^{(X)} = \mathbf{0}$, $\Lambda_0^{(X)} = 0.1\,\mathbf{I}$, $a_0^{(X)} = 2$, $b_0^{(X)} = 1$.

\section{Experimental Setup Details}
\label{app:exp-setup}

This appendix expands Sec.~\ref{sec:exp-setup}: full
\bench{$\Gamma$-SCM} structural equations, the \bench{$\Gamma$-Stress} preset,
posterior-backend hyperparameters, and the per-bench deviations
table.

\textbf{Structural equations.} \bench{$\Gamma$-SCM} is a single
parametric SCM with five axes ($\gamma, \eta, \nu, \sigma_Q, q$)
selected per bench via Hydra overrides:
\begin{align*}
T &\sim \mathcal{N}(0, 0.3^2),\quad U \sim \mathrm{Ber}(0.5),\\
M &= 0.8 T + N_M,\quad N_M \sim \mathcal{N}(0, 0.2^2),\\
Q &= \gamma\bigl((1{-}\eta)\,1.5 M + \eta\,2 \tanh(2 M)\bigr) + 0.3 U \\
  &\quad + \bigl((1{-}\nu) + \nu(1 + 0.4\gamma|M|)\bigr)N_Q,
    \;N_Q \sim \mathcal{N}(0, \sigma_Q^2),\\
A &\sim \mathcal{N}(0, 0.1^2),\\
Y &= (1{-}q)\,A Q + q\bigl(\!-(A {-} 0.5 Q)^2 + 4\bigr) + 0.3 U + N_Y,
     \;N_Y \sim \mathcal{N}(0, 0.08^2),\\
R &= Y + N_R,\;N_R \sim \mathcal{N}(0, 0.5^2).
\end{align*}
Here $Q$ is the inner context $C = C^{(1)}$ of Sec.~\ref{sec:algorithm}, not the posterior $Q$ of Theorem~\ref{thm:pac-bayes-causal}. Axis semantics: $\gamma$ (lever strength on $M{\to}Q$);
$\eta$ (mechanism nonlinearity, linear $1.5 M$ vs.\ saturating
$2\tanh(2 M)$); $\nu$ (homoscedastic additive vs.\ heteroscedastic
noise, $\propto |M|$); $\sigma_Q$ (inner-context noise std);
$q$ (bilinear $A Q$ reward vs.\ quadratic $-(A{-}0.5 Q)^2 + 4$ with
interior optimum $A^\star\!=\!0.5 Q$). Hierarchy: meta-action $M$
outer, inner-action $A$ inner, both bounded $[-2, 2]$. The simulator
redraws the latent $U$ at every context reveal and at the outcome, so in the reported runs it acts as noise per mechanism and links no two variables; identifiability holds via the Causal Markov property
over $(T, M, Q, A, Y, R)$.

\textbf{The \bench{$\Gamma$-Stress} preset.} Unless stated otherwise, every benchmark
runs at \bench{$\Gamma$-Stress}: $\eta\!=\!1$ (nonlinear $f_Q$),
$\nu\!=\!0$ (additive noise), $\sigma_Q\!=\!3$ (high inner-noise:
$P(\mathrm{sign}(Q)\!\neq\!\mathrm{sign}(M))\!\approx\!25\%$ at
$|M|\!=\!2$), $q\!=\!0$ (bilinear), $J\!=\!1$ meta-step per episode,
$I\!=\!20$ inner steps. In this regime $M$ is costly to change between episodes and $A$ is
cheap; below this noise floor the main signals (factorisation+commit shape
in Sec.~\ref{par:bench-1a}/Sec.~\ref{par:bench-1b},
the parameter-count contrast in Sec.~\ref{par:bench-3b}) are not
empirically separable.

\textbf{Benchmark registry.} Table~\ref{tab:bench-registry} names every benchmark once; the text refers to them by these names.

\begin{table}[h]
\centering\small
\caption{Benchmark registry. $L$: number of levels. The \bench{$\Gamma$-}benches are presets of \bench{$\Gamma$-SCM}; the methods are defined in Table~\ref{tab:exp-arms}.}
\label{tab:bench-registry}
\resizebox{\linewidth}{!}{%
\begin{tabular}{@{}l >{\raggedright\arraybackslash}p{3.5cm} c >{\raggedright\arraybackslash}p{4.3cm} l@{}}
\toprule
\textbf{Name} & \textbf{Question it answers} & $L$ & \textbf{Key settings} & \textbf{Used in}\\
\midrule
\bench{$\Gamma$-SCM} & parametric SCM family & $2$ & axes $\gamma, \eta, \nu, \sigma_Q, q$ (equations above) & throughout\\
\bench{$\Gamma$-Stress} & default; do factorisation and nesting pay? & $2$ & $\eta{=}1$, $\nu{=}0$, $\sigma_Q{=}3$, $q{=}0$, $J{=}1$, $I{=}20$ & Sec.~\ref{sec:exp-constructive}\\
Shift-$T$, -$\mathcal{N}_Q$, -$U$ & does the factorised posterior transfer? & $2$ & zero-shot \bench{$\Gamma$-Stress} targets with shifted $T$, $N_Q$ or $U$ & Sec.~\ref{par:bench-1a}\\
\bench{Contract} & does the absolute bound contract? & $2$ & \arm{NCTS} on \bench{$\Gamma$-Stress}, reward clipped to $[-15, 15]$, prefixes $k \in \{200, \ldots, 2000\}$, $D{=}128$ RFF & Sec.~\ref{par:bench-1c}\\
\bench{Crossing} & does the certified gain cross zero? & $2$ & uniform incumbent, $3$-point leaf grid, $B_w{=}3$, reward clipped to $[-15, 15]$, $N_L{=}20$ & Fig.~\ref{fig:banner}, App.~\ref{app:exp-crossings}\\
\bench{Bounded} & crossing under an almost-sure reward bound & $2$ & interaction basis $[Q, A, QA, 1]$ & App.~\ref{app:exp-crossings}\\
\bench{$\Gamma$-Add} & does a deterministic legacy let the gate fire? & $2$ & $f_Y{=}2A + 0.5Q$, reward clipped to $[-20, 20]$, $\gamma{=}1$, $J{=}1$, $I{=}20$ & App.~\ref{app:exp-aegis-flip}\\
\bench{$\Gamma$-SCM}, $\sigma_Q{=}1$ & does a wrong graph falsely certify? & $2$ & six corruption modes & Sec.~\ref{sec:graph-error}\\
\bench{Bin-3L} & does the gate fire at depth? & $3$ & binary actions, uniform legacy, $B_w{=}2$ per level & App.~\ref{app:depth-nccb}\\
\bench{Depth-Bin} & how does firing scale with depth? does the gate protect? & $3$--$5$ & binary actions, $b{=}2$, $R \in [-3, 3]$ & App.~\ref{app:depth-sweep}, Table~\ref{tab:effect-attribution}\\
\bench{Depth-Grid} & how does overlap compound with depth? & $3$ & \bench{$\Gamma$-SCM} plus an outer level, $7$-point grids, $B^{(\ell)}_w{=}7$ & App.~\ref{app:depth}\\
\bottomrule
\end{tabular}}
\end{table}

\textbf{Posterior backends.} Two backends, both implementing the
Theorem~\ref{thm:pac-bayes-causal} bound. \texttt{NIG-Gibbs}: the
linear-conjugate Normal-Inverse-Gamma posterior with closed-form
update (used as a misspecification contrast against the high-noise
$f_Y\!=\!A Q$ at $\sigma_Q\!=\!3$). \texttt{RFF-GP-Gibbs}:
$D{=}128$ random Fourier features with RBF kernel $\exp(-\lVert x - x'\rVert^2)$
and Gibbs posterior at $\lambda$ chosen per bench;
RFF-GP is the default backend for every bandit in Sec.~\ref{sec:experiments} except where
NIG is named explicitly. Both use the prior $\mu_0 = \mathbf{0}$, $\Lambda_0 = \mathbf{I}$, $a_0 = 2$, $b_0 = 1$.

\textbf{Agent taxonomy.} Table~\ref{tab:exp-arms} fixes the bandit
vocabulary. All bandits share the same RFF-GP function class unless
noted; bandits differ only in (i) the reward model, (ii) the
meta-to-inner commit shape, and (iii) the meta-action rule, isolating
each NCCB component as a single axis of comparison.

\begin{table}[h]
\centering\small
\caption{Agent bandits used across Sec.~7 and App.~\ref{app:exp-contrastive}.
Same RFF-GP basis ($D{=}128$) unless noted.}
\label{tab:exp-arms}
\resizebox{\linewidth}{!}{%
\begin{tabular}{@{}lllll@{}}
\toprule
\textbf{Bandit} & \textbf{Reward model} & \textbf{Commit shape} & \textbf{Meta-rule} & \textbf{First in}\\
\midrule
\arm{Flat-TS}        & joint $(T,M,C,A){\to}R$    & none (flat)               & TS over $A$, $\mu_M$ exo.       & \S\ref{par:bench-1a}\\
\arm{Factored-TS}       & SCM $f_C, f_Y, f_R$        & none (flat)               & TS over $A$, $\mu_M$ exo.       & \S\ref{par:bench-1a}\\
\arm{Joint-Commit}      & SCM factorised             & $(M,A)$ joint at meta-start & TS at $\widehat C$            & \S\ref{par:bench-1b}\\
\arm{NCTS}     & SCM factorised             & meta${\to}$inner recursive  & Algorithm~\ref{alg:ncts}      & \S\ref{par:bench-1b}\\
\arm{AEGIS-NCTS} & = \arm{NCTS} & + closed-form $\lambda^\star$, optional LCB inner action & + per-level handover gate, legacy log $\mu_\mathrm{leg}$ & App.~\ref{app:exp-aegis-flip}\\
\arm{Tabular-TS}     & tabular $(M,A)$ cells      & joint                       & Thompson per cell (= PSRL)    & App.~\ref{par:bench-3b}\\
\bottomrule
\end{tabular}}
\end{table}

\textbf{Per-bench deviations.}
Table~\ref{tab:exp-deviations} lists the main deviations from
the Sec.~\ref{sec:exp-setup} defaults ($K{=}2000$, $10$ seeds, RFF-GP,
\bench{$\Gamma$-Stress} preset).

\begin{table}[h]
\centering\small
\caption{Per-bench deviations from Sec.~\ref{sec:exp-setup} defaults.}
\label{tab:exp-deviations}
\begin{tabular}{@{}lll@{}}
\toprule
\textbf{Sec.} & \textbf{Deviation} & \textbf{Reason}\\
\midrule
\ref{par:bench-1a} & $K_\mathrm{src}{=}2000, K_\mathrm{tgt}{=}25$, freeze\,+\,replan        & zero-shot transfer protocol\\
\ref{par:bench-1c} & $K{=}2000$, IS-weighted random-split $\alpha{=}0.5$                   & expose bound's $\sqrt{1/K}$ contraction\\
\ref{par:bench-2a} & $K{=}200$, two posterior backends                                       & posterior-class contrast\\
\ref{par:bench-3a} & $K{=}100$, NIG posterior, four meta-action rules                     & meta-rule contrast\\
\ref{par:bench-3b} & $K{=}100$, $q{=}1$ (quadratic $f_Y$), $n_a \in \{21, 41, 81\}$        & action-grid scaling\\
\bottomrule
\end{tabular}
\end{table}

\section{Contrastive Comparisons}
\label{app:exp-contrastive}

Three contrastive comparisons probe the bound's behaviour along
orthogonal axes: the posterior class
(Sec.~\ref{par:bench-2a}), the meta-action rule
(Sec.~\ref{par:bench-3a}), and the action-grid resolution
(Sec.~\ref{par:bench-3b}). They live in the appendix because they
support the bound's predicted invariances rather than the body's
constructive claim; the Sec.~\ref{sec:exp-setup} agent taxonomy and
\bench{$\Gamma$-Stress} preset apply unchanged.

\begin{figure}[h]\centering
  \begin{subfigure}[t]{0.32\linewidth}\centering
    \includegraphics[width=\linewidth]{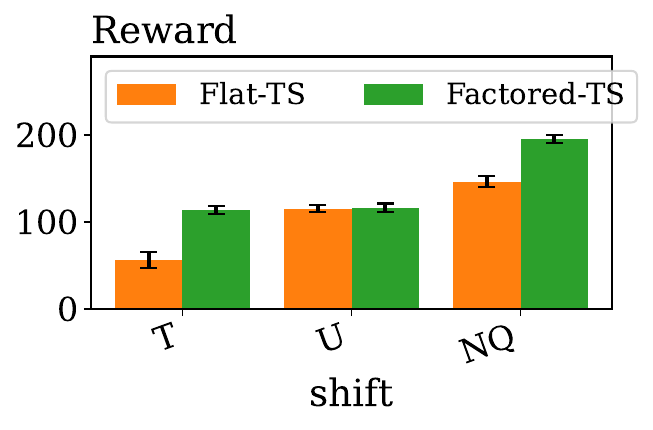}
    \caption{Zero-shot OOD reward.}\end{subfigure}\hfill
  \begin{subfigure}[t]{0.32\linewidth}\centering
    \includegraphics[width=\linewidth]{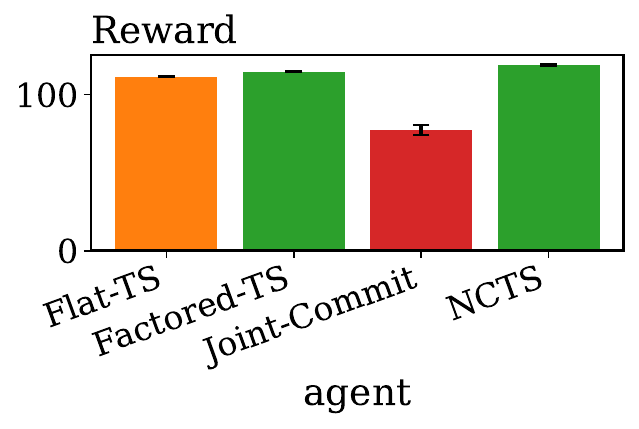}
    \caption{Four-arm decomposition.}\end{subfigure}\hfill
  \begin{subfigure}[t]{0.32\linewidth}\centering
    \includegraphics[width=\linewidth]{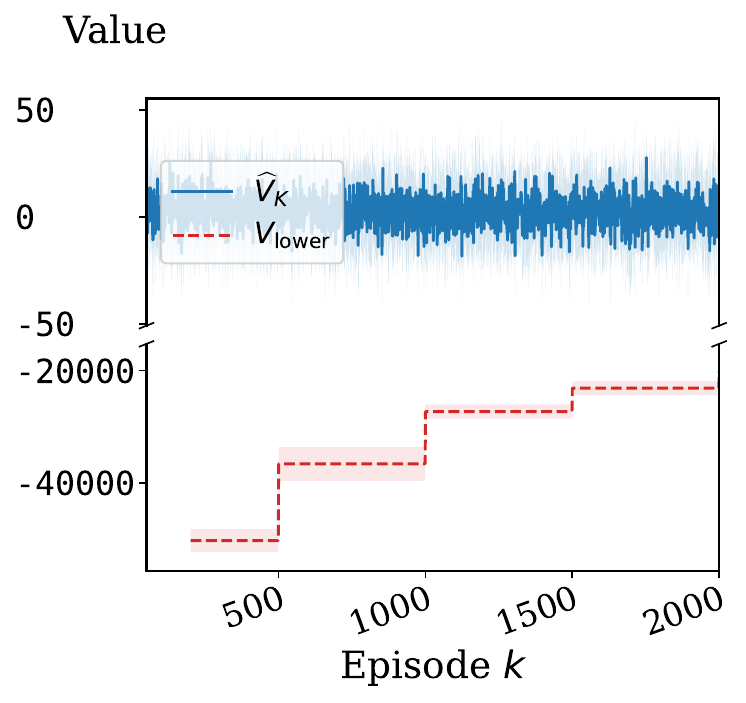}
    \caption{NCTS lower bound contracts.}
  \end{subfigure}
  \caption{Constructive ablations of Section~\ref{sec:exp-constructive}: detailed setup, numerics, and
  interpretation in Sec.~\ref{par:bench-1a}--Sec.~\ref{par:bench-1c}.}
  \label{fig:exp:constructive}
\end{figure}

\paragraph{Sensitivity-analysis map.} The four NCTS-robustness
dimensions (action-grid resolution, posterior-class misspecification,
overlap, graph error) map onto the existing benches as follows.
\emph{(i) Action-grid resolution.} Sec.~\ref{par:bench-3b} sweeps
$n_a \in \{21, 41, 81\}$ and verifies the
invariance in $n_a$ the bound predicts (factorised arms vary by
$\Delta\!<\!1$ reward across the grid; \arm{Tabular-TS}
collapses to negative reward).
\emph{(ii) Posterior-class misspecification.} Sec.~\ref{par:bench-2a}
contrasts well-specified RFF-GP-Gibbs against shape-misspecified
NIG-Gibbs: the wrong class costs about $34$ reward, not validity,
since the bound holds for every posterior.
\emph{(iii) Overlap.} The handover gate's per-level overlap
$B^{(\ell)}_w$ enters the bound through $\lambda<1/B^{\mathcal{I}}$ and is
automatically respected by the closed-form $\lambda^\star$
\eqref{eq:lam-star}. The AEGIS demonstration
(App.~\ref{app:exp-aegis-flip}) stresses overlap through a
\emph{deterministic} inner-level legacy
$a^{(1)}_\mathrm{legacy}\!=\!2\,\mathrm{sign}(T)$ with
$\varepsilon$-greedy positivity at the meta level
($\varepsilon\!=\!0.05$) and truncated IPS (a capped variant of the clipped weights of \citet{bottou2013counterfactual}) at
$w_\mathrm{max}\!=\!1.0$; capping lowers the estimate for non-negative
rewards ($\mathbb{E}[\min(w, w_\mathrm{max})\, R] \leq V$), but the rewards
here are signed, so this bench has no such guarantee.
\emph{(iv) Graph error.} A wrong-graph contrast (parents
permuted under the same SCM) is inconclusive at the \bench{$\Gamma$-Stress} preset
under the RFF-GP backbone (graph-misspecification gap below the noise floor);
the principled prediction (Pearl-truncated factorisation collapses
under wrong parents) holds at the level of the bound's KL
decomposition but is empirically dominated by inner-context noise at
$\sigma_Q\!=\!3$. At $\sigma_Q\!=\!1$ the contrast
resolves: the wrong graph inflates the mean mechanism KL
from $3363.2$ to $19884.6$ under the Gibbs backend but shrinks it from
$468.3$ to $58.5$ under the conjugate NIG backend (with the bound
remaining valid), the backend-dependence taken up in
Appendix~\ref{app:offline-effect}; graph and logging error are probed
against the certificate directly by the corruption study of
Section~\ref{sec:graph-error}.

\paragraph{NCTS bound trajectory under the data-split protocol.}
Per-$k$ values of the random-split lower bound from
Sec.~\ref{par:bench-1c}. These are \emph{absolute} value bounds on
$V(\pi_{Q_k}; \theta^\star)$, not gain bounds. They are stated per
episode under the episode-level aggregation of
Lemma~\ref{lem:episode-unbiased}: one episode-mean record per
certified episode, so the regulariser is divided by the episode count
$k$, not by the sample count $N_L k$. Protocol: \arm{NCTS}
runs for $K{=}2000$ episodes, and each episode is assigned independently
to a prior or a certificate set with probability $\alpha{=}0.5$. For
each $k$ the theorem is applied separately to the first $k$ episodes:
$p_0$ is fitted once on the prior portion (size $\alpha k$), and the
episode-mean LS estimator (Eq.~\eqref{eq:ls-estimator}) and
$\KL(Q\|p_0)$ are evaluated on the certificate portion (size
$(1-\alpha) k$). The loss range of the LS supermartingale is data-free:
$B' = B_w (R_\mathrm{hi} - R_\mathrm{lo})$ for the declared reward range
$[R_\mathrm{lo}, R_\mathrm{hi}] = [-15, 15]$. The legacy logs $A$
uniformly over the $7$ grid points that the certified policy plays, so the
propensity $\mu(a) = 1/7$ is exact and the weight $\pi(a)/\mu(a) = 7\,\pi(a)$
attains its ceiling $B_w = 7$ exactly; hence $B' = 210$ and the temperature
$\lambda = 1/(2B')$ is data-free as well. Assumption~\ref{ass:bounded-weights}(i)
holds by construction, because the environment clips every reward to this
range. The empirical range of the observed losses would make $B'$ depend
on the certification data. The implementation shifts the weighted loss, $wR - B_w R_\mathrm{lo}$, not the reward; since $\mu$ puts mass $1/7$ on every grid point, this has mean $V - B_w R_\mathrm{lo}$ and lies in $[0, B']$, so the theorem applies to it directly.
The paired comparison $\widehat V_K \geq
V_\mathrm{lower}$ has mean difference ${+}20{,}056$ ($t{=}42.2$,
$p\!<\!10^{-11}$). The contraction itself is significant: paired over seeds,
$V_\mathrm{lower}(2000) - V_\mathrm{lower}(200)$ has mean ${+}30{,}221$
($t{=}31.3$, $df{=}9$, $p\!<\!10^{-8}$). The naive unsplit LS estimator, with the same $B'$
and $\lambda$, gives
$V_\mathrm{lower}^\mathrm{naive} = {-}153{,}066{\pm}1{,}941$ against
$V_\mathrm{lower}^\mathrm{split} = {-}20{,}053{\pm}1{,}503$ (paired $t{=}156.6$,
$p\!<\!10^{-16}$); it is not a valid bound on $V(\pi_{Q_K})$, since its
weights $\pi_{\theta_k}(a)/\mu(a)$ use the single Thompson draw
$\theta_k \sim Q_{k-1}$ of episode $k$, not $\theta \sim Q_K$, while
its KL term is $\KL(Q_K \| Q_0)$.

\begin{table}[h]\footnotesize\centering
\caption{Absolute bound $V_\mathrm{lower}(k)$ under nested random
split, per episode, episode-level aggregation (regulariser over $k$
episodes); mean$\,\pm\,1\sigma$ across $10$ seeds, $K{=}2000$,
\arm{NCTS}, no AEGIS switches.}
\label{tab:bench-1c}
\begin{tabular}{@{}lrrrrr@{}}
\toprule
$k$                & $200$              & $500$              & $1000$             & $1500$             & $2000$\\
\midrule
$V_\mathrm{lower}$ & ${-}50{,}274{\pm}2{,}149$ & ${-}36{,}609{\pm}3{,}185$ & ${-}27{,}309{\pm}1{,}339$ & ${-}23{,}137{\pm}1{,}297$ & ${-}20{,}053{\pm}1{,}503$\\
\bottomrule
\end{tabular}
\end{table}

\paragraph{Anytime trajectory under one pre-fixed temperature.}
Table~\ref{tab:bench-1c} and Sec.~\ref{par:bench-1c} report a
sample-size sweep: each plot point re-applies
Theorem~\ref{thm:pac-bayes-causal} to a prefix.
Fig.~\ref{fig:anytime-single-lambda} instead evaluates \emph{two}
Ville supermartingales per seed, one per level of the hierarchy,
along a single continuous run. The prior $p_0$ is fixed once from an
offline prior portion. Rewards are clipped to $[-15, 15]$ by the
environment and mapped affinely to $[0, 1]$ before weighting, so with
$B_w = 2$ the loss lies in $[0, B_w]$, $B' = B_w + 10^{-3} = 2.001$, and
Assumption~\ref{ass:bounded-weights}(i) holds by construction. One
global temperature $\lambda = 1/(2B') = 0.2499$ is fixed from these
declared constants; the per-episode
$\lambda^\star_k$ of Eq.~\eqref{eq:lam-star} is disabled. Both levels
aggregate one term per certified episode
(Lemma~\ref{lem:episode-unbiased}). The two levels share one budget
$\delta = 0.05$ ($\delta/2$ per level). The LS leg is time-uniform at
the fixed $\lambda$, and the Hoeffding leg on $V(\mu)$ is made anytime
by spending $\delta_2\, 6/(\pi^2 j^2)$ at the $j$-th of the $20$ chunk
boundaries evaluated. The posterior keeps learning during the run, so
every evaluation rescores the whole legacy log under the current
posterior and charges that posterior's $\KL$: each point is a
certificate for one $Q$, and all $40$ points of a seed hold
simultaneously with probability at least $1 - \delta$. We read
$V_\mathrm{lower}(k)$ at the pre-fixed checkpoints
$k \in \{200, 500, 1000, 1500, 2000\}$.

No point lies above its seed's Monte Carlo true value (${+}83.83$,
${+}75.84$, ${+}86.21$); the smallest margin over all $120$ evaluated
points is $609.0$ per episode. From $k = 200$ to $k = 2000$ the meta
bound rises from between ${-}910.6$ and ${-}897.5$ to between
${-}805.0$ and ${-}671.7$, and the inner bound from between ${-}625.9$
and ${-}615.9$ to between ${-}545.0$ and ${-}522.8$, without retuning
$\lambda$. The regulariser and the Hoeffding term shrink with $k$
(together $41.2$ to $63.9$ per episode at $k = 2000$). Two terms do
not: the Jensen gap $D_K(\lambda)$ of the LS estimator
(Appendix~\ref{app:offline-effect}, Effect~5), $340$ to $464$ per
episode at this $\lambda$, and the coverage deficit of the
deterministic legacies, between ${-}303.4$ and ${-}144.5$ at
$k = 2000$ and non-positive at every evaluated point (largest value
${-}141.9$). What the figure does \emph{not} demonstrate is a gate with a
$\lambda^\star_k$ fitted on the certification data, which lies outside
the scope of a fixed-temperature supermartingale.

\begin{figure}[h]\centering
\begin{tikzpicture}
\definecolor{metaBlue}{RGB}{0,114,178}
\definecolor{innerVerm}{RGB}{213,94,0}
\begin{axis}[
  width=.85\linewidth, height=6.2cm,
  xlabel={certified episodes $k$},
  ylabel={per-episode value},
  xmin=100, xmax=2100, ymin=-950, ymax=150,
  xtick={200,500,1000,1500,2000},
  ytick={0,-250,-500,-750},
  grid=major, grid style={draw=gray!20},
  tick label style={font=\scriptsize},
  label style={font=\small},
  legend style={font=\scriptsize, at={(0.97,0.72)}, anchor=north east,
    draw=gray!40},
  legend cell align=left,
]
\addplot[color=metaBlue, thick, mark=*, mark size=1.5pt]
  coordinates {(200,-897.51) (500,-846.21) (1000,-826.51)
    (1500,-833.22) (2000,-692.48)};
\addlegendentry{meta $V_\mathrm{lower}(k)$, per seed}
\addplot[color=metaBlue, thick, mark=*, mark size=1.5pt, forget plot]
  coordinates {(200,-910.58) (500,-853.68) (1000,-835.78)
    (1500,-820.50) (2000,-671.67)};
\addplot[color=metaBlue, thick, mark=*, mark size=1.5pt, forget plot]
  coordinates {(200,-899.82) (500,-845.28) (1000,-821.58)
    (1500,-811.44) (2000,-805.01)};
\addplot[color=innerVerm, thick, mark=square*, mark size=1.5pt]
  coordinates {(200,-615.93) (500,-564.61) (1000,-544.33)
    (1500,-546.95) (2000,-544.97)};
\addlegendentry{inner $V_\mathrm{lower}(k)$, per seed}
\addplot[color=innerVerm, thick, mark=square*, mark size=1.5pt,
  forget plot]
  coordinates {(200,-625.87) (500,-579.58) (1000,-572.13)
    (1500,-551.59) (2000,-540.87)};
\addplot[color=innerVerm, thick, mark=square*, mark size=1.5pt,
  forget plot]
  coordinates {(200,-625.45) (500,-564.61) (1000,-550.16)
    (1500,-539.31) (2000,-522.80)};
\addplot[gray, dashed]
  coordinates {(100,83.83) (2100,83.83)};
\addlegendentry{true $V$ (MC), per seed}
\addplot[gray, dashed, forget plot]
  coordinates {(100,75.84) (2100,75.84)};
\addplot[gray, dashed, forget plot]
  coordinates {(100,86.21) (2100,86.21)};
\end{axis}
\end{tikzpicture}
\caption{\textbf{Anytime trajectory under a single pre-fixed
temperature.} Two Ville supermartingales per seed, one per level of
the $2$-level hierarchy ($3$ seeds, $K{=}2000$ certified episodes),
sharing $\delta = 0.05$ ($\delta/2$ per level, Hoeffding leg anytime
over the $20$ evaluations): one a-priori $\lambda = 1/(2B') = 0.2499$
on the $[0, 1]$-mapped rewards, fixed from the declared constants
and never re-optimised, with
$V_\mathrm{lower}(k)$ read at five pre-fixed checkpoints out of
the $20$ evaluated, for the meta level (blue circles) and the inner
level (vermillion squares). Both levels aggregate one term per
certified episode, and each evaluation rescores the legacy log under
the current posterior. Every certificate stays below its seed's
Monte Carlo true value (dashed grey) at all $120$ evaluated points and
rises from $k = 200$ to $k = 2000$ at every seed and level; the
remaining gap is the LS Jensen gap at this $\lambda$ plus the
non-positive coverage deficit of the deterministic legacy policies,
neither of which shrinks with $k$. A gate with a $\lambda^\star_k$
fitted on the certification data is not covered by this
demonstration.}
\label{fig:anytime-single-lambda}
\end{figure}

\subsection{Posterior-class misspecification costs reward, not validity.}\label{par:bench-2a}
On \bench{$\Gamma$-Stress} $f_Q$ ($10$ seeds, $K{=}200$, $\lambda{=}1$),
the matched RFF-GP backend earns $115.8 \pm 0.5$ per episode and the
shape-misspecified NIG-Gibbs backend $81.9 \pm 1.9$ (mean $\pm$ SE): the
conjugate form cannot calibrate $\sigma^2$ at $\sigma_Q\!=\!3$, and the
wrong class costs about $34$ reward. It cannot cost validity:
Theorem~\ref{thm:pac-bayes-causal} holds for every posterior, so the
certificate stays valid for the misspecified class.

\subsection{Under the NIG posterior, D-optimal collapses under stress; Thompson and greedy dominate.}\label{par:bench-3a}
\begin{wraptable}{r}{0.32\linewidth}
\vspace{-\baselineskip}\small\centering
\caption{(Sec.~\ref{par:bench-3a}) Mean reward by meta-action rule on
\bench{$\Gamma$-Stress} $f_Q$ (NIG posterior, $K{=}100$, $10$ seeds); bold marks the winner.}
\label{tab:bench-3a}
\begin{tabular}{@{}lr@{}}
\toprule
Rule                         & Reward\\
\midrule
\arm{Thompson}            & $\mathbf{80.2}$\\
\arm{Greedy}              & $77.5$\\
\arm{D-opt-Hybrid}      & $10.0$\\
\arm{D-optimal}          & $7.0$\\
\bottomrule
\end{tabular}
\end{wraptable}
\emph{Question.} Across three meta-action rules (Thompson as in
Eq.~\eqref{eq:ck-ts}, D-optimal
$\arg\max_m \mathbb{E}_{p_{k-1}}[\log\det F(\theta;\,\doIntervention(M{=}m))]$,
greedy posterior-mean), plus a top-$k$ Thompson-tiebreak
hybrid, which survives at high inner-context noise? All four rules
use the NIG posterior ($K{=}100$, $10$ seeds).
\emph{Result} (Fig.~\ref{fig:exp:contrastive}a, Table~\ref{tab:bench-3a}).
\emph{Interpretation.} Pure D-optimal collapses to single-digit
reward: under high inner-context noise the info-gain objective
concentrates on $m$ values informative about $\theta_C$ but noisy on
the reward channel; the top-$k$ Thompson-tiebreak hybrid only
partially recovers. Nor does D-optimal buy a better fit of $f_Q$ here:
its mechanism MSE on $f_Q$ is $9.1$, against $9.8$ for Thompson.
The collapse is specific to the NIG posterior: with the RFF-GP backend,
pure D-optimal earns $112.4$ per episode against $115.8$ for Thompson
($K{=}200$).

\subsection{Invariance in $n_a$ for the factorised arms; \arm{Tabular-TS} goes negative.}\label{par:bench-3b}
\begin{wraptable}{r}{0.45\linewidth}
\vspace{-\baselineskip}\small\centering
\caption{(Sec.~\ref{par:bench-3b}) Mean reward across action-grid
resolutions $n_a$ on \bench{$\Gamma$-Stress} with quadratic reward ($q{=}1$).}
\label{tab:bench-3b}
\begin{tabular}{@{}lrrr@{}}
\toprule
Arm                  & $n_a{=}21$ & $41$ & $81$\\
\midrule
\arm{NCTS} & $67.49$    & $67.60$ & $68.00$\\
\arm{Factored-TS} & $66.56$    & $67.23$ & $67.47$\\
\arm{Tabular-TS} & ${-}8.99$  & ${-}8.10$ & ${-}5.09$\\
\bottomrule
\end{tabular}
\end{wraptable}
\emph{Question.} The bound predicts that factorised-arm reward depends on
the parameter count of the factorised model, $D\!=\!128$ RFF features per
mechanism and independent of action-grid resolution $n_a$, while a
tabular flat-joint basis has one parameter per cell of $|\mathcal{A}_S|\!=\!|\mathcal M|\!\cdot\!n_a$.
Does this invariance survive at the stress regime?
\emph{Setup.} \bench{$\Gamma$-Stress} with quadratic reward ($q{=}1$)
($Y\!=\!-(A{-}0.5\,Q)^2 + 4$); $A^\star\!=\!0.5\,Q$ is continuous in $Q$,
and the grid on $[-2, 2]$ has spacing $h\!=\!4/(n_a{-}1)\!\le\!0.2$, so
discretisation costs at most $(h/2)^2\!\le\!0.01$ reward per step against
the best action in $[-2, 2]$; the $n_a$ axis tests exploration cost up to
this term. Arms: \arm{NCTS} (RFF-GP, factorised + targeted),
\arm{Factored-TS} (RFF-GP, factorised, no targeting),
\arm{Tabular-TS} (tabular).
\emph{Result} (Fig.~\ref{fig:exp:contrastive}b, Table~\ref{tab:bench-3b}).
Factorised arms vary by $\Delta\!<\!1$ reward across the grid;
\arm{Tabular-TS} stays in the negative-reward floor at every $n_a$.
\emph{Interpretation.} Factorised arms are flat in $n_a$, the
invariance the bound predicts. Tabular goes negative
because with $K\!\cdot\!J\!\cdot\!I / (7\!\cdot\!n_a)$ samples per
cell, even $n_a\!=\!21$ leaves only $\sim\!14$ samples per arm against
$\mathrm{Var}(R)\!\sim\!9$, so the Thompson argmax is dominated by
noise. The $|\mathcal A_S|$-penalty the bound predicts within tabular
is swamped at the negative-reward floor; what is visible is the
$|\mathcal Z|$-vs-$|\mathcal A_S|$ \emph{level} gap of $\sim\!75$
reward between the structured factorised arms and unstructured tabular. The bound's
prediction holds at the level; the within-tabular $n_a$-scaling is
swamped by the noise.

\begin{figure}[H]\centering
  \begin{subfigure}[t]{0.48\linewidth}\centering
    \benchfig{bench_3a_meta_action_options/00_reward_nonlinear.pdf}
    \caption{Reward by meta-action rule, NIG posterior (Sec.~\ref{par:bench-3a}).}
  \end{subfigure}\hfill
  \begin{subfigure}[t]{0.48\linewidth}\centering
    \benchfig{bench_3b_ucb_z_scaling/00_reward_vs_n_a.pdf}
    \caption{Reward vs.\ $n_a$, three arms (Sec.~\ref{par:bench-3b}).}
  \end{subfigure}
  \caption{Contrastive comparisons: detailed setup, numerics, and
  interpretation in Sec.~\ref{par:bench-3a} and
  Sec.~\ref{par:bench-3b}.}
  \label{fig:exp:contrastive}
\end{figure}
\subsection{Effect attribution: shift rows}
\label{app:exp-attribution}
Nesting dominates in distribution, factorisation under shift.
Table~\ref{tab:effect-attribution} in the body carries the
in-distribution decomposition and the gate ablation.
Table~\ref{tab:attribution-shift} adds the Welch $t$ statistics and
the exogenous-shift rows of Sec.~\ref{par:bench-1a}. On the mild Shift-$U$
($\lambda_U\!:\!0.5\!\to\!0.9$) the factorised bandit attains
$116.1{\pm}5.0$ against $115.0{\pm}4.3$ ($\Delta{=}{+}1.2$): at
$K{=}2000$ the joint regression has accumulated enough data to
compensate on the mild shift. The gate itself is isolated on a bench
where it fires (Table~\ref{tab:effect-attribution}, bottom).

\begin{table}[h]
\centering\small
\caption{Attribution with Welch $t$, including the exogenous-shift
bench. Per-episode reward, mean $\pm$ sd over $10$ seeds.}
\label{tab:attribution-shift}
\resizebox{\linewidth}{!}{%
\begin{tabular}{@{}llccrrrr@{}}
\toprule
\textbf{Effect} & \textbf{Contrast} & \textbf{A} & \textbf{B} & $\Delta$ & $d$ & $t$ & $p$\\
\midrule
Factorisation & \arm{Factored-TS} vs \arm{Flat-TS} & $114.7{\pm}0.3$ & $111.5{\pm}0.5$ & ${+}3.3$ & $8.21$ & $18.36$ & $5.1{\times}10^{-11}$\\
Nesting (on top) & \arm{NCTS} vs \arm{Factored-TS} & $118.6{\pm}0.6$ & $114.7{\pm}0.3$ & ${+}3.9$ & $7.81$ & $17.46$ & $7.7{\times}10^{-10}$\\
Both (total) & \arm{NCTS} vs \arm{Flat-TS} & $118.6{\pm}0.6$ & $111.5{\pm}0.5$ & ${+}7.1$ & $12.34$ & $27.60$ & $1.7{\times}10^{-15}$\\
Nesting vs joint & \arm{NCTS} vs \arm{Joint-Commit} & $118.6{\pm}0.6$ & $77.4{\pm}3.4$ & ${+}41.2$ & $16.95$ & $37.90$ & $7.8{\times}10^{-12}$\\
\midrule
Factorisation (Shift-$T$) & \arm{Factored-TS} vs \arm{Flat-TS} & $113.6{\pm}4.9$ & $56.1{\pm}10.1$ & ${+}57.5$ & $7.2$ & $16.17$ & $5.8{\times}10^{-10}$\\
Factorisation (Shift-$\mathcal{N}_Q$) & \arm{Factored-TS} vs \arm{Flat-TS} & $194.7{\pm}4.9$ & $146.1{\pm}6.6$ & ${+}48.6$ & $8.4$ & $18.69$ & $1.3{\times}10^{-12}$\\
\bottomrule
\end{tabular}}
\end{table}

\paragraph{Gate ablation protocol.} The deployment rules \emph{never},
\emph{always} and AEGIS run on the meta level of \bench{Depth-Bin} at
$L{=}3$ (App.~\ref{app:depth-sweep}; binary action sets, uniform legacy,
interaction strength $b{=}2$, per-level weight ceiling $B_w{=}2$), over a $6{,}000$-episode horizon
and $10$ seeds, with the same certificate and $\delta$ as the depth
sweep. The good candidate is the learned meta policy (its posterior
mean; the certified Thompson mixture differs by at most $0.03$ per
episode); the bad
candidate is the same policy negated. AEGIS is not rolled out: its
reward is the legacy reward up to the crossing and the candidate reward
after it. Against the good candidate the
certified crossing occurs at $K{=}1{,}268$ to $4{,}872$; per-episode
rewards are $0.00{\pm}0.37$ (\emph{never}), $23.56{\pm}0.53$
(\emph{always}) and $10.86{\pm}3.90$ (AEGIS). Against the bad
candidate the gate never fires, so AEGIS coincides with \emph{never}
on every seed ($0.00{\pm}0.37$), while \emph{always} collects
${-}23.56{\pm}0.66$.

\subsection{Certified-gain crossings}
\label{app:exp-crossings}

\paragraph{Crossing scale.} The implemented per-level certificate is
the gain bound~\eqref{eq:gain-bound} with its two confidence legs made
explicit: the certified gain per episode is
$\widehat V^{\mathrm{LS}}_{\mathrm{pes}} - \widehat V^{\mathrm{IPS}}(\mu)$
minus a regulariser
$N_L (R_\mathrm{hi} - R_\mathrm{lo})\,(\KL_\Phi + \log(1/\delta_1))/(\lambda K)$
minus a Hoeffding width at confidence $\delta_2$, with
$\delta_1 + \delta_2 \leq \delta$. At the $j$-th checkpoint the
Hoeffding width is evaluated at confidence $\delta_2 \cdot 6/(\pi^2 j^2)$;
since $\sum_{j \geq 1} 6/(\pi^2 j^2) = 1$, a union bound makes it hold at
all checkpoints simultaneously, as the LS term does by Ville's
inequality. On the first four rows of Table~\ref{tab:crossings} each of the $L$ levels is certified at confidence $\delta/L$, as on the depth benches of App.~\ref{app:depth}. Algorithm~\ref{alg:ra-ncts} restarts a buffer at a flip, which starts a fresh supermartingale at a stopping time, and a level is gated in at most $L$ such epochs. A union bound over the at most $L^2$ pairs of level and epoch then certifies a run that checks a gate after a flip at $\delta/L^2$ per pair, or at any fixed weights that sum to at most $\delta$ over these pairs; no run here does so.
With the posterior learning, the LS sum at each checkpoint is evaluated
under the current posterior, the posterior whose KL enters the bound
(Theorem~\ref{thm:pac-bayes-causal}). Let $g$ be the candidate's per-step
advantage over the incumbent. Optimising $\lambda$, the bound crosses
zero, to leading order, at
\begin{equation}
  K_\mathrm{cross} \;\approx\;
  \Big(\frac{(R_\mathrm{hi} - R_\mathrm{lo})\,A}{g}\Big)^{2},
  \qquad
  A = B_w \sqrt{2\,(\KL_\Phi + \log(1/\delta_1))}
      + \sqrt{\log(1/\delta_2)/2},
  \label{eq:k-cross}
\end{equation}
and $N_L$ cancels. This takes $K_\mu = K$ and is an upper estimate: Effect~5 has $S \leq B_w$ in place of $B_w$. At a fixed $\lambda$ the Jensen floor of condition C1 applies instead. So $K_\mathrm{cross}$ is linear in the KL budget and
quadratic in $B_w$, the reward span and $1/g$. \bench{Contract}
(Sec.~\ref{par:bench-1c}), whose $\KL_\Phi$ grows with the feature
dimension, therefore contracts without crossing at $K{=}2000$. Benches
with a small exact $B_w$ and a low-dimensional mechanism class cross
within a few thousand episodes.

\paragraph{The benches.} Table~\ref{tab:crossings} lists the benches
on which the certified gain crosses zero at $\delta{=}0.05$ (\bench{Bin-3L}: App.~\ref{app:depth-nccb}). On the first
four rows the declared weight ceiling is the exact realised supremum
(attained at every checkpoint and never exceeded), so no weight is truncated and
Assumption~\ref{ass:bounded-weights}(ii) holds with this ceiling. There $\lambda$
is fitted once on the offline prior split and held constant, an
$\mathcal{H}_0$-measurable choice licensed by
Theorem~\ref{thm:pac-bayes-causal}. The two-level bench
\bench{Crossing} uses a uniform incumbent on a $3$-point leaf
grid, so $B_w{=}3$, with a degree-$2$ feature map, reward clipped to
$[-15, 15]$ and $N_L{=}20$. Only the leaf level certifies; the meta level
does not cross on any seed. The first four rows run with \eqref{eq:ra-inner} off, as Algorithm~\ref{alg:ra-ncts} specifies.

\begin{table}[h]
\centering\small
\caption{Certified-gain crossings at overall $\delta{=}0.05$ ($\delta/L$ per level). Gains are
certified per-episode lower bounds on $V(\pi) - V(\mu)$ at the
crossing checkpoint; the last column is the certificate later in the
run, with the gate suppressed so that it keeps being evaluated after the crossing. Frozen: posterior
fixed on the offline split, so $\KL \equiv 0$.}
\label{tab:crossings}
\resizebox{\linewidth}{!}{%
\begin{tabular}{@{}lllllll@{}}
\toprule
\textbf{Bench} & \textbf{Posterior} & \textbf{Cross} & $K_\mathrm{cross}$ & $\KL$ at cross & Gain at cross & Later\\
\midrule
\bench{Crossing} & live & $10/10$ & $3{,}000$ to $16{,}000$ & $2.55$ to $25.73$ & ${+}0.50$ to ${+}18.89$ & ${+}25.69$ to ${+}42.86$ at $52{,}000$\\
\bench{Crossing} & frozen & $10/10$ & $6{,}000$ to $7{,}000$ & $0$ & ${+}1.50$ to ${+}12.87$ & \\
non-additive noise & frozen & $3/3$ & $5{,}000$ & $0$ & ${+}16.74$, ${+}18.40$, ${+}17.75$ & \\
bounded reward & frozen & $1/1$ & $24{,}000$ & $0$ & ${+}4.915$ & ${+}115.927$ at $1{,}380{,}000$\\
nested, $L{=}3,4,5$ & frozen & \multicolumn{5}{l}{every level at every depth, $10$ seeds each (App.~\ref{app:depth})}\\
\bottomrule
\end{tabular}}
\end{table}

\paragraph{Live posterior.} With the posterior left learning, the KL
at the crossing is $2.55$ to $25.73$ nats and grows to $14.11$ to
$62.44$ nats by $K{=}52{,}000$. At the crossing it accounts for $40.4$
to $87.2\%$ of the regulariser. Two qualifications apply. First,
the live and frozen arms differ in more than the freeze: the prior-split
$\lambda$ budgets the KL it anticipates, which is zero under the freeze,
so switching the freeze off also raises $\lambda$ from about $0.0044$ to
about $0.030$ (a factor of $6.9$); both values are
$\mathcal{H}_0$-measurable. Second, the regulariser charges the
aggregate KL over all of $\Phi$, not the restriction to
$\Phi_{\mathcal{I}}$ that Lemma~\ref{lem:causal-kl} licenses, so the
quoted KL over-charges the certificate. The frozen arm crosses at
$K{=}6{,}000$ or $7{,}000$ on all ten seeds. The freeze buys a certificate
that is smoother in $K$, not one that fires where the live certificate
cannot.

\paragraph{Deployment trigger.} The gate as run deploys at the first
crossing. On $6$ of the $10$ live seeds the certified gain stays
positive at every checkpoint after the first crossing; on the other
$4$ it dips below zero at least once before staying positive. A
trigger that requires the gain to stay positive at every later
checkpoint, evaluated on the gate-suppressed counterfactual
trajectories, still fires on all ten seeds, at $K{=}3{,}000$ to
$16{,}000$, costing at most $11{,}000$ extra episodes on the worst
seed. We recommend this sustained trigger for deployment.

\paragraph{Non-additive noise.} The non-additive variant keeps the
hierarchy of \bench{Crossing} and makes the $Q$-noise scale with the
parent action, which violates the additive-noise form of the working
model by construction. It is a noise variant of the same hierarchy, not
a different environment family. With the gate suppressed, the leaf
certificate crosses at $K{=}5{,}000$ on all three seeds.

\paragraph{Bounded reward.} \bench{Bounded} has an almost-sure reward
envelope, so clause (i) of Assumption~\ref{ass:bounded-weights} holds by
derivation, not as a tail bound. The outcome mechanism uses a
four-coefficient interaction basis $[Q, A, QA, 1]$. The crossing lies
between two checkpoints of a $2{,}000$-episode grid (${-}5.101$ at
$K{=}22{,}000$, ${+}4.915$ at $K{=}24{,}000$), and all $679$ checkpoints
from there on are positive. At $K{=}1{,}380{,}000$ the certificate of
${+}115.927$ per episode is $96.4\%$ of the paired Monte Carlo
advantage of $120.198$ ($95\%$ interval $[119.892, 120.504]$). The
certified policy is context-dependent, and its advantage survives
setting the mean of $U$ to zero ($120.202$ against $120.198$).
This is a single run.
\section{Systematic Handover Evaluation}
\label{app:exp-aegis-flip}

The deployment story of NCCB / AEGIS is \emph{progressive certified
handover}: control passes from a legacy controller to the agent at
each timescale only once the per-level certificate meets a
practitioner-specified risk threshold. This appendix covers the
benches with a deterministic legacy. It reports (i) the reference
trajectory on \bench{$\Gamma$-Add}, which shows no certified crossing, and (ii) the consolidated handover statistics
(Table~\ref{tab:handover-stats}). On these benches the gate does not
fire. Section~\ref{sec:experiments} reports the benches
on which the gate fires, among them the \bench{Crossing} bench of
Fig.~\ref{fig:banner}\subref{fig:banner-bench}.

\subsection*{Reference trajectory: \bench{$\Gamma$-Add}}

\paragraph{Setup.} Environment: \bench{$\Gamma$-Add}
($\gamma\!=\!1$, $J\!=\!1$, $I\!=\!20$, reward clipped to $[-20, 20]$), a linear-additive variant
of \bench{$\Gamma$-SCM} in which $f_Y\!=\!2A + 0.5Q$ is correctly
specified by the agent's linear-NIG posterior. Legacy controller
$\mu_\mathrm{legacy}$: deterministic
$a^{(1)}_\mathrm{legacy}\!=\!2\,\mathrm{sign}(T)$, which matches the
oracle argmax $a^\star\!=\!+2$ on $50\%$ of meta-contexts. Agent:
\arm{AEGIS-NCTS} over the linear-NIG posterior, with LCB inner action
\eqref{eq:ra-inner} and per-level handover gate, whose own gate test is
replaced by the certificate below. Hyperparameters:
$\varepsilon$-greedy positivity on $M$ at $\varepsilon\!=\!0.05$;
truncated IPS with $w_\mathrm{max}\!=\!1.0$ (a capped variant of the weights of \citet{bottou2013counterfactual}); gain-bound check stride
$50$; certificate confidence $\delta = 0.5$ (a median-confidence
setting; at $\delta = 0.05$ every bound below is strictly lower). The
posterior is not frozen ($\KL$ reaches about $735$ nats by
$K = 5{,}000$). Budget: $K\!=\!5{,}000$ episodes $\times$ $3$ seeds.
The certificate is evaluated at every gain check in the form
Theorem~\ref{thm:pac-bayes-causal} licenses: one term per certified
episode at both levels (Lemma~\ref{lem:episode-unbiased}); the
expectation over $Q$ outside the logarithm, over $4$ posterior draws;
rewards mapped to $[0, 1]$ before weighting, so that
$B' = B_w + 10^{-3} = 1.001$ with $B_w = w_\mathrm{max} = 1$; the
data-free $\lambda = 1/(2B') = 0.4995$, since no prior split exists on
this bench; $\delta$ shared by the two levels; the Hoeffding leg
anytime over the $101$ checks; and the whole legacy log rescored under
the current posterior at every check, with that posterior's $\KL$. The
legacy holds control throughout; up to the first crossing of this
certificate this is the trajectory of a gate driven by it.

\paragraph{No certified crossing.} Under the licensed accounting the certificate
contracts from about $-4.8\times10^{4}$ at $K = 1$ to $-596.7$,
$-599.5$ and $-598.6$ at $K = 5{,}000$ at the inner level ($-978.0$,
$-969.2$ and $-956.2$ at the meta level), and never crosses $0$ up to $K = 5{,}000$.
Every weight is at most $1$ (deterministic
inner legacy; truncation at $w_\mathrm{max} = 1$ at the meta level), and with rewards mapped to $[0, 1]$ the coverage deficit is
then non-positive (Effect~5). The measured
inner deficit is $0$ at $K = 1$ on one seed, where every logged weight
equals $1$, and lies in $[-168.1, -142.8]$ from $K = 100$ on; the meta
deficit is at most $-307.5$ at every check. At $K = 5{,}000$ the
remaining terms are the regulariser (about $236$ per episode, from the
$\KL$ of the learning posterior), the Jensen gap at the data-free
$\lambda$ ($172$ to $173$ at the inner level, $302$ to $309$ at the meta
level) and the Hoeffding term ($27.5$). We claim no handover time for
this bench and report no pre/post-handover reward comparison,
because it would be conditioned on a gate
event the bound does not produce. The trajectory illustrates the
deployment contract (a per-level guarantee, a practitioner threshold,
a sticky handover) and the contraction of the bound, not a certified
firing.

\subsection*{Consolidated handover statistics}

Table~\ref{tab:handover-stats} collects the handover evidence. The
long-$K$ runs (one seed per $\sigma_Q$, $2 \times 10^5$ certified
episodes after $2 \times 10^5$ offline legacy episodes, $\delta = 0.5$)
certify in the form of Theorem~\ref{thm:pac-bayes-causal}: the
posterior is fitted on the offline split and frozen ($\KL = 0$);
$\lambda$ is Eq.~\eqref{eq:lam-star} on the offline split, held
constant ($2.91 \times 10^{-3}$ at every $\sigma_Q$); rewards are
mapped to $[0, 1]$ before weighting ($B' = 2.001$); both levels
aggregate one term per certified episode
(Lemma~\ref{lem:episode-unbiased}) and share $\delta$ ($\delta/2$ per
level); and the Hoeffding leg is anytime over the $200$ checkpoints.

\begin{table}[h]\footnotesize\centering\setlength{\tabcolsep}{4pt}
\caption{Firing on the deterministic-legacy benches; long-$K$
bounds are certificates (frozen posterior, prior-split $\lambda$).}
\label{tab:handover-stats}
\begin{tabular}{@{}p{0.42\linewidth}p{0.53\linewidth}@{}}
\toprule
Inner level, long-$K$ runs & no crossing at any $\sigma_Q \in \{1, 3, 6\}$ or checkpoint; from $K = 1{,}000$ to $200{,}000$ the certificate rises from ${-}602.0$ to ${-}157.0$ per episode at $\sigma_Q = 1$ (${-}601.9$ to ${-}157.0$ and ${-}602.4$ to ${-}157.4$ at $\sigma_Q = 3, 6$), below a coverage deficit of at most ${-}146.8$\\
Meta level & never fires in any seed; true meta advantage over the legacy: mean $-0.10$ per episode, $95\%$ CI $[-0.98, +0.77]$ ($20{,}000$ paired Monte Carlo episodes per arm)\\
Per-level firing order & not observable: neither level fires on these benches\\
\bottomrule
\end{tabular}
\end{table}

\paragraph{Reading.} The meta-level refusal is correct: the true meta advantage is indistinguishable
from zero. With a deterministic legacy and rewards mapped to $[0, 1]$
the coverage deficit is non-positive
(Effect~5), and no level fires at any checkpoint. The measured inner
deficit is at most ${-}146.8$, ${-}146.8$ and ${-}147.2$ per episode at
$\sigma_Q = 1, 3, 6$. Firing order and firing across $\sigma_Q$ are
therefore not observable on these benches. The three arms share the
prior-split temperature to within $0.1\%$, and their final inner
certificates differ by at most $0.4$ per episode: the deficit, which
$\sigma_Q$ barely moves, sets the level. AEGIS gates on the $(1{-}\delta)$ lower bound of the
per-level gain rather than on its empirical mean; the resulting
closed-gate stretch and its certified-safety price are discussed where
the gate is introduced (Section~\ref{sec:ra-ncts}).

\section{Extended Limitations, Graph Provenance and Future Work}
\label{app:ext-limitations}

\paragraph{Limitations.} Three limitations stand out. The first is within-level exogeneity (Assumption~\ref{ass:iid}). The assumption is exact for batch domains whose inner steps are parallel or exchangeable units under one decision (wafer lots, field plots under one input allocation, per-batch quality samples in process control), and state carried \emph{across} timescales is covered as well, since it is modelled by the parent-level context $C^{(\ell)}$, which the parent action shifts. Its boundary is treatment titration: in a stepwise dose adjustment, each inner step reacts to the response of the previous one, so the within-level context process is \emph{endogenous} in the within-level actions, which is precisely what the assumption excludes; the level-factored weight \eqref{eq:hybrid-weight} then no longer identifies the deployed policy's value from logged data. With an engineered self-destabilising candidate the level-factored target overstates the candidate's value; at the data-free temperature the bound's slack absorbs this overstatement at AR(1) coefficient $\rho = 0.9$, but not in a near-unit-root variant ($\rho = 0.95$, shock half-life $13.5$ instead of $6.6$ inner steps; state penalty $c = 0.7$ instead of $0.3$), and the certificate then falsely certifies and would trigger a false handover in every trial (measurements below). Titration and process control therefore need the carried state modelled in $C^{(\ell)}$, as our instantiation does. The natural route to the endogenous case replaces the level-factored weight by trajectory-product weights within each episode, whose overlap constant compounds over the $I_\ell$ inner steps; that extension is left to future work. The second limitation is the mechanism class: additive-noise models (linear with NIG, nonlinear with online Gibbs), with the non-additive case exercised only as a misspecification stress test rather than modelled directly. The third is evaluation scope: exact validity verification remains synthetic, since checking that a certificate holds requires ground-truth policy values, which external logs do not provide; on a real log, coverage can be checked (Section~\ref{sec:exp-external}), exact validity cannot.

\paragraph{Measured boundary of Assumption~\ref{ass:iid}.}
A one-dimensional linear state $s_{i+1} = \rho\, s_i + \beta\, a_i + \varepsilon_i$ (an AR(1) process within the level; for $\beta = 0$ the autoregressive coefficient $\rho$ is the lag-1 autocorrelation of the state and sets its persistence: a shock decays to half after $\ln 2/\ln(1/\rho)$ inner steps, $6.6$ at $\rho = 0.9$ and $13.5$ at $\rho = 0.95$, against $I_\ell = 6$ steps per level) placed inside one level separates the two within-level demands of the assumption ($\beta = 0$: exogenous serial dependence; $\beta > 0$: state driven by the within-level actions). The reward $a_i s_i - c\, s_i^2$ is clipped to $[-4, 2]$ and used signed (no rescaling); with the uniform logger over $\{-1, 0, 1\}$ the level-factored weight~\eqref{eq:hybrid-weight} is at most $B_w = 3$, so the declared loss range gives $B' = 3 \cdot 6 + 10^{-3} = 18.001$, and the temperature is the data-free $\lambda = 1/(2B') = 0.0278$, fixed before any data as Theorem~\ref{thm:pac-bayes-causal} requires ($\mathrm{KL} = 0$, $K = 4000$, $I_\ell = 6$, $\delta = 0.05$; values per inner step). The column ``target'' is what the level-factored certificate estimates, $\mathbb{E}_\mu[\tfrac1{I_\ell}\sum_i w_i r_i]$ ($2\times10^6$ Monte Carlo episodes); the certificate is sound only where it does not exceed $V(\pi)$ by more than the bound's slack.

\begin{center}
\begin{tabular}{rrrrrrrr}
$\rho$ & $\beta$ & $c$ & trials & $V(\pi)$ & target & mean $V_{\mathrm{lower}}$ & false cert.\ \\
\hline
$0.9$  & $0.0$ & $0.3$ & $60$ & $0.51$  & $0.51$  & $0.01$  & $0\%$ \\
$0.9$  & $0.8$ & $0.3$ & $60$ & $0.17$  & $0.52$  & $0.03$  & $0\%$ \\
$0.95$ & $0.8$ & $0.7$ & $40$ & $-1.95$ & $-0.31$ & $-0.99$ & $100\%$ \\
\end{tabular}
\end{center}

Serial dependence alone never falsely certifies: the target equals $V(\pi)$ (0 of 60 trials at each $\rho \in \{0, 0.5, 0.9\}$; violation rate below $4.87\%$ at one-sided 95\% confidence). Action-driven state breaks identification. At $\rho = 0.9$, $\beta = 0.8$ the target overstates $V(\pi)$ by $0.35$ per step, less than the logarithmic-smoothing pessimism at $\lambda = 1/(2B')$ (about $0.5$ per step, not shrinking in $K$), so no trial falsely certifies: coverage through slack, not detection. A smaller temperature exposes the failure: the closed-form $\lambda^\star$ fitted on the certification sample (outside the theorem) falsely certifies in $60/60$ trials, and $\lambda^\star$ rounded to the grid of Remark~\ref{rem:lam-grid} ($\lambda = 0.0066$) in $58/60$. In the near-unit-root variant (third row: $\rho = 0.95$, $c = 0.7$) the overstatement is $1.64$ and exceeds the slack: every trial falsely certifies, and since $V_{\mathrm{lower}} = -0.99 > V(\mu) = -1.17 > V(\pi) = -1.95$, every trial would trigger a false handover. The candidate is engineered and self-destabilising ($\pi = \mathrm{sign}(s)$, myopically attractive); at $\rho \in \{0, 0.5\}$, $\beta = 0.8$ the target lies below $V(\pi)$ and the certificate stays valid.

Two assumptions are stress-tested by the paper's experiments: posterior-class misspecification in Appendix~\ref{par:bench-2a} (a shape-misspecified NIG backend against the matched RFF-GP class costs about $34$ reward, not validity, since the bound holds for every posterior) and the additive-noise violation (frozen posterior, $3$ seeds; App.~\ref{app:exp-crossings}). Shared within-level confounding enters only \bench{Bin-3L} and \bench{Depth-Bin}, with a frozen posterior (App.~\ref{app:depth}); the working model of Remark~\ref{rem:working-model} is not tested under it.

NCTS is a baseline; the NCCB class admits a wider design space along two axes. First, sharper bounds: PRISM's multiplicative overlap penalty $B^{\mathcal{I}}$ is sidestepped level by level by the progressive workflow of Section~\ref{sec:per-level} but remains loose under one-shot full deployment, and self-normalised importance sampling \citep{hesterberg1995weighted,swaminathan2015self}, PAC-Bayes-Bernstein refinements on the per-mechanism variance structure, and self-normalised maximal inequalities \citep{girard2025online} are promising tighteners. Second, alternative NC-X hosts, including NC-UCB \citep{lu2022efficient}, NC-IDS \citep{russo2014ids}, NC-PB (direct PAC-Bayes minimisation), NC-DOpt \citep{chaloner1995bayesian}, and pessimistic warm-starts \citep{aouali2024unified}, all wrapped by AEGIS (Algorithm~\ref{alg:ra-ncts}) via its host-agnostic interface. Further open problems (AEGIS-conservativeness reduction, graph learning, continuous domains, adaptive episode length, confounder estimation, real-data deployment) are collected in App.~\ref{app:further-open}.

\paragraph{Where the graph comes from.} When is such an I-map available? In the motivating domains the superset structure is typically documented rather than discovered: process engineering fixes mechanism structure in plant flowsheets and P\&IDs, clinical pharmacology in trial-established dose-response pathways, and agronomy in managed input experiments. Each such document is signed off by a domain expert and may safely contain edges the true mechanism does not use, since only an I-map is required; none must get the mechanism form right, that being a working model whose misspecification costs tightness and not validity (Remark~\ref{rem:working-model}). The strongest case is \emph{randomised logging}, as in micro-randomised trials and A/B logs: there the logged action probabilities depend on nothing outside the log by experimental design, so the logging condition of Section~\ref{sec:graph-error} is met by construction. The fully unknown-graph case is the graph-posterior extension sketched in App.~\ref{app:further-open}.

\section{Depth: the $L = 3$ bench and the depth sweep}
\label{app:depth}

Most body experiments run at $L = 2$. This appendix asks two separate questions at greater depth. First, does NCTS still behave as Algorithm~\ref{alg:ncts} intends when a third level is added (Section~\ref{app:depth-reward})? Second, how does the certificate of Theorem~\ref{thm:pac-bayes-causal} behave as $L$ grows, given that the importance weight of Eq.~\eqref{eq:hybrid-weight} is a product over handed-over levels (Sections~\ref{app:depth-l3-cert}, \ref{app:depth-compounding} and~\ref{app:depth-firing})? Throughout, $\delta = 0.05$ unless stated otherwise, effect sizes are Cohen's $d$ with pooled standard deviation, and tests are Welch tests over seed means. Table~\ref{tab:depth-regimes} records, for every exhibit, whether the posterior is learning or frozen, since a frozen posterior sets the KL term to zero by construction.

\begin{table}[h]
\centering
\small
\caption{Depth exhibits and their posterior regime. ``Frozen'': the posterior is fitted on a disjoint offline split and pinned as its own prior, so $\KL = 0$ exactly. ``Learned'': a non-zero mechanism KL is charged.}
\label{tab:depth-regimes}
\setlength{\tabcolsep}{3.5pt}\begin{tabular}{llll}
\toprule
Exhibit & Depth & Posterior & Certificate crosses zero? \\
\midrule
Reward triple (\S\ref{app:depth-reward}) & $3$ & online NCTS (learning) & not a certificate \\
\bench{Depth-Grid}, deterministic legacies (\S\ref{app:depth-l3-cert}) & $3$ & frozen & no, up to $K = 2 \times 10^{5}$ \\
Per-scope compounding, \bench{Depth-Grid} (\S\ref{app:depth-compounding}) & $3$ & learned & no, at $K = 2{,}000$ \\
\bench{Bin-3L} (\S\ref{app:depth-nccb}) & $3$ & frozen & yes, $9/9$ and $30/30$ \\
\bench{Depth-Bin} (\S\ref{app:depth-sweep}) & $3, 4, 5$ & frozen & yes, $120/120$ \\
\bottomrule
\end{tabular}
\end{table}

The live-KL crossing, in which the certified gain crosses zero with the posterior still learning, is an $L = 2$ result and is reported with the other crossing exhibits; nothing in this appendix depends on it.

\subsection{The three-level \bench{Depth-Grid} bench}
\label{app:depth-scm}

The three-level bench \bench{Depth-Grid} extends \bench{$\Gamma$-SCM} by an outer strategic level whose action shifts the distribution of the meta-context. In the paper's inside-out indexing, $C^{(3)} = T$, $A^{(3)} = M$ (outer, strategic), $C^{(2)} = Q$, $A^{(2)}$ (mid), and $C^{(1)} = S$, $A^{(1)} = B$ (leaf), with hierarchy $(T, M) \to (Q, A^{(2)}) \to (S, B) \to Y \to R$. Under an intervention on all three actions the structural equations are
\begin{align}
  T &= \varepsilon_T, \label{eq:l3-scm-t}\\
  Q &= \gamma\bigl[(1 - \eta)\, 1.5\, M + \eta\, 2 \tanh(2M)\bigr] + \varepsilon_Q + 0.3\,U, \label{eq:l3-scm-q}\\
  S &= 1.5\,\gamma\, A^{(2)} + 0.5\, Q + \varepsilon_S + 0.3\,U, \label{eq:l3-scm-s}\\
  Y &= B \cdot S + 0.3\,U + \varepsilon_Y, \qquad R = Y + \varepsilon_R, \label{eq:l3-scm-yr}
\end{align}
with $U \sim \mathrm{Ber}(0.5)$ a latent noise term, independent Gaussian noises $\varepsilon_T \sim \mathcal{N}(0, 0.3^2)$, $\varepsilon_Q \sim \mathcal{N}(0, \sigma_Q^2)$, $\varepsilon_S \sim \mathcal{N}(0, \sigma_S^2)$, $\varepsilon_Y \sim \mathcal{N}(0, 0.08^2)$, $\varepsilon_R \sim \mathcal{N}(0, 0.5^2)$, and actions $M, A^{(2)}, B \in [-2, 2]$. The \bench{$\Gamma$-Stress} preset uses $\gamma = 1$, nonlinearity $\eta = 1$, $\sigma_Q = 3$ and $\sigma_S = 0.04$; the episode has one outer step, four mid steps per outer step and five leaf steps per mid step, so $N_L = 20$ leaf rewards per episode. The leaf mechanism is bilinear, $\mathbb{E}[Y \mid B, S] = B S$ up to the offset from $U$, so the optimal leaf action depends on the realised leaf context, and the realised mid context $Q$ enters $S$ with coefficient $0.5$.

The environment implements the context semantics of Algorithm~\ref{alg:ncts}: an ancestor context realised at its decision time is held fixed for all descendant decisions within that decision, so each leaf context is drawn from $f_S$ at the \emph{observed} $Q$ and the committed $A^{(2)}$. As a check, the ordinary least-squares slope of the leaf $S$ on the observed $Q$ is $0.5013$ against the structural coefficient $0.5$. The declared reward envelope is $R \in [-25, 25]$ per leaf step; it was fixed before the runs below but calibrated on a pilot scan of corner interventions in the same domain, which observed $R \in [-20.2, 20.5]$, so it is a pilot-calibrated envelope rather than one derived from the model. Every certificate on this bench uses $R$ clipped to this envelope inside the SCM, so Assumption~\ref{ass:bounded-weights}(i) holds by construction.

\subsection{Reward at $L = 3$}
\label{app:depth-reward}

We compare three arms on the \bench{$\Gamma$-Stress} preset ($K = 2{,}000$ episodes per seed, 10 seeds, all arms in one batch) with a matched function class (random Fourier features, 64 features, NIG posterior) and identical one-draw Thompson exploration in every arm: the nested factorised agent (pure NCTS), a joint-commit ablation that commits all three actions at episode start on a single value model over $(T, M, A^{(2)}, B)$ and never observes $Q$ or $S$, and a structure-free flat agent with the factorisation switched off.

\begin{table}[h]
\centering
\small
\caption{$L = 3$ reward triple: mean episode reward over all $K = 2{,}000$ episodes and over the last 25\%, mean $\pm$ sd over 10 seeds.}
\label{tab:depth-reward}
\begin{tabular}{lcc}
\toprule
Arm & Mean reward & Last 25\% \\
\midrule
Nested (NCTS) & $+180.52 \pm 6.90$ & $+181.91 \pm 7.00$ \\
Joint commit & $+167.75 \pm 1.37$ & $+171.14 \pm 0.82$ \\
Flat (no factorisation) & $+46.61 \pm 6.42$ & $+51.16 \pm 10.88$ \\
\midrule
\multicolumn{3}{l}{Nested vs.\ flat: $\Delta = +133.91$, $d = 20.08$, $t(17.9) = 44.91$, $p < .001$} \\
\multicolumn{3}{l}{Nested vs.\ joint: $\Delta = +12.77$, $d = 2.56$, $t(9.7) = 5.74$, $p < .001$ (permutation $p = 1.2 \times 10^{-4}$)} \\
\multicolumn{3}{l}{Joint vs.\ flat: $\Delta = +121.14$, $d = 26.09$, $t(9.8) = 58.34$, $p < .001$} \\
\bottomrule
\end{tabular}
\end{table}

The factorisation axis carries most of the gain at depth 3: the 10 nested and 10 flat seed means separate completely. One caveat applies to this contrast. The flat arm has no mechanism model of $Q \mid M$ or $S \mid (Q, A^{(2)})$ and imputes unrevealed deeper contexts at $0$ at decision time, so an unquantified part of the $+133.91$ is a handicap of the baseline rather than pure factorisation.

On this bench, re-deciding on realised contexts beats committing all levels at episode start by $+12.77$ per episode. Separation is near-complete rather than complete: 9 of 10 nested seed means exceed the joint maximum, and one nested seed ($+161.15$) lies below the joint minimum. We report this as a measurement on this bench and do not claim it as a law in $L$; the $L = 2$ ablation of Section~\ref{par:bench-1b} remains the primary evidence for recursive re-decision. The contrast depends on the context semantics above: a run in which every descendant decision redraws the whole SCM, which makes the mid context causally inert, reverses its sign, so the joint-commit comparison is informative only under the semantics Algorithm~\ref{alg:ncts} specifies. Reproduction is statistical rather than bitwise, because one posterior draw consumes an unseeded random stream; means and verdicts reproduce, exact per-seed values do not.

\subsection{Certificate on \bench{Depth-Grid}: why no scope crosses}
\label{app:depth-l3-cert}

On the bench of Section~\ref{app:depth-scm} the legacies are deterministic at all three levels. A deterministic legacy has $\mu^{(\ell)}(a^{\mu} \mid \tilde c^{(\ell)}) = 1$ at every logged action, so every importance weight is at most $1$. Since rewards are mapped to $[0, 1]$ before weighting, the coverage deficit of the gain bound is then non-positive at every $\lambda$ and every $K$ (Appendix~\ref{app:offline-effect}, Effect~5), and no episode budget can produce a crossing. A certified long-$K$ run measures the trajectory (one seed, $2 \times 10^5$ certified episodes after $2 \times 10^5$ offline legacy episodes, reward clipped to $[-25, 25]$). The posterior is fitted on the offline split and frozen, so $\KL = 0$; the temperature $\lambda = 2.85 \times 10^{-3}$ is Eq.~\eqref{eq:lam-star} evaluated on the offline split and held constant; $B' = B_w + 10^{-3} = 2.001$ with $B_w = 2$; every level aggregates one term per certified episode (Lemma~\ref{lem:episode-unbiased}); the three levels share $\delta = 0.5$ ($\delta/3$ each), and the Hoeffding leg is anytime over the $400$ checkpoints. The largest coverage deficit over all checkpoints is ${-}485.6$ (outer), ${-}495.0$ (mid) and ${-}245.3$ (leaf) per episode, and no level crosses. From the first checkpoint ($K = 500$) to the last, the certified bounds rise from ${-}2{,}305.7$, ${-}2{,}303.9$ and ${-}2{,}048.3$ to ${-}513.6$, ${-}511.8$ and ${-}261.8$ per episode (outer, mid, leaf), a $4.5$-, $4.5$- and $7.8$-fold contraction towards the respective deficits; the leaf, whose deficit is smallest in magnitude, is the tightest at every checkpoint. The non-crossing is a property of the legacy and not of the depth; Section~\ref{app:depth-firing} holds the depth at three and replaces the legacy with one that randomises.

\subsection{Compounding: the KL adds, the overlap multiplies}
\label{app:depth-compounding}

The two penalties of Theorem~\ref{thm:pac-bayes-causal} scale differently in $L$. At levels the legacy still runs, the deployed and logging policies coincide and their ratio cancels in Eq.~\eqref{eq:hybrid-weight}, so a single-level scope pays one overlap factor at any depth, while one-shot full handover pays the product $\prod_{\ell} B^{(\ell)}_w$. The regulariser, by contrast, depends on the scope only through $\max(\mathcal{I})$ (Definition~\ref{def:intervention-scope}): the deepest single-level scope $\{L\}$ and the full scope $\{1, \ldots, L\}$ charge an identical KL, so at a temperature $\lambda \propto 1/B^{\mathcal{I}}$, as used below, the regulariser of the full scope exceeds that of $\{L\}$ by exactly the overlap factor $\prod_{\ell < L} B^{(\ell)}_w$, while the width of both also carries a legacy-side concentration term that does not depend on the scope.

We measure this on \bench{Depth-Grid} at the \bench{$\Gamma$-Stress} preset of Section~\ref{app:depth-scm}, with its deterministic legacies replaced at every level by a legacy that samples uniformly at random from the $n_a = 7$ points of the level's action grid, the same grid the certified policy plays on, so that every scope has overlap ($10$ seeds, $K = 2{,}000$, $\delta = 0.05$, 7-point action grids on $[-2, 2]$, reward clipped to $[-25, 25]$). The certified policy is the nested agent with linear features, trained on-policy for $400$ episodes separately from the certification data; every scope is then certified from the same logged legacy stream. Each seed retrains the policy and draws a fresh legacy log. The temperature is the data-free $\lambda = 1/(2B')$ with $B' = B_w^{\mathcal{I}}(R_{\mathrm{hi}} - R_{\mathrm{lo}})$, fixed per scope before the data, as on \bench{Contract} (App.~\ref{app:exp-contrastive}). The legacy propensity is exactly $1/n_a$ at every logged action, so the per-level weight is $n_a\, \pi^{(\ell)}(a)$ and the per-level ceiling is $B^{(\ell)}_w = n_a = 7$. The bound uses the declared ceiling $B_w^{\mathcal{I}} = \prod_{\ell \in \mathcal{I}} B^{(\ell)}_w = 7^{|\mathcal{I}|}$ in $B'$ and $\lambda$, not the largest weight observed in the log. At $\{1\}$, $\{2\}$ and $\{1, 2\}$ the largest observed weight equals this ceiling in every seed: the certified policy is deterministic at levels $1$ and $2$, so every logged action that matches its choice has weight $7$ per level. At the scopes containing level $3$ it stays below the ceiling, because the policy is stochastic there (largest action probability below $1$).

\begin{table}[h]
\centering
\small
\caption{Per-scope certificates at $L = 3$: declared and realised overlap constant $B_w^{\mathcal{I}}$, mechanisms in $\Phi_{\mathcal{I}}$, regulariser and width ($K = 2{,}000$, $\delta = 0.05$, data-free $\lambda = 1/(2B')$; mean $\pm$ standard deviation over $10$ seeds, realised $B_w^{\mathcal{I}}$ as the range over seeds). Levels are indexed inside out, so $\ell = 1$ is the leaf.}
\label{tab:depth-compounding}
\begin{tabular}{cccccc}
\toprule
Scope $\mathcal{I}$ & $B_w^{\mathcal{I}}$ declared & $B_w^{\mathcal{I}}$ realised & $\Phi_{\mathcal{I}}$ & KL (nats) & Width \\
\midrule
$\{1\}$ & $7$ & $7.000$ & $Y, R$ & $386 \pm 225$ & $2{,}755 \pm 1{,}573$ \\
$\{2\}$ & $7$ & $7.000$ & $S, Y, R$ & $481 \pm 224$ & $3{,}425 \pm 1{,}571$ \\
$\{3\}$ & $7$ & $3.938$ to $4.813$ & $Q, S, Y, R$ & $496 \pm 224$ & $3{,}530 \pm 1{,}571$ \\
$\{1, 2\}$ & $49$ & $49.000$ & $S, Y, R$ & $481 \pm 224$ & $23{,}795 \pm 11{,}000$ \\
$\{1, 2, 3\}$ & $343$ & $192.9$ to $235.8$ & $Q, S, Y, R$ & $496 \pm 224$ & $171{,}507 \pm 76{,}977$ \\
\bottomrule
\end{tabular}
\end{table}

The comparison is clean because the scopes pair up at identical regularisers (Table~\ref{tab:depth-compounding}). In every seed, $\{2\}$ and $\{1, 2\}$ charge the same KL, and their widths stand in the ratio $6.93 \pm 0.05$ (range $6.83$ to $6.97$) against a predicted $7$; $\{3\}$ and $\{1, 2, 3\}$ charge the same KL, and their widths stand in the ratio $48.47 \pm 0.33$ (range $47.77$ to $48.77$) against a predicted $7^{L-1} = 49$. Within each pair nothing but the overlap product separates the two scopes. At the data-free temperature the regulariser $(\mathrm{KL} + \log(2/\delta))/(\lambda K)$ is linear in $B'$, so the regularisers stand in the ratios $7$ and $49$ in every seed, up to rounding; the small shortfall of the widths comes from the legacy-side concentration term ($30.37$ per episode), which does not depend on the scope. The ratios are therefore stable although the KL itself varies by about $225$ nats across seeds with the trained policy. With the policy held fixed, $10$ fresh legacy logs leave every width unchanged, since no term of the width reads the data, and move the lower bounds by at most $62.6$ (standard deviation at most $17.8$, both at $\{1, 2, 3\}$), so the spread across seeds comes from training, not from the log. Each single-level scope pays one factor of $7$, at mean widths from $2{,}755$ to $3{,}530$; their spread comes from the regulariser alone, as $\Phi_{\mathcal{I}}$ takes in one more context mechanism per level, and we claim no fixed multiple for it. The theory predicts the same $7^{L-1}$ separation at every depth; we quote the measured point at $L = 3$ rather than the extrapolation. No scope crosses zero at $K = 2{,}000$ on this bench, in any seed (the highest lower bound is $-916.9$, at $\{1\}$). AEGIS never pays the full product: it certifies each level at single-level scope, so the full-scope factor is not charged; two levels deployed in turn are certified one at a time, the second on a buffer collected after the first handover (App.~\ref{app:ra-ncts-design}).

\subsection{Firing at depth: \bench{Bin-3L} and \bench{Depth-Bin}}
\label{app:depth-firing}

Both benches below use binary action sets and uniform legacies, so that $B^{(\ell)}_w = 2$ per level exactly and the composed ceiling is $2^L$. In both, the posterior is fitted on a disjoint offline split and frozen, so $\KL = 0$ by construction; these are therefore crossings of the overlap, Hoeffding and logarithmic-smoothing terms, not of the regulariser. Each per-level certificate is evaluated at $\delta/L$, so that the $L$ per-level statements tested by the gate hold jointly at $\delta = 0.05$; the composed certificate of Section~\ref{app:depth-nccb} is a single statement, not tested by the gate, and is evaluated at $\delta$. Bounded rewards (Assumption~\ref{ass:bounded-weights}(i)) hold exactly because contexts are squashed by $\tanh$ and the environment clips rewards to the declared range $[-3, 3]$; Assumption~\ref{ass:bounded-weights}(ii) holds by measurement, the realised weight attaining the declared ceiling at every level and seed without exceeding it.

\subsubsection{\bench{Bin-3L}: a three-level NCCB with a randomising legacy}
\label{app:depth-nccb}

The nesting follows the two-level nested-timescale testbed of \citet{cr_li2023dyadic}, extended by one level and re-implemented from the published description: a cohort action reshapes the weekly context, a weekly action reshapes the daily context, and the daily action produces a bounded reward, with no separate outcome $Y$,
\begin{equation}
  \begin{gathered}
  X_0 \text{ exogenous}, \quad X_1 = f_1(X_0, A_0, \varepsilon_1), \quad X_2 = f_2(X_1, A_1, U, \varepsilon_2), \\
  R = \mathrm{clip}_{[-3,3]}\Bigl(\textstyle\sum_{\ell} b_\ell\, A_\ell\, (\vartheta_\ell^\top X_\ell) + c\, U + \varepsilon_R\Bigr),
  \end{gathered}
  \label{eq:depth-nccb}
\end{equation}
with a latent $U = \tanh(Z)$, $Z \sim \mathcal{N}(0, 1)$, drawn per leaf step and shared by the leaf context $X_2$ and the reward with loading $c = 0.3$ (confounding within the leaf level only, Assumption~\ref{ass:markov}), $A_\ell \in \{-1, +1\}$ (outer level indexed $0$ here), four weeks of seven days ($N_L = 28$ leaf steps per episode), and the interaction feature map $[1, x, a, a x]$. Uniform legacy at every level, $K_{\mathrm{cert}} = 40{,}000$ certified episodes, $\lambda$ chosen from an a-priori 16-point geometric grid closed by a union bound costing $\log 16$ nats.

\emph{Per-level certificates.} All nine per-level certified gains (three seeds by three levels) cross zero, at $K$ between $4{,}597$ and $8{,}517$; over ten seeds the crossing holds 30 times out of 30, at $K$ between $4{,}411$ and $9{,}280$. A no-interaction control (the same bench with $a x$ removed from the feature map) crosses at no level and no seed, so the gate is seen both to fire and to decline.

\emph{Composed policy.} A per-level certificate flips one level and leaves the rest logging, so its overlap constant is $2$ regardless of $L$ and nothing compounds. The question of how the certificate behaves as the weight compounds is answered by the composed three-level policy, whose per-leaf-step weight is the product of three per-level weights with ceiling $2^3 = 8$, realised exactly. It crosses zero at a $K$ in $(9{,}477, 15{,}315]$ on all three seeds, between $1.6$ and $2.5$ times the median per-level budget rather than a blow-up, with the composed certificate at $\delta$ and the per-level ones at $\delta/L$. At $K = 40{,}000$ the certified gain is $+7.07$ to $+9.83$ per level and $+17.63$ to $+22.92$ composed, against measured advantages of $14.11$ to $16.41$ and $42.18$ to $47.03$: the bound recovers $49\%$ to $60\%$ of the measured advantage per level and $42\%$ to $49\%$ composed.

\emph{What is designed in.} The bench is synthetic and its favourable properties were designed and then verified. The context-dependence of the optimal action is built in by the interaction term; that the fitted policy recovers it is measured (agreement with the oracle $0.992$ to $0.999$ over ten seeds). The uniform legacy is also the randomising legacy that minimises $B_w$ on two actions. Raising only the declared $B_w$ from $2$ to $4$ multiplies the crossing $K$ by $4.86$ to $6.71$ (mean $5.37$; $K_{\mathrm{cert}} = 80{,}000$ so that every level crosses at $B_w = 4$), a rise of the order of the $B_w^2$ factor the penalty predicts. Against a competent floor-softmax incumbent aligned with the true reward direction (floor $0.25$, declared $B_w = 4$) no level crosses at $K = 40{,}000$, but all three cross at $K = 78{,}049$ to $84{,}598$ (one seed), about $13$ times the uniform-incumbent budget, of which a factor of about $5.4$ is the overlap constant and the rest the smaller advantage left to certify. This bench is a synthetic demonstration of the mechanism, not evidence of field-deployability: priced at published effect sizes of two mobile-health settings (the Intern Health Study \citep{necamp2020assessing,wang2023effectiveness} and the nested-timescale testbed of \citet{cr_li2023dyadic}), the certificate would need between about $300$ and $3.5 \times 10^5$ times the data those settings provide.

\subsubsection{Depth sweep at $L = 3$, $4$ and $5$}
\label{app:depth-sweep}

A second binary-action bench, \bench{Depth-Bin}, varies the depth directly with the mechanisms of Eq.~\eqref{eq:depth-nccb}, including the per-leaf-step $U$ (10 seeds per depth, interaction strength $b = 2$, declared reward range $[-3, 3]$, offline, legacy and certification splits of $1{,}500$, $3{,}000$ and $6{,}000$ episodes).

\begin{table}[h]
\centering
\small
\caption{Depth sweep with frozen posterior ($\KL = 0$), $\delta = 0.05$ ($\delta/L$ per level), 10 seeds per depth. ``Median $K$'' is the range over levels of the per-level median crossing episode. Clip: fraction of rewards clipped to the declared range.}
\label{tab:depth-sweep}
\begin{tabular}{cccccc}
\toprule
$L$ & Levels firing & Median $K$ at crossing & Mean certified gain & Clip & Composed ceiling $2^L$ \\
\midrule
$3$ & $30/30$ & $1{,}930$ to $3{,}509$ & $+7.6$ & $0.196$ & $8$ \\
$4$ & $40/40$ & $136$ to $182$ & $+76.7$ & $0.259$ & $16$ \\
$5$ & $50/50$ & $134$ to $179$ & $+162.3$ & $0.320$ & $32$ \\
\bottomrule
\end{tabular}
\end{table}

Every one of the 120 per-level certificates crosses zero (Table~\ref{tab:depth-sweep}). The crossing budget falls with $L$ on this bench, and we claim no trend: the deeper benches carry more total signal, since each added level contributes its own action-by-context interaction, so depth buys both a larger overlap product and a larger available gain. Clipping binds on a fraction of rewards that grows with depth. What the sweep shows is that the per-level certificate does not degrade with depth on this bench: the same gate at the same joint $\delta$ fires at every level of a five-level hierarchy.

\paragraph{Across these two benches, per-level overlap rather than depth tracks feasibility.}
Taken together, the two benches separate depth from overlap. Nothing crosses on \bench{Depth-Grid} (Section~\ref{app:depth-compounding}), where $B^{(\ell)}_w = n_a = 7$ makes the full-scope product $7^L$ ($343$ at $L = 3$); everything crosses to $L = 5$ on \bench{Depth-Bin}, where $B^{(\ell)}_w = 2$ makes it $2^L$. $L$ enters only as the exponent, so a practitioner controls feasibility through the per-level action resolution, which is a design-time choice, and not through the depth, which the problem fixes. The benches differ in more than this constant (reward scale, interaction strength, horizon), so we report this as the pattern across both benches and not as the isolation of a single variable.

\section{Further Open Problems}
\label{app:further-open}

Appendix~\ref{app:ext-limitations} discusses the two open directions
closest to the present results (tighter bounds and alternative NC-X
hosts). The following are further open problems
that the NCCB framework opens but which are out of the present
paper's scope.

\begin{itemize}[leftmargin=1.4em,itemsep=2pt]
\item \textbf{Reducing AEGIS's certified-safety conservativeness.}
AEGIS gates on a $(1{-}\delta)$ lower bound, so the handover can stay
closed while $V(\pi_{Q_K}^{(\ell)})$ already exceeds $V(\mu^{(\ell)}_\mathrm{legacy})$
in expectation. Tightening the gap (PAC-Bayes-Bernstein refinements
that absorb low per-step variance, dependence-aware overlap bounds,
or a calibrated $(\varepsilon_\ell, \delta)$-schedule that anneals
the guarantee level as $K$ grows) without losing the anytime
guarantee is open.
\item \textbf{Joint graph-structure and parameter learning.} Replace
the known-graph assumption with a posterior over graphs (BayesDAG
\citep{annadani2023bayesdag} or DAG-NoCurl \citep{yu2021dags}); the
resulting bound decomposes as $\KL(\hat Q(\mathcal{G})\|Q_0(\mathcal{G}))
+ \mathbb{E}_{\hat Q(\mathcal{G})}[\sum_X \KL(Q_X\|p_{0,X})]$
marginalised over the graph posterior; online graph-discovery is open.
\item \textbf{Continuous variable domains.} Extend the analysis from
discrete-grid action and meta-action spaces to continuous domains
via covering-number / metric-entropy arguments, with parametric
policy networks $\pi_\theta$ and $\KL(\pi_\theta\|\pi_0)$ as the
PAC-Bayes complexity term; in the spirit of PAC-Bayes deep RL but
with the mechanism-factorised prior preserved.
\item \textbf{Adaptive episode length.} Turn the between-episode
problem into an optimal-stopping / semi-MDP with switching costs,
deciding when to commit a fresh meta-action versus continuing the
current one.
\item \textbf{Latent confounder estimation.} Combine the PAC-Bayes
mechanism posteriors with instrumental-variable or proximal causal
inference to estimate $U$-loadings rather than absorbing them into
the noise terms; tighten the bound on environments with identifiable
confounders.
\item \textbf{Real-data deployment.} The external-data checks of
App.~\ref{app:external} exercise the certificate on two public logs,
but neither can settle gate firing: both are value certificates well
below their logging policies' measured values, and neither log
randomises over a candidate with a certifiable advantage at its sample
size. The most direct remaining test of the safe-deployment workflow is
a logged industrial or clinical deployment (process control,
semiconductor yield, demand-response) with randomised legacy logging
and a candidate worth certifying.
\end{itemize}

\section{External-Data Checks}
\label{app:external}

This appendix gives the constructions, caveats and full tables behind
Section~\ref{sec:exp-external}. Both logs are public and used under
their licences (CC BY 4.0 in both cases); neither contains data we
generated, and certificates need only the logged
$(x_t, a^{\mu}_t, R_t)$ tuples and the logging propensities.

\subsection{Open Bandit Dataset}
\label{app:external-obd}

\textbf{Construction.} The Open Bandit Dataset \citep{cr_saito2021open}
(campaign \texttt{all}, $10{,}000$-round sample per policy) logged the
same platform in the same week under a uniform Random policy and a
Bernoulli Thompson Sampling (BTS) policy on A/B-split traffic. The
\emph{forward} direction certifies the BTS policy from the Random
log; its target is the empirical action frequency $\bar\pi_\mathrm{bts}$
of the BTS log (a fixed policy, so $\KL = 0$), and its weight ceiling
$B_w = 80 \times 0.1105 = 8.84$ is exact because uniform logging fixes
$\mu = 1/80$ at every round. The \emph{reverse} direction evaluates the
Random policy from the BTS log; its $B_w = 277.8$ is the empirical
propensity floor of the BTS log rather than an Assumption-level ceiling,
so it is a poor-overlap contrast and not a certificate. Each direction
has $81$ cells ($K \in \{100, 300, 1000, 3000\}$ with $20$ subsample
seeds each, plus the full log), $\delta = 0.05$, and $\lambda$
pre-fixed per $K$ by the data-free rule
$\lambda = \min\{(1-\eta_{\mathrm{clip}})/B_w,\ \sqrt{2\log(1/\delta)/(K B_w^2)}\}$,
$\eta_{\mathrm{clip}} = 0.05$. On-policy references are $\widehat V(\pi_\mathrm{bts}) =
0.0042$ $[0.0030, 0.0054]$ and $\widehat V(\pi_\mathrm{rand}) = 0.0038$
$[0.0027, 0.0050]$ (1000-resample bootstrap).

\begin{table}[h]\footnotesize\centering\setlength{\tabcolsep}{5pt}
\caption{Open Bandit Dataset: certified $V_\mathrm{lower}$ (mean over
$20$ subsample seeds; $1$ cell at $K = 10{,}000$), $\delta = 0.05$.
Violations: cells with $V_\mathrm{lower}$ above the on-policy estimate.
The zero-click control sets every click to $0$.}
\label{tab:obd}
\begin{tabular}{@{}rrrrrr@{}}
\toprule
$K$ & forward $V_\mathrm{lower}$ & zero-click control & reverse $V_\mathrm{lower}$ & reverse/forward & violations (fwd / rev)\\
\midrule
$100$    & $-2.3767$ & $-2.3810$ & $-74.8036$ & $31.5$ & $0/20$ / $0/20$\\
$300$    & $-1.3108$ & $-1.3153$ & $-41.3244$ & $31.5$ & $0/20$ / $0/20$\\
$1000$   & $-0.6972$ & $-0.7031$ & $-22.0873$ & $31.7$ & $0/20$ / $0/20$\\
$3000$   & $-0.3970$ & $-0.4012$ & $-12.6026$ & $31.7$ & $0/20$ / $0/20$\\
$10000$  & $-0.2134$ & $-0.2182$ & $-6.8534$  & $32.1$ & $0/1$ / $0/1$\\
\bottomrule
\end{tabular}
\end{table}

\textbf{Reading.} Zero violations in all $162$ cells rule out gross
miscalibration only: with clicks in $\{0,1\}$ at a click-through rate
of about $0.4\%$, a forward violation would need the bound to be biased upward
by $52$ to $570$ times the on-policy value. The cells are also nested
(subsamples of one $10{,}000$-round log), so $0/81$ is not $81$
independent successes. The forward log-log slope is $-0.52$ (OLS over
the five $K$ values), but the zero-click control, which contains no
information beyond $K$ and $B_w$, has slope $-0.518$ and reproduces
the curve to within $+0.004$ to $+0.006$ at every $K$: the slope is
consistent with the $O(1/\sqrt K)$ rate, not a measurement of it. The
reverse direction is $31.5$ to $32.1$ times looser at matched $K$, the
overlap cost measured on external data. No positive certificate is
reachable at this $n$, by design of the dataset, and none is claimed.
The log is single-level, so it exercises the certificate and not the
nested workflow.

\subsection{HeartSteps V1}
\label{app:external-heartsteps}

\textbf{Construction.} HeartSteps V1 \citep{klasnja2019efficacy} is a
public micro-randomised mHealth trial ($44$ adults enrolled, $37$ in the public release, nominally $42$ days)
with activity suggestions randomised at five decision points per day,
$P(\mathrm{send} \mid \mathrm{available}) = 0.6$ by design. We
instantiate two timescales inside the trial's randomisation: the
day's slot-$1$ suggestion as the meta-action, evaluated on the daily
step return, and the slot-$2$ to slot-$5$ suggestions as the inner
level, evaluated on $30$-minute proximal step returns (rewards are
$\log(\mathrm{steps} + 0.5)$). Slot-$1$ as meta-action is a
designation of a decision point rather than a distinct intervention
channel, and we disclose it as such: the trial's daily planning
randomisation is not part of the public release. The estimand is
availability-conditional (the standard MRT causal-excursion
estimand), with user-day episodes as the independence unit: $1{,}319$
outer and $1{,}641$ inner episodes, each split $50/50$ into a fit
half and a held-out evaluation half. The importance weights are
bounded by $1/0.4 = 2.5$ exactly by design, and the LS temperature
$\lambda^\star(K)$ is derived from the disjoint fit half, so the
evaluation sample selects neither the policy nor its own $\lambda$.
The one quantity that still reads the whole log is the reward range
when $R_\mathrm{hi}$ is taken as the empirical maximum, whose
maximising inner episode falls in the evaluation half; the headline
numbers below therefore use a-priori physiological caps, and the
empirical-max figures are reported as a sensitivity. Three structural
caveats are stated with the result. All episodes come from the same
$37$ users ($16$--$46$ outer and $36$--$63$ inner episodes each, with
both halves drawing on all users), so cross-episode independence is an
approximation at the user level and a cluster-robust analysis over the
$37$ users is the natural check. Inner episodes hold $1$ to $4$
available slots (mean $3.14$) against the nominal $N_L = 4$, a
convention shared by the certificate, by $\widehat V(\mu)$ and by the
trivial floor, so their comparison is internally consistent while the
fixed $N_L$ is out-of-model for this trial. Finally, if suggestions
carry over within a day, the inner certificate bounds the
causal-excursion value, that is, deviating to $\pi_Q$ at one
randomised decision point with the remainder of the day left under
$\mu$, rather than the value of deploying $\pi_Q$ at every slot; the
outer level, with one decision point per episode, is immune to this
distinction.

\textbf{Results} (Table~\ref{tab:heartsteps}, $\delta{=}0.05$). Each
level's trivial floor is $\log(0.5)$ per decision point ($-0.69$
outer, $-2.77$ inner), the a-priori minimum of any policy's value.
The outer certificate is positive on this real log:
$V_\mathrm{lower} = {+}5.63$ and ${+}5.39$ per episode at full sample
($660$ evaluation episodes) under the two a-priori caps, ${+}5.68$
under the in-sample cap; under that in-sample cap, the one the $K$
sweep was computed with, it clears the trivial floor from
$K \approx 100$ and is positive in the mean from $K = 200$. The inner
certificate sits ${+}2.43$ and ${+}1.33$ above its trivial floor under
the a-priori caps (${+}3.36$ under the in-sample cap); its absolute
sign is cap-dependent, positive only under the in-sample cap, and is
disclosed rather than optimised over. Both levels contract at
approximately $1/\sqrt{K}$. The randomised parent action shifts the
inner-context distribution by ${+}0.18$ log-steps (SE $0.11$,
$z = 1.59$ at the user-day randomisation unit) and the inner day-mean
response by ${+}0.11$ (SE $0.11$, $z = 1.00$): point estimates
consistent with the nested interface, not individually significant at
the randomisation unit.

\textbf{What is and is not claimed.} These are value certificates
plus contraction, not certified improvement: the trial's treatment
effects are within about $1.5$ standard errors of zero, so no valid
$95\%$ bound can place $V(\pi)$ above the logging policy's measured
value ($\widehat V(\mu) = {+}8.78$ outer, ${+}12.54$ inner) at this
sample size, at either level, and none is claimed.

\begin{table}[h]\footnotesize\centering\setlength{\tabcolsep}{4pt}
\caption{HeartSteps V1: per-episode $V_\mathrm{lower}$ on held-out
evaluation subsamples of size $K$, mean$\,\pm\,1\sigma$ across $20$
subsample seeds ($1$ seed at full sample: $K{=}660$ outer, $K{=}821$
inner), $\delta{=}0.05$. The $K$ columns use the in-sample
empirical-max reward cap, under which the sweep was computed; the last
column gives the primary a-priori-capped full-sample certificates
($50{,}000$ / $100{,}000$ steps per day outer, $10{,}000$ / $20{,}000$
steps per $30$ min inner), which sit ${+}6.33$ / ${+}6.08$ (outer) and
${+}2.43$ / ${+}1.33$ (inner) above the trivial floors $-0.69$ and
$-2.77$.}
\label{tab:heartsteps}
\begin{tabular}{@{}lrrrrrrr@{}}
\toprule
$K$   & $25$                 & $50$               & $100$              & $200$               & $400$              & full & a-priori\\
\midrule
outer & ${-}18.24{\pm}5.20$  & ${-}7.58{\pm}1.55$ & ${-}0.11{\pm}0.89$ & ${+}2.90{\pm}0.46$  & ${+}4.77{\pm}0.25$ & ${+}5.68$ & ${+}5.63$/${+}5.39$\\
inner & ${-}103.90{\pm}10.68$& ${-}53.48{\pm}3.24$& ${-}27.08{\pm}1.58$& ${-}13.02{\pm}1.11$ & ${-}4.99{\pm}0.65$ & ${+}0.59$ & ${-}0.35$/${-}1.44$\\
\bottomrule
\end{tabular}
\end{table}

The certificate treats user-days as independent episodes, and since
suggestions within a day may influence later availability, the inner
level sits in the endogenous within-level regime that
Assumption~\ref{ass:iid} excludes; the coverage claim therefore rests on the Open Bandit check.

\clearpage
\section*{NeurIPS Paper Checklist}

\begin{enumerate}

\item {\bf Claims}
    \item[] Question: Do the main claims made in the abstract and introduction accurately reflect the paper's contributions and scope?
    \item[] Answer: \answerYes{}
    \item[] Justification: The abstract and Section~\ref{sec:intro} state three contributions, namely the NCCB problem class and NCTS algorithm (Sections~\ref{sec:problem}--\ref{sec:algorithm}), the causal PAC-Bayesian value bound certifying any intervention scope $\mathcal{I} \subseteq [L]$ off-policy and anytime (Theorem~\ref{thm:pac-bayes-causal}, Section~\ref{sec:per-level}), and empirical evidence on the hierarchical SCM with reported $p$-values (Section~\ref{sec:experiments}). Each contribution is delivered in the corresponding numbered section and appendix.

\item {\bf Limitations}
    \item[] Question: Does the paper discuss the limitations of the work performed by the authors?
    \item[] Answer: \answerYes{}
    \item[] Justification: Section~\ref{sec:limitations} flags three limitations: within-level exogeneity (Assumption~\ref{ass:iid}), whose boundary (endogenous within-level state, as in titration) is worked out there; the restriction to additive-noise (ANM) mechanisms, with the non-additive case a stress test only; and the synthetic scope of exact validity verification. The multiplicative overlap penalty $B^{\mathcal{I}} \leq R_{\max} \prod_{\ell \in \mathcal{I}} B^{(\ell)}_w$ in Theorem~\ref{thm:pac-bayes-causal} is discussed in Appendix~\ref{app:ext-limitations} alongside concrete tightening directions (self-normalised importance sampling, PAC-Bayes-Bernstein refinements, self-normalised maximal inequalities).

\item {\bf Theory assumptions and proofs}
    \item[] Question: For each theoretical result, does the paper provide the full set of assumptions and a complete (and correct) proof?
    \item[] Answer: \answerYes{}
    \item[] Justification: Theorem~\ref{thm:pac-bayes-causal} explicitly references Assumptions~\ref{ass:markov}--\ref{ass:bounded-weights} (Causal Markov plus exogenous noise, per-level i.i.d.\ contexts, overlap and backdoor admissibility, and per-level bounded importance weights). The full proof is given in Appendix~\ref{app:proof-pacbayes} in four steps (filtration setup, causal KL decomposition, logarithmic-smoothing supermartingale, Ville's inequality), followed by the proof of the excess-risk corollary, with a proof sketch in the main text and supporting Lemmas~2--4 stated and proved. All theorems, formulas, and assumptions are numbered and cross-referenced.

    \item {\bf Experimental result reproducibility}
    \item[] Question: Does the paper fully disclose all the information needed to reproduce the main experimental results of the paper to the extent that it affects the main claims and/or conclusions of the paper (regardless of whether the code and data are provided or not)?
    \item[] Answer: \answerYes{}
    \item[] Justification: Section~\ref{sec:exp-setup} fixes the agent taxonomy, statistical protocol (10 seeds, $K{=}2000$ episodes, Welch's $t$-test), and \bench{$\Gamma$-Stress} preset; Appendix~\ref{app:exp-setup} provides the full \bench{$\Gamma$-SCM} structural equations, axis semantics, posterior-backend hyperparameters (RFF-GP-Gibbs with $D{=}128$, RBF kernel $\exp(-\lVert x - x'\rVert^2)$; NIG-Gibbs as alternative), and per-bench deviations from defaults; Appendix~\ref{app:env-params} gives a worked linear instance with its legacy logging policy and prior. Algorithms~\ref{alg:ncts} (NCTS) and~\ref{alg:ra-ncts} (AEGIS) are stated in pseudocode.

\item {\bf Open access to data and code}
    \item[] Question: Does the paper provide open access to the data and code, with sufficient instructions to faithfully reproduce the main experimental results, as described in supplemental material?
    \item[] Answer: \answerYes{}
    \item[] Justification: A code repository is provided at \url{https://github.com/PaulDavidZuercher/AEGIS}, with configuration files and reproduction scripts for the reported numbers. Synthetic data are simulated from the parametric \bench{$\Gamma$-SCM} specified in Appendix~\ref{app:exp-setup}; the two external-data checks (Appendix~\ref{app:external}) use the public HeartSteps V1 release \citep{klasnja2019efficacy} and the Open Bandit Dataset \citep{cr_saito2021open}, both openly available under CC BY 4.0, with preprocessing scripts in the repository.

\item {\bf Experimental setting/details}
    \item[] Question: Does the paper specify all the training and test details (e.g., data splits, hyperparameters, how they were chosen, type of optimizer) necessary to understand the results?
    \item[] Answer: \answerYes{}
    \item[] Justification: Section~\ref{sec:exp-setup} states the experiment defaults; Appendices~\ref{app:env-params} and~\ref{app:exp-setup} list all SCM parameters, prior hyperparameters (NIG with $\mu_0{=}0, \Lambda_0{=}I, a_0{=}2, b_0{=}1$ in the experiments), posterior-update rules (NIG-conjugate for linear ANM, online Gibbs for nonlinear ANM), legacy logging policy, RFF-GP-Gibbs hyperparameters ($D{=}128$, RBF kernel $\exp(-\lVert x - x'\rVert^2)$, $\lambda$ chosen per bench), and per-bench data-split protocols (e.g., $\alpha{=}0.5$ Bernoulli random split for the bound contraction in Section~\ref{par:bench-1c}).

\item {\bf Experiment statistical significance}
    \item[] Question: Does the paper report error bars suitably and correctly defined or other appropriate information about the statistical significance of the experiments?
    \item[] Answer: \answerYes{}
    \item[] Justification: Section~\ref{sec:exp-setup} fixes the default: mean $\pm 1\sigma$ across $10$ seeds with significance via Welch's $t$-test; deviations (fewer seeds, standard errors) are stated inline. Headline claims report Cohen's $d$ and $p$-value (Table~\ref{tab:effect-attribution}), with Welch $t$-statistics in App.~\ref{app:exp-attribution} (e.g., $p < 10^{-7}$ on Shift-$T$, $p < 10^{-10}$ on the recursive-vs-joint contrast in Section~\ref{sec:exp-constructive}; paired-endpoint test $t{=}31.3$, $df{=}9$, $p < 10^{-8}$, in the bound-contraction benchmark).

\item {\bf Experiments compute resources}
    \item[] Question: For each experiment, does the paper provide sufficient information on the computer resources (type of compute workers, memory, time of execution) needed to reproduce the experiments?
    \item[] Answer: \answerNo{}
    \item[] Justification: Compute resources are not explicitly tabulated. Experiments run on a single machine with no GPU required (the two external-data checks run in seconds); the heaviest are the long-$K$ runs of up to $1.5 \times 10^{6}$ certified episodes.
    
\item {\bf Code of ethics}
    \item[] Question: Does the research conducted in the paper conform, in every respect, with the NeurIPS Code of Ethics \url{https://neurips.cc/public/EthicsGuidelines}?
    \item[] Answer: \answerYes{}
    \item[] Justification: The work is theoretical and methodological; experiments use simulated environments plus secondary analysis of two public, de-identified datasets released for research (HeartSteps V1, Open Bandit Dataset); no new human-subjects data, no scraped data, no released pre-trained models. The deployment framing (progressive certified handover with practitioner-set risk thresholds) is explicitly oriented towards responsible deployment.

\item {\bf Broader impacts}
    \item[] Question: Does the paper discuss both potential positive societal impacts and negative societal impacts of the work performed?
    \item[] Answer: \answerYes{}
    \item[] Justification: The paper targets safe deployment in high-stakes sequential domains (agronomic input allocation, process control; Section~\ref{sec:intro}) where mis-deployed learned policies can cause direct harm. \emph{Progressive certified handover} (Section~\ref{sec:per-level}) is the explicit positive-impact mechanism: control transfers from a legacy controller to the agent only when a $(1{-}\delta)$-certified gain is met. A false certificate can hand control to a harmful policy: the certificate does not detect wrong propensities (Section~\ref{sec:graph-error}) or within-level endogenous state (Appendix~\ref{app:ext-limitations}), so logging and Assumption~\ref{ass:iid} must be audited before deployment. We see no direct path to deliberate misuse, since the framework adds conservative gating to systems an operator already runs.
    
\item {\bf Safeguards}
    \item[] Question: Does the paper describe safeguards that have been put in place for responsible release of data or models that have a high risk for misuse (e.g., pre-trained language models, image generators, or scraped datasets)?
    \item[] Answer: \answerNA{}
    \item[] Justification: The paper releases neither pre-trained models, image generators, nor scraped datasets; the contributions are an algorithm (NCTS), a deployment-time wrapper (AEGIS), and a PAC-Bayesian bound (Theorem~\ref{thm:pac-bayes-causal}), none of which carry high misuse risk.

\item {\bf Licenses for existing assets}
    \item[] Question: Are the creators or original owners of assets (e.g., code, data, models), used in the paper, properly credited and are the license and terms of use explicitly mentioned and properly respected?
    \item[] Answer: \answerYes{}
    \item[] Justification: Prior algorithmic frameworks (PSRL, C-MDPs / C-UCBVI, structural causal bandits, online PAC-Bayes, logarithmic smoothing, pessimistic offline policy learning) are cited at first use in Section~\ref{sec:related} and throughout the body. Two third-party datasets are used, both cited and both under CC BY 4.0: HeartSteps V1 \citep{klasnja2019efficacy} and the Open Bandit Dataset \citep{cr_saito2021open} (Appendix~\ref{app:external}); all other experiments run on the simulated \bench{$\Gamma$-SCM} in Appendix~\ref{app:exp-setup}; no pre-trained models are loaded.

\item {\bf New assets}
    \item[] Question: Are new assets introduced in the paper well documented and is the documentation provided alongside the assets?
    \item[] Answer: \answerYes{}
    \item[] Justification: A code repository at \url{https://github.com/PaulDavidZuercher/AEGIS} contains the implementation of NCTS (Algorithm~\ref{alg:ncts}), AEGIS (Algorithm~\ref{alg:ra-ncts}), the \bench{$\Gamma$-SCM} simulator, and reproduction scripts, with in-repo documentation.

\item {\bf Crowdsourcing and research with human subjects}
    \item[] Question: For crowdsourcing experiments and research with human subjects, does the paper include the full text of instructions given to participants and screenshots, if applicable, as well as details about compensation (if any)? 
    \item[] Answer: \answerNA{}
    \item[] Justification: The paper involves no crowdsourcing and collects no data from human subjects; besides simulated SCM data, it performs a secondary analysis of two public, de-identified releases, HeartSteps V1 \citep{klasnja2019efficacy} and the Open Bandit Dataset \citep{cr_saito2021open}.

\item {\bf Institutional review board (IRB) approvals or equivalent for research with human subjects}
    \item[] Question: Does the paper describe potential risks incurred by study participants, whether such risks were disclosed to the subjects, and whether Institutional Review Board (IRB) approvals (or an equivalent approval/review based on the requirements of your country or institution) were obtained?
    \item[] Answer: \answerNA{}
    \item[] Justification: The paper collects no data from human subjects; its human-derived data are two public, de-identified releases (HeartSteps V1, Open Bandit Dataset), analysed secondarily under their CC BY 4.0 licences; no new IRB review was required for this secondary use.

\item {\bf Declaration of LLM usage}
    \item[] Question: Does the paper describe the usage of LLMs if it is an important, original, or non-standard component of the core methods in this research? Note that if the LLM is used only for writing, editing, or formatting purposes and does \emph{not} impact the core methodology, scientific rigor, or originality of the research, declaration is not required.
    \item[] Answer: \answerNA{}
    \item[] Justification: LLMs are not part of the core methodology; the algorithms (NCTS in Algorithm~\ref{alg:ncts}, AEGIS in Algorithm~\ref{alg:ra-ncts}) and the bound (Theorem~\ref{thm:pac-bayes-causal}) involve no LLM components.

\end{enumerate}

\end{document}